\documentclass{article}
\usepackage{ijcai25}

\usepackage{times}
\usepackage{soul}
\usepackage{url}
\usepackage[hidelinks]{hyperref}
\usepackage[utf8]{inputenc}
\usepackage[small]{caption}
\usepackage{graphicx}
\usepackage{amsmath}
\usepackage{amsthm}
\usepackage{booktabs}
\usepackage{algorithm}
\usepackage{algorithmic}
\usepackage[switch]{lineno}

\usepackage{amssymb}
\usepackage{amsfonts}
\usepackage{mathtools}
\usepackage{dsfont} % Used for \N \R etc.
\usepackage{xspace} % For controlling spaces after newcommand
\usepackage{subcaption} % ensure subfigure will work
\usepackage{natbib}  % DO NOT CHANGE THIS AND DO NOT ADD ANY OPTIONS TO IT
\usepackage{caption} % DO NOT CHANGE THIS AND DO NOT ADD ANY OPTIONS TO IT
\usepackage{thmtools,thm-restate}
\usepackage{aligned-overset} 
\usepackage{url}
\usepackage{xcolor}

\usepackage{diagbox}
\usepackage{multirow}

\title{Randomised Optimism via Competitive Co-Evolution \\for Matrix Games with Bandit Feedback}

\author{
    Shishen Lin
    \affiliations
    School of Computer Science, University of Birmingham, Birmingham, United Kingdom
    \emails
    sxl1242@student.bham.ac.uk
}

\theoremstyle{definition}
\newtheorem{definition}{Definition}

\theoremstyle{plain} % this sets the style for all new environments created using \newtheorem to have the "plain" style, which as a bold title, italic text and vertical space above and below it. 
\renewcommand{\epsilon}{\varepsilon}
\newcommand{\N}{\ensuremath{\mathds{N}}\xspace}

\newcommand{\X}{\ensuremath{\mathcal{X}}\xspace}
\newcommand{\Y}{\ensuremath{\mathcal{Y}}\xspace}
\newcommand{\V}{\ensuremath{\mathcal{V}}\xspace}

\newcommand{\F}{\ensuremath{\mathcal{F}}\xspace}

\newcommand{\R}{\ensuremath{\mathcal{R}}\xspace}

\newcommand{\Nr}{\ensuremath{\mathcal{N}}\xspace}

\newcommand{\E}[1]{\ensuremath{\mathrm{E}\mathord{\left(#1\right)}}}
\newcommand{\EAlg}[1]{\ensuremath{\mathrm{E}_{\eta,\alg}\mathord{\left(#1\right)}}}
\newcommand{\Et}[1]{\ensuremath{\mathrm{E}_t\mathord{\left(#1\right)}}}

\newcommand{\COEBL}{\textsc{Coebl}\xspace}
\newcommand{\coebl}{\COEBL}
\newcommand{\HEDGE}{\textsc{Hedge}\xspace}
\newcommand{\hedge}{\HEDGE}
\newcommand{\EXPThree}{\textsc{Exp3}\xspace}
\newcommand{\exptr}{\EXPThree}
\newcommand{\EXPTrNi}{\textsc{Exp3-IX}\xspace}
\newcommand{\exptrni}{\EXPTrNi}
\newcommand{\UCB}{\textsc{Ucb}\xspace}
\newcommand{\ucb}{\UCB}

\newcommand{\DIAGONAL}{\textsc{Diagonal}\xspace}
\newcommand{\Diagonal}{\DIAGONAL}

\newcommand{\BIGNUM}{\textsc{BiggerNumber}\xspace}
\newcommand{\BigNum}{\BIGNUM}

\newcommand{\A}{\ensuremath{\mathcal{A}}\xspace}

\newcommand{\U}{\ensuremath{\mathcal{U}}\xspace}
\newcommand{\Or}[1]{\ensuremath{\mathcal{O}\left(#1\right)}}
\newcommand{\AsyO}[1]{\ensuremath{\tilde{\mathcal{O}}\left(#1\right)}}

\newcommand{\alg}{\ensuremath{\textsc{Alg}}\xspace}

\newcommand{\worstre}{\ensuremath{\textsc{WorstCaseRegret}}\xspace}

\newcommand{\appendixtableofcontents}{
  \begingroup
  \let\clearpage\relax
  \let\cleardoublepage\relax
  \let\cleardoublepage\relax
  \tableofcontents
  \endgroup
}

\begin{document}

\maketitle

\begin{abstract}
   Learning in games is a fundamental problem in machine learning and artificial intelligence, with numerous applications~\citep{silver2016mastering,schrittwieser2020mastering}. This work investigates two-player zero-sum matrix games with an unknown payoff matrix and bandit feedback, where each player observes their  actions and the corresponding noisy payoff. Prior studies have proposed algorithms for this setting~\citep{o2021matrix,maiti2023query,cai2024uncoupled}, with \citet{o2021matrix} demonstrating the effectiveness of deterministic optimism (e.g., \ucb) in achieving sublinear regret. However, the potential of randomised optimism in matrix games remains theoretically unexplored.
   
   We propose Competitive Co-evolutionary Bandit Learning (\coebl), a novel algorithm that integrates evolutionary algorithms (EAs) into the bandit framework to implement randomised optimism through EA variation operators. We prove that \coebl achieves sublinear regret, matching the performance of deterministic optimism-based methods. To the best of our knowledge, this is the first theoretical regret analysis of an evolutionary bandit learning algorithm in matrix games.
   
   Empirical evaluations on diverse matrix game benchmarks demonstrate that \coebl not only achieves sublinear regret but also consistently outperforms classical bandit algorithms, including \exptr~\citep{auer2002nonstochastic}, the variant \exptrni~\citep{cai2024uncoupled}, and \ucb~\citep{o2021matrix}. These results highlight the potential of evolutionary bandit learning, particularly the efficacy of randomised optimism via evolutionary algorithms in game-theoretic settings.
   \end{abstract}

\section{Introduction}

       \subsection{Two-Player Zero-Sum Games}
       Triggered by von Neumann’s seminal work~\citep{jv1928theorie,neumann1953theory},
       the maximin optimisation problem (i.e., $\max_{x \in \X}\min_{y \in \Y} g(x, y)$) has become a major research topic in machine learning and optimisation.
       In particular, two-player zero-sum games, represented by a payoff matrix $A \in \mathbb{R}^{m \times m}$, constitute a widely studied problem class in the machine learning and AI literature~\citep{littman1994markov,pmlrv39auger14,o2021matrix,cai2024uncoupled}. 
       The row player selects $i \in [m]$, the column player selects $j \in [m]$ and these choices, leading to a payoff $A_{ij}$ (i.e. the row player receives the payoff $A_{ij}$ and the column player receives the payoff $-A_{ij}$).
       The objective is to find the optimal mixed strategy, which is a probability distribution over actions for each player. 
       Formally, we define our problem as follows: to find $x^*, y^* \in \Delta_m$, where $\Delta_m$ denotes the probability simplex of dimension $m-1$, satisfying 
       \begin{align}
               V_A^*:=\max_{x \in \Delta_m} \min_{y \in \Delta_m} y^T A x . \label{eq:Nash1}
       \end{align}
       % where $\Delta_m$ denotes the probability simplex of dimension $m-1$.
       By von Neumann's minimax theorem~\citep{jv1928theorie}, $V_A^*=\min_{y \in \Delta_m}\max_{x \in \Delta_m} y^T A x$.
       $(x^*, y^*)$ solving for Eq.(\ref{eq:Nash1}) is called a Nash equilibrium. 
       $V_A^*$ is the shared optimal quantity at the Nash equilibrium.
       In this paper, we call it the Nash equilibrium payoff.
       
       Nash's Theorem, or von Neumann's minimax theorem, guarantees the existence of $(x^*,y^*)$ for Eq.(\ref{eq:Nash1}) \citep{jv1928theorie,nash1950equilibrium}.
       If the payoff matrix is given or known, then Eq.(\ref{eq:Nash1}) can be reformulated as a linear programming problem, and it can be solved in polynomial time using algorithms including the ellipsoid method or interior point method~\citep{bubeck2015convex,maiti2023query}.
       Now, if the payoff matrix is unknown, let the row and column player play an iterative two-player zero-sum game. 
      %  At each iteration, based on their action, we can query one of the entries in the payoff matrix, and then both players can adjust their strategies based on the observed payoff (or reward). 
      %  We repeat these iterations until the stopping criteria are met. 
      %  This kind of two-player zero-sum game is called repeated matrix games (or matrix games, for short).
      At each iteration, they select actions and observe the corresponding payoff entry from the matrix. Based on the observed rewards, both players update their strategies.
      This iterative setting is referred to as a repeated matrix game (or matrix games, for short).
      Our goal is to design algorithms that can perform competitively in such games.
      A standard measure of performance in this setting is regret, which we define formally in later sections. 
      We are also interested in whether the algorithm can approximate the Nash equilibrium $(x^*, y^*)$, measured using divergence metrics such as KL-divergence or total variation distance.
    
       % \lss{To do: polish this paragraph and add more references (ICLR)}
    
       \subsection{Evolutionary Reinforcement Learning}
       
       Evolutionary Algorithms (EAs) are randomised heuristics inspired by natural selection, designed to solve optimisation problems~\citep{rozenberg_coevolutionary_2012,eiben_introduction_2015}.
       EAs aim to find global optima with minimal knowledge about fitness functions, making them well-suited for black-box or oracle settings compared to gradient-based methods. 
       They are powerful tools for discovering effective reinforcement learning policies.
       EAs are particularly useful because they can identify good representations, manage continuous action spaces, and handle partial observability.
       Due to these strengths, evolutionary reinforcement learning (ERL) techniques have shown strong empirical success, and we refer readers to~\citep{whiteson2012evolutionary,bai2023evolutionary,li2024bridging} for detailed reviews of ERL.

       Coevolution, rooted in evolutionary biology, involves the simultaneous evolution of multiple interacting populations~\citep{rozenberg_coevolutionary_2012}.
       Interactions can be cooperative (e.g., humans and gut bacteria) or competitive (e.g., predator-prey dynamics).
       These co-evolutionary dynamics have been studied and applied in ERL, demonstrating empirical effectiveness in many applications~\citep{whiteson2012evolutionary,xue2023sample,li2024bridging}. 
       For example, co-evolutionary algorithms (CoEAs), a subset of EAs, have been applied in many black-box optimisation problems under various game-theoretic scenarios~\citep{xue2023sample,gomes2014novelty,hemberg2021spatial,flores2022coevolutionary,MarioRLSBi23,hevia2024analysis,benford2024runtime}.	
       % However, there is limited research on coevolutionary learning in matrix games with bandit feedback.
    
      %  Evolutionary reinforcement learning has achieved great success in many applications, including game playing, robotics, and optimisation~\citep{moriarty1999evolutionary,khadka2018evolution,pourchot2018cemrl,hao2023erlre,Li2024erl,Li2024erlver2}, 
      %  but there is barely any theoretical analysis or guarantees of these powerful methods \citep{li2024bridging}.
      %  In particular, the theoretical understanding of coevolutionary learning remains blanked, especially in the context of matrix games.
      Despite the practical success of evolutionary reinforcement learning in domains such as game playing, robotics, and optimisation~\citep{moriarty1999evolutionary,khadka2018evolution,pourchot2018cemrl,hao2023erlre,Li2024erl,Li2024erlver2}, there is a lack of rigorous theoretical analysis~\citep{li2024bridging}.
      In particular, the theoretical foundations of coevolutionary learning in matrix games remain largely unexplored.
      As a starting point, in this work, we aim to bridge this gap by combining evolutionary heuristics with bandit learning and analysing their performance in matrix games from both theoretical and empirical perspectives.

      \subsection{Contributions}
      This paper introduces evolutionary algorithms for learning in matrix games with bandit feedback.
      To the best of our knowledge, this is the first work to provide a rigorous regret analysis of evolutionary reinforcement learning (i.e., \coebl) in matrix games with bandit feedback.
      Specifically, we show that randomised optimism implemented via evolutionary algorithms can achieve sublinear regret in this setting. 
      Our empirical results demonstrate that \coebl outperforms existing bandit learning baselines for matrix games, including \exptr, \ucb, and the \exptrni variant.
      These findings highlight the significant potential of evolutionary algorithms for bandit learning in game-theoretic environments and reveal the role of randomness in effective game-play.
      This work serves as a first step towards a rigorous theoretical understanding of evolutionary reinforcement learning in game-theoretic settings.

       \subsection{Related Works}
      \subsubsection{Regret Analysis of Bandit Learning in Matrix Games}
      Theoretical analysis of bandit learning algorithms in matrix games has been extensively studied.
      Recent works, such as~\citep{pmlrv39auger14,o2021matrix,cai2024uncoupled}, have studied classical bandit algorithms under settings where only rewards or payoffs are observed. 
      In particular, \citet{o2021matrix} conducted an in-depth regret analysis of the \ucb algorithm, Thompson Sampling, and K-Learning, showing that these methods achieve sublinear regret in matrix games.
      \citet{neu2015explore} established a sublinear regret bound for \exptrni, which was subsequently extended to matrix games by \citet{cai2024uncoupled} through a new variant.
      Additionally, \citet{pmlrv39auger14} provided convergence analyses of bandit algorithms in sparse binary zero-sum games, while \citet{cai2024uncoupled} extended these results to uncoupled learning in two-player zero-sum Markov games.
      A more recent study by \citet{li2024optimistic} investigates adversarial regret for Optimistic Thompson Sampling, exploiting repeated-game structures and partial observability to anticipate opponent strategies.

      In contrast, our work aligns with \citet{o2021matrix} in focusing on Nash regret in two-player zero-sum matrix games with bandit feedback. 
      While their approach highlights adversarial dynamics, ours demonstrates the potential of randomised optimism via evolutionary algorithms. 
      We provide a novel regret analysis that complements stochastic optimism-based methods in achieving sublinear Nash regret.
      Moreover, the theoretical understanding of evolutionary bandit learning remains largely unexplored. 
      This paper aims to fill that gap, marking a first step towards the rigorous study of evolutionary bandit learning in matrix games, an area that remains both promising and under-explored.

      \subsubsection{Runtime Analysis of Co-evolutionary Algorithms}
      Recent studies have conducted runtime analyses of cooperative and competitive co-evolutionary algorithms~\citep{jansen_cooperative_2004,lehre_runtime_2022}.
      Here, runtime refers to the number of function evaluations required for the algorithm to find the Nash equilibrium.
      For a comprehensive overview of these contributions to competitive co-evolutionary algorithms, we refer readers to the recent papers \citep{lehre_runtime_2022,hevia2023fitnessAggregation,MarioRLSBi23,lin2024overcoming,benford2024runtime,benford2024runtimeImpartial,lin2024neurips,lin2025aaai}.
      Although we do not analyse the runtime of \coebl in this paper, an interesting direction for future research is to explore how the runtime of \coebl could be studied in the context of matrix games with bandit feedback.
      The idea of competitive coevolution in game-theoretic settings, as explored in the aforementioned works, serves as the foundation for our application of evolutionary methods to bandit learning in matrix games.

\section{Preliminaries}
       \subsection{Notations}
       Given $n \in \mathbb{N}$, we write $[n]:=\{1,2, \cdots ,n\}$.
       $\mathbb{F}_p$ denotes the finite field of $p$ (prime number) elements.
       For example, $\mathbb{F}_3$ denotes the finite field of three elements, $\{-1,0,1\}$.
       We denote the row player by the $x$-player and the column player by the $y$-player.
       $f(n) \in O(h(n))$ if there exists some constant $c>0$ such that $f(n) \leq c h(n)$.
       $f(n) \in \tilde{O}(h(n))$ if there exists some constant $k>0$ such that $f \in O(h(n) \log ^k \left(h(n)) \right)$.
       We define the $(m-1)$-dimensional probability simplex as $\Delta_m:=\{z \in \mathbb{R}^m \mid \sum_{i=1}^mz_i=1, z_i\geq 0\}$.
       In each round $t \in \mathbb{N}$, the row player chooses $i_t \in [m]$, and the column player chooses $j_t \in [m]$; and then $r_t$ is the reward obtained by the row player.
       We define the corresponding filtration $\F_t$ prior to round $t$ by $\F_t := (i_1, j_1, r_1, \ldots, i_{t-1}, j_{t-1}, r_{t-1})$.
       We denoted $\Et{\cdot}:= \E{\cdot \mid \F_t}$.
       For any real number $x$, we define $1 \lor x := \max(1, x)$.
       Given $x \in \{0,1\}^n$, $|x|_1:=\sum_{i=1}^nx_i$.

       \begin{definition}
           A random variable \(X \in \mathbb{R}\) is $\sigma^2$-sub-Gaussian with variance proxy \(\sigma^2\) if \(\E{X} = 0\) and satisfies $ \E{\exp(sX)} \leq \exp\left(\frac{\sigma^2 s^2}{2}\right)$, for all $s \in \mathbb{R}$.
       % \begin{align*}
       % \E{\exp(sX)} \leq \exp\left(\frac{\sigma^2 s^2}{2}\right), \quad \forall s \in \mathbb{R}.    
       % \end{align*}
       
       \end{definition}

      \subsection{Two-Player Zero-Sum Games and Nash Regret}

      A two-player game is characterised by the strategy spaces \(\X\) and \(\Y\), along with payoff functions \(g_i: \X \times \Y \to \mathbb{R}\), where \(i \in [2]\). Here, \(g_i(x, y)\) denotes the payoff received by player \(i\) when player~1 plays strategy \(x\) and player~2 plays strategy \(y\).
      
      \begin{definition}
      Given a two-player game with strategy spaces \(\X\) and \(\Y\), and a prime number \(p \in \N\), let the payoff functions \(g_1, g_2: \X \times \Y \to \mathbb{R}\) represent the payoffs for player~1 and player~2, respectively. The game is said to be \emph{zero-sum} if the gain of one player is exactly the loss of the other, i.e., \(g_1(x, y) + g_2(x, y) = 0\) for all \(x \in \X\) and \(y \in \Y\).
      
      \end{definition}
      % If \(g_1(x, y), g_2(x, y) \in \Ftp\) for all \(x, y \in \X\), the game is referred to as \emph{$p$-ary} (e.g., binary if \(p = 2\), ternary if \(p = 3\), or quinary if \(p = 5\)). Given a single payoff function \(g: \X \times \Y \rightarrow \Ftp\), we define a \emph{$p$-ary zero-sum game} by setting \(g_1(x, y) = g(x, y)\) and \(g_2(x, y) = -g(x, y)\).

      Many classical games where the outcomes are win, lose and draw, such as Rock-Paper-Scissors, Tic-Tac-Toe and Go, can be modelled as ternary zero-sum games,where \(g(x,y)=1\) denotes a win for player~1, \(g(x,y)=-1\) a win for player~2, and \(g(x,y)=0\) a draw.
      In this paper, we mainly focus on ternary two-player zero-sum games. In matrix games, we evaluate performance using the \emph{Nash regret}, defined as the cumulative difference between the Nash equilibrium payoff in Eq.~(\ref{eq:Nash1}) and the actual rewards obtained by the players.      
       
    %    \lss{Motivation and signficance of using Nash regret}

       \begin{definition}[Nash Regret~\citep{o2021matrix}]
       \label{def:regret}
       Consider any matrix game with payoff matrix $A \in  \mathbb{R}^{m \times m} $
       and the reward for the row player choosing action $i_t\in [m]$ and the column player choosing action $j_t \in [m]$ is given by 
       $r_t=A_{i_tj_t}+\eta_t$, where $\eta_t$ is zero-mean noise, independent and identically distributed from a known distribution at iteration $t \in \mathbb{N}$.
       Given an algorithm \alg\ that maps the filtration $\F_t$ to a distribution over actions $x \in \Delta_m$, we define the Nash regret with respect to the Nash equilibrium payoff $V_A^* \in \mathbb{R}$ by
       \begin{align*}
           \R \left(A, \alg, T \right):= \EAlg{\sum_{t=1}^T V_A^*-r_t}. 
       \end{align*}
       % where $A^{\eta}$ is an unknown matrix fixed at the start of the play and kept the same during the play.
       % \footnote{To simplify the notation, we use $A$ later to denote the noisy payoff matrix in this paper.}. 
       % We use $A$ later to denote the noisy payoff matrix for shorthand. 
       Given any class of games $A \in \A$, for any $T\in \mathbb{N}$, we define 
       \begin{align*}
           \worstre\left(\A, \alg, T\right):= \max_{A \in \A}\R \left(A, \alg, T \right).
       \end{align*}
       % Moreover, given a collection of algorithms $\algclass$, for any $\alg \in \algclass$, we can define the minimax Nash regret as
       % \begin{align*}
       %      \minimaxre \left(\A, \algclass, T\right):= \min_{\alg \in \algclass}\max_{A \in \A}\R \left(A, \alg, T \right).
       % \end{align*}
       % In this paper, we consider $\D$ defined by the general class of algorithms for matrix games with bandit feedback (c.f. Algorithm~\ref{alg:BBA}).
       \end{definition}
    
       Given a fixed unknown payoff matrix $A$, the regret $\R\left(A, \alg, T\right)$ represents the expected cumulative difference between the Nash equilibrium payoff and the actual rewards obtained by player 1 using algorithm $\alg$ over $T$ iterations. The worst-case regret $\worstre\left(\A, \alg, T\right)$ denotes the maximum regret of algorithm $\alg$ across all possible payoff matrices within the class of games $\A$, thus capturing performance in the worst-case scenario.       

      %  Given a fixed unknown payoff matrix $A$, $ \R \left(A, \alg, T \right)$ represents the expected cumulative difference between the Nash equilibrium payoff and the actual rewards obtained by player 1 using \alg over $T$ iterations.
      %  $ \worstre\left(\A, \alg, T\right)$ considers the maximum regret of Algorithm \alg over all the possible payoff matrices in the class of games $\A$.
      %  In other words, it denotes the expected regrets under the worst-case scenario.

       Nash regret serves as a fundamental measure for evaluating an agent's performance against a best-response opponent, enabling fair comparison with prior work \citep{o2021matrix}. It emphasises long-term convergence to equilibrium strategies, thereby reflecting robust and generalised behaviour in game-theoretic settings. Although it may be less informative for analysing intermediate behaviours \citep{li2024optimistic}, it offers critical guarantees regarding the agent's capacity to adapt towards optimal and resilient strategies against adversarial best-response opponents.

\section{Co-evolutionary Bandit Learning}

\subsection{Learning in Games and \coebl}
The study of learning dynamics in games seeks to understand how players can adapt their strategies to approach the equilibrium when interacting with rational opponents \citep{fudenberg1998theory}.
A commonly used metric for evaluating algorithmic performance in such settings is regret, as formally defined in Definition~\ref{def:regret}.
Alternative evaluation criteria include convergence to Nash equilibrium, which may be measured using KL-divergence or total variation distance.

In this section, we only present the algorithm for the $x$-player, noting that the counterpart for the $y$-player is symmetric.
The proposed method, \coebl (Co-Evolutionary Bandit Learning), leverages co-evolutionary approach in bandit feedback settings.
We denote the empirical mean of the rewards sampled from entry $A_{ij}$ by $\Bar{A^t_{ij}}$, and the number of times up to round $t$ that the row player has selected action $i$ while the column player has selected action $j$, by $n_{ij}^t \in [t] \cup \{0\}$.

    %    \lss{highlight the difference and similarity between Thompson Sampling and CoEBL}
       \begin{algorithm}[!ht]
       \caption{\coebl for matrix games}
       \label{alg:CoEBL}
       \begin{algorithmic}[1]
       \REQUIRE $\text{Fitness}(x,B):=\min_{y\in \Delta_m} y^T B x$ where $B \in \mathbb{R}^{m \times m}$ and $x \in \Delta_m$. 
       % \REQUIRE Fitness function: $\text{Fit}(x,B,t):=\min_{y\in \Delta_m}y^T \bar{A}^t x$
       % \REQUIRE Fitness function: $\text{Fit}(x,t):=\min_{y\in \Delta_m}y^T \tilde{A}^t x$
       \STATE \textbf{Initialisation:} $x_0, y_0=  {(1/m,\ldots ,1/m)}$ and  $n^0_{ij}=0$ for all $i,j \in [m]$
       \FOR{round $t = 1, 2, \ldots, T$}
           \FOR{all $i,j \in [m]$} 
           \STATE Compute $\tilde{A}^t_{ij} = \text{Mutate}(\Bar{A^t_{ij}}, {1}/{1 \lor n^t_{ij}})$
          %  \STATE Compute $\tilde{B}^t_{ij} = \Bar{A^t_{ij}} + \sqrt{{2 \log (2T^2 m^2)}/{1 \lor n^t_{ij}}}$
           \ENDFOR
           \STATE Obtain $x' \in \arg\max_{x \in \Delta_m} \min_{y \in \Delta_m} y^T \tilde{A}^t x$
           \IF{$\text{Fitness}(x',\tilde{A}^t) > \text{Fitness }(x_{t-1},\tilde{A}^t)$}
           \STATE  Update policy $x_t:=x'$   
           \ELSE
           \STATE Update policy $x_t:=x_{t-1}$
           \ENDIF
           \STATE  Update the query number of each entry in the payoff matrix $n^t_{ij}$ for all $i,j \in [m]$
       \ENDFOR
       \end{algorithmic}
       \end{algorithm}
    
       The following mutation variant is considered in this paper.
    
       \begin{align*}
             \text{Mutate}(\Bar{A^t_{ij}}, \frac{1}{1 \lor n^t_{ij}}) 
             =&\quad \Bar{A^t_{ij}} \\
              +& \Nr \left(\sqrt{\frac{c \log (2T^2 m^2)}{1 \lor n^t_{ij} +1}},\frac{1}{(1 \lor n^t_{ij})^2} \right),
       \end{align*}
       where $\Nr \left(\mu,\sigma^2\right)$ denotes a Gaussian random variable with mean $\mu$ and variance $\sigma^2$, and $c$ is some constant with respect to $T$ and $m$.

      %  \par Evolutionary algorithms consist of two main components: variation operators and selection mechanisms.
      %  Variation operators can generate new individuals from the current population, and the selection mechanism samples the best individuals from the population based on fitness values.
      %  In \coebl,  the fitness function $\text{Fitness}(x,B):=\min_{y\in \Delta_m}y^T B x$ is used to evaluate the performance of policy $x$ against the best response of the opponent given payoff matrix $B$.
      %  At the beginning, \coebl employs a Gaussian mutation operator to generate a new estimated payoff matrix $\tilde{A}^t$ and then the mutated policy $x'$ for the row player.
      %  Note that, since the estimated payoff matrix $\tilde{A}$ is always fully accessible to the $x$-player, $x'$ in line 6 can be obtained by solving a linear programming problem \citep{bubeck2015convex,maiti2023query}.
      %  Between lines 7 and 10, we use the fitness function to evaluate the performance of policy $x'$ and compare it with the previous policy $x_{t-1}$.
      %  If the new policy $x'$ strictly outperforms the previous policy $x_{t-1}$, we update the policy $x_t$ to $x'$; otherwise we keep the policy $x_t$ as $x_{t-1}$ to avoid potential perturbation in maximin solution to the linear programming problem.

      Evolutionary algorithms typically consist of two main components: variation operators and selection mechanisms.
      Variation operators generate new candidate policies from the current population, while the selection mechanism retains the most promising policies based on a predefined fitness function.
      In \coebl, we define the fitness function as $\text{Fitness}(x, B) := \min_{y \in \Delta_m} y^T B x$, which evaluates the performance of a policy $x$ against the best response of the opponent, given the payoff matrix $B$.
      Initially, \coebl applies a Gaussian mutation operator to perturb the estimated payoff matrix $\tilde{A}^t$, and generates a mutated policy $x'$ for the row player.
      As the estimated matrix $\tilde{A}$ is fully observable by the $x$-player, the optimal response $x'$ in line 6 is computed by solving a linear programming problem~\citep{bubeck2015convex,maiti2023query}.
      Between lines 7 and 10, the new policy $x'$ is evaluated using the fitness function and compared against the previous policy $x_{t-1}$.
      If $x'$ strictly outperforms $x_{t-1}$, the policy is updated; otherwise, we retain $x_{t-1}$ to avoid potential instability in the maximin solution.

      \par
      The central idea behind \coebl is to adopt the principle of \emph{optimism in the face of uncertainty} (OFU) to explore the action space and exploit the opponent's best response~\citep{bubeck2012regret,lattimore2020bandit}.
      However, unlike traditional bandit algorithms such as those in the \ucb family, \coebl implements \emph{randomised optimism} via evolutionary algorithms.
      Through the variation operator, \coebl generates diverse estimated payoff matrices, which in turn lead to a broader range of diverse candidate policies.
      The selection mechanisms then guide the evolutionary process towards higher fitness.
      
      Although \coebl shares similar priors as Thompson Sampling using Gaussian priors by \citep{agrawal2017near} and Optimism-then-NoRegret (OTN) framework \citep{li2024optimistic}, it differs in two key ways.
      First, our Gaussian prior is defined via upper confidence bounds rather than empirical means alone, with a higher variance term ${1}/{\left(1 + n_{ij}^{t} \right)}$ compared with one by \citet{agrawal2017near} and a different scaling constant than used in OTN \citep{li2024optimistic}.
      Second, \coebl incorporates a selection mechanism, which implicitly adjusts the sampling distribution and is uncommon in Thompson Sampling.
      These two key changes are proven to be beneficial for adversarial matrix game settings in later sections.
      
      While \citet{o2021matrix} has demonstrated that deterministic optimism enables \ucb (Algorithm~\ref{alg:UCB}) to achieve sublinear regret and outperform classic \exptr, and other bandit baselines, we will show that randomised optimism (via evolution) also exhibits sublinear Nash regret.
      More importantly, we will demonstrate that randomised optimism in matrix games can be more effective and adaptive in preventing exploitation by the opponent than deterministic optimism, and thus outperforms existing bandit baselines.
      Specifically, it outperforms the current bandit baseline algorithms for matrix games, including \exptr (Algorithm~\ref{alg:Exp3}), \ucb (Algorithm~\ref{alg:UCB}) and the \exptrni variant (Algorithm~\ref{alg:MGBF}) on the matrix game benchmarks considered in this paper.

    \subsection{Regret Analysis of \coebl}
    \par In this section, we conduct the regret analysis of \coebl in matrix games.
    Before our analysis, we need some technical lemmas. 
    We defer these lemmas to the appendix.
       
       \par We follow the setting in \citep{o2021matrix} and consider the case where there is 1-sub-Gaussian noise when querying the payoff matrix. Assume that given $t\in \mathbb{N}$:
       
       \begin{quote}
           (A): The noise process $\eta_t$ is 1-sub-Gaussian 
           and the payoff matrix satisfies $A \in  [0, 1]^{m \times m}$.
           \label{eq:assumption}
       \end{quote}
       % Note that we restrict $A \in  [0, 1]^{m \times m}$ in the analysis for simplification. 
       % However, the proof works for any bounded $A \in  [-b, b]^{m \times m}$ where $b$ is constant with respect to $T$ and $m$ by simply shifting from $[-b,b]$ to $[0,2b]$ and normalising the entries in $[0,1]$.
       
       % \begin{restatable}[Main Result I]{theorem}{SecThreeTheZero}
       % \label{thm:mainzero}
       %    Suppose Assumption (A) holds with $m>1$ and $T \geq 2m^2$, then let $\A$ be the class of matrix games defined by (A) and $\D$ be the class of algorithms defined by Algorithm~\ref{alg:BBA}.
       %    Then, we have
       %    \begin{align*}
       %       \minimaxre(\A,\D, T) = \Omega(\sqrt{m^2T}).
       %    \end{align*}
       % \end{restatable}
    
       % Theorem~\ref{thm:mainzero} shows the limitation of algorithms can achieve in Nash regret for matrix games with bandit feedback.
    
       \begin{restatable}{lemma}{SecThreeLemThree}
       \label{lem:HoeffdingBound}
          
       Suppose Assumption (A) holds with $T\geq 2m^2 \geq 2$ and $\delta:=\left( 1/2T^2m^2 \right)^{c/8}$ where $c>0$ is the mutation rate in \coebl. 
       For each iteration $t \in \mathbb{N}$, given $\tilde{A}^t$ in Algorithm~\ref{alg:CoEBL}, we have:
       \begin{equation}
           \Pr \left(A_{ij}- (\tilde{A_t})_{ij} \leq 0 \right) \geq 1 -\delta, \quad \text{ for all $i, j \in [m]$}. \label{eq:Hoeffding}
       \end{equation}
       \end{restatable}

    \begin{restatable}[Main Result]{theorem}{SecThreeMain}
       % \begin{theorem}[Main Result]
       \label{thm:main}
       Consider any two-player zero-sum matrix game. 
       Under Assumption (A) with $T \geq 2m^2 \geq 2$ and $\delta = \left(1/2T^2m^2\right)^{c/8}$, where $c > 0$ is the mutation rate in \coebl, the worst-case Nash regret of \coebl for $c \geq 8$ is bounded by $2 \sqrt{2c T m^2 \log(2T^2m^2)}$, i.e., $\tilde{O}(\sqrt{m^2T})$.
       \end{restatable}

       \begin{proof}[Sketch of Proof]
       Due to page limit, we defer the full proofs of Lemma~\ref{lem:HoeffdingBound} and Theorem~\ref{thm:main} to the appendix and provide a simple proof sketch here.
       First, we bound the regret under the case where all the entries of the estimated payoff matrix are greater than those of the real, unknown payoff matrix (this event is denoted by $E_t^c$ at iteration $t \in \mathbb{N}$).   
       Secondly, we use the law of total probability to consider both cases: when all the entries of the estimated payoff matrix are greater than the real payoff matrix, and the converse (i.e., event $E_t$).
       We already have the upper bound for the first part; the second part can be trivially bounded by $1$ in each iteration.
       Using Lemma~\ref{lem:HoeffdingBound}, we can obtain the upper bound of probability of event $E_t$.
       Combining these bounds provides us with the upper bound for the regret of \coebl.
       \end{proof}
    
      %  Theorem~\ref{thm:main} shows that, under the worst-case scenario (assuming the best response of the opponent across all the possible matrix game instances under Assumption (A)), \coebl can also exhibit sublinear regret. 
      %  More precisely, the regret of \coebl is bounded by $\tilde{O}(\sqrt{m^2T})$, which is the same as the regret bound of \ucb. 
      %  This implies that deterministic optimism in the face of uncertainty is not the crucial factor for achieving sublinear regret, as discussed in \citep{o2021matrix}.
      %  The current results considers $c \geq 8$ in the analysis due to current technical limitations.
      %  We conjecture that the regret bound can be improved by considering smaller $c$ values, 
      %  and thus, in practical use, we suggest one may need hyperparameter tuning in various problems.
      %  Additionally, as we will show later, randomised optimism via evolution can be more robust than deterministic optimism in game playing, and therefore \coebl\ outperforms the other algorithms in the following benchmarks.
      %  Technically, it provides a more straightforward analysis by simply employing Hoeffding's inequality instead of anti-concentration inequality of maximum of Gaussian random variables used for adversarial regret analysis of optimistic Thompson Sampling \citep{li2024optimistic}.

      Theorem~\ref{thm:main} demonstrates that, under the worst-case scenario (assuming the best response of the opponent across all possible matrix game instances under Assumption (A)), \coebl\ achieves sublinear regret. 
      Specifically, the regret of \coebl is bounded by $\tilde{O}(\sqrt{m^2T})$, matching the regret bound of \ucb. This result suggests that deterministic optimism in the face of uncertainty, as highlighted in \citep{o2021matrix}, is not the sole determinant for achieving sublinear regret. In fact, it indicates that the mechanism of optimism: whether deterministic or stochastic is not necessarily critical to the asymptotic regret guarantees in adversarial settings.

      However, a key distinction lies in the practical robustness of randomised (stochastic) optimism via evolution compared to deterministic optimism. As we demonstrate in later sections, randomised optimism offers greater adaptability and robustness in game-playing scenarios, enabling \coebl to outperform other algorithms in benchmark tasks. 
      Intuitively, stochastic optimism facilitates a more balanced exploration-exploitation trade-off by incorporating upper-confidence-bound randomness, which helps prevent premature convergence to sub-optimal strategies and allows the algorithm to respond more flexibly to dynamic and adversarial environments.

      This analysis highlights the potential advantages of stochastic methods in algorithm design for complex environments.
      Notably, the current analysis of \coebl\ assumes $c \geq 8$ due to technical constraints. We conjecture that the regret bound can be improved by considering smaller values of $c$, which may further enhance the practical performance of \coebl. Therefore, we recommend hyper-parameter tuning to optimise performance across various problem settings.

      In conclusion, the results highlight that stochastic optimism, as implemented through evolutionary methods in \coebl, provides both theoretical guarantees and practical advantages, making it more robust than deterministic optimism in certain adversarial and game-theoretic scenarios.

\section{Empirical Results}

       In this section, we present empirical results comparing the discussed algorithms.
       We are interested in empirical regret in specific game instances, measured by cumulative (absolute) regret, 
       i.e.,
       \begin{align}
           \sum_{t=1}^T |V_{A}^* - r_t| \quad  \text{and}  \quad \sum_{t=1}^T V_{A}^* - r_t \label{eq:absoluteRegret}
       \end{align}
       where $r_t$ is the obtained reward at round $t$.
       We focus on two scenarios, including self-play and $\alg~1$-vs-$\alg~2$.
       In the self-play scenario, both row and column players use the same algorithm with the same information. 
       We use the absolute regret (the first metric) to measure the performance of the algorithms in this case.
       The $\alg~1$-vs-$\alg~2$ is a generalisation of the self-play scenario. 
       We use the second metric in Eq.~\ref{eq:absoluteRegret} to measure the performance of the algorithms.
       The $\alg~1$-vs-$\alg~2$ means the row player uses $\alg~1$, and the column player uses $\alg~2$ with the same information. 
       As in the setting of \citep{o2021matrix}, 
       the plots below show the regret (not absolute regret) from the maximiser's (\alg~1) perspective. 
       A positive regret value means that the minimiser (\alg~2) is, on average winning and \textit{vice versa}.
       This allows us to compare our algorithms directly.

       \par Moreover, to measure how far the players are from the Nash equilibrium, we use the KL-divergence between the policies of both players and the Nash equilibrium or the total variation distance (for the case where the KL-divergence is not well-defined), i.e., $ \text{KL}(x_t,x^*)+\text{KL}(y_t,y^*) $ and $\text{TV}(x_t,x^*)+\text{TV}(y_t,y^*)$, where
       \begin{align}
           \text{KL}(a, b) &:= \sum_{i} a(i) \ln \left( \frac{a(i)}{b(i)}\right) \nonumber \\
           \text{TV}(a,b) &:= \frac{1}{2}\sum_{i} |a(i)- b(i)| \nonumber 
       \end{align}
       for any $a, b\in \Delta_m$ and $(x^*, y^*)$ is the Nash equilibrium of $A$. 
       
       \subsubsection*{Parameter Settings}
       Given $K$ is the number of actions for each player and $T$ is the time horizon,
       for \exptr, we use the exploration rate $\gamma_t=\min \{\sqrt{K \log K /t}, 1\}$ and learning rate $\eta_t=\sqrt{2\log K /tK}$ as suggested in \citep{o2021matrix}.
       For the variant of \exptrni, we use the same settings $\eta_t = t^{-k_\eta}$, $\beta_t = t^{-k_\beta}$, $\epsilon_t = t^{-k_\epsilon}$ where $k_\eta = {5}/{8}$, $k_\beta = {3}/{8}$, $k_\epsilon = {1}/{8}$ as suggested in \citep{cai2024uncoupled}.
       For \coebl, we set the mutation rate $c=2$ for the RPS game and $c=8$ for the rest of the games.
       There is no hyper-parameter needed for \ucb.
       For the observed reward, we consider standard Gaussian noise with zero mean and unit variance, 
       i.e. $r_t=A_{i_t,j_t}+\eta_t$ where $\eta_t \sim \Nr (0,1)$.
       We compute the empirical mean of the regrets and the KL-divergence (or total variation distance), and present the $95\%$ confidence intervals in the plots.
       We run $50$ independent simulations (up to $3000$ iterations) for each configuration (over $50$ seeds).
       
       \subsection{Rock-Paper-Scissors Game}
       We consider the classic matrix game benchmark:  Rock-Paper-Scissors games~\citep{littman1994markov,o2021matrix}, and its payoff matrix is defined as follows.
       
       \begin{table}[H]
           \centering
           \begin{tabular}{c|c c c}
              & R & P & S \\ \hline
           R  & 0 & 1 & -1 \\
           P  & -1 & 0 & 1 \\
           S  & 1 & -1 & 0 \\
           \end{tabular}
           \caption{The payoff matrix of RPS game. R denotes rock, P denotes paper, and S denotes scissors.}
           \end{table}
           
           It is well known that $x^*,y^*=(1/3,1/3,1/3)$ is the unique mixed Nash equilibrium of the RPS game for both players.
           We conduct experiments using
           Algorithms~\ref{alg:Exp3} to~\ref{alg:MGBF} and compare them with our proposed Algorithm~\ref{alg:CoEBL} 
           (i.e. \coebl) on the classic matrix game benchmark: the RPS game.
       
           \begin{figure}[H]
               \centering
            \subfloat{
                \includegraphics[width=0.32\linewidth]{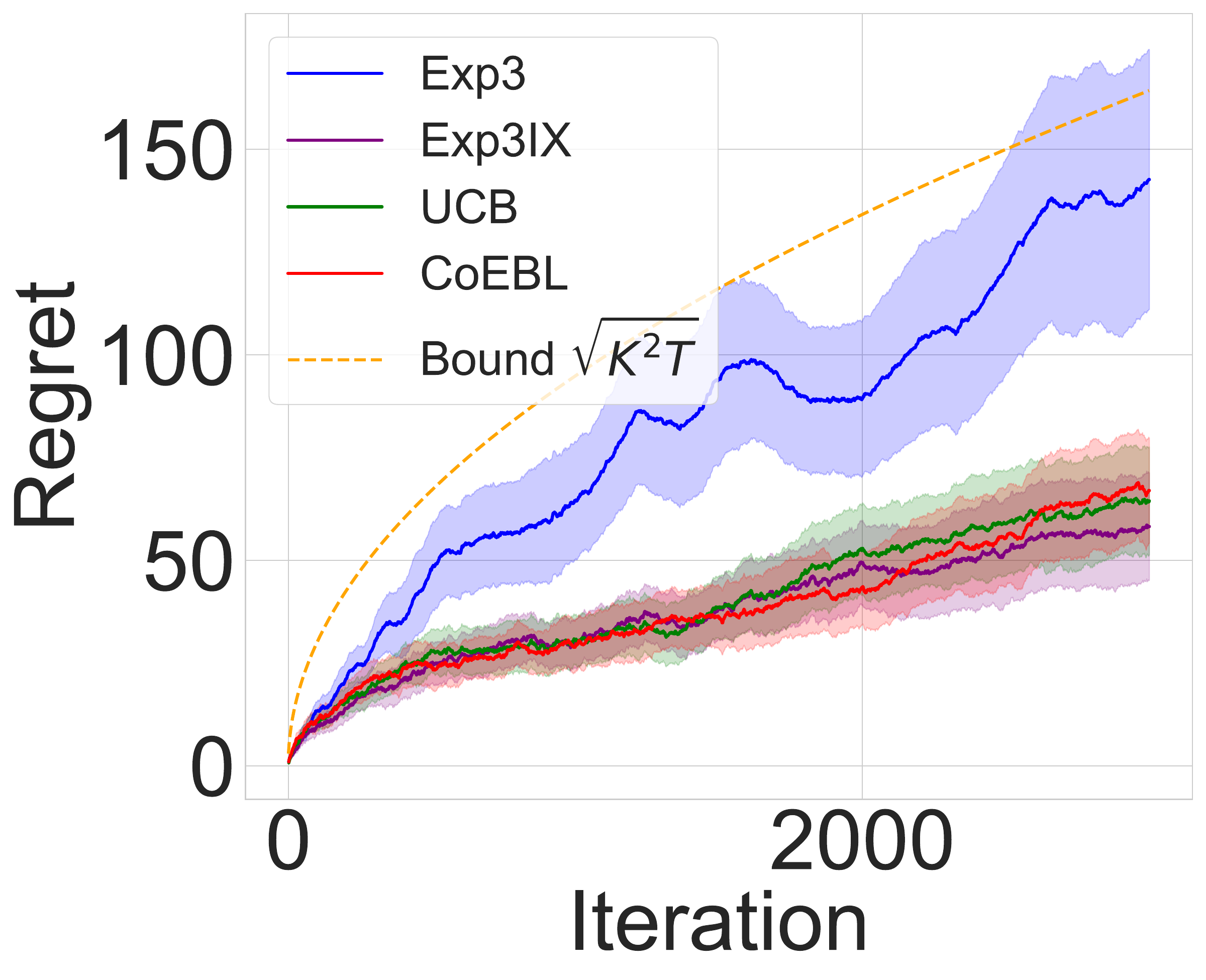}
            }
            \hfill
            \subfloat
            % [KL-Divergence among all algorithms]
            {%
                   \includegraphics[width=0.32\linewidth]{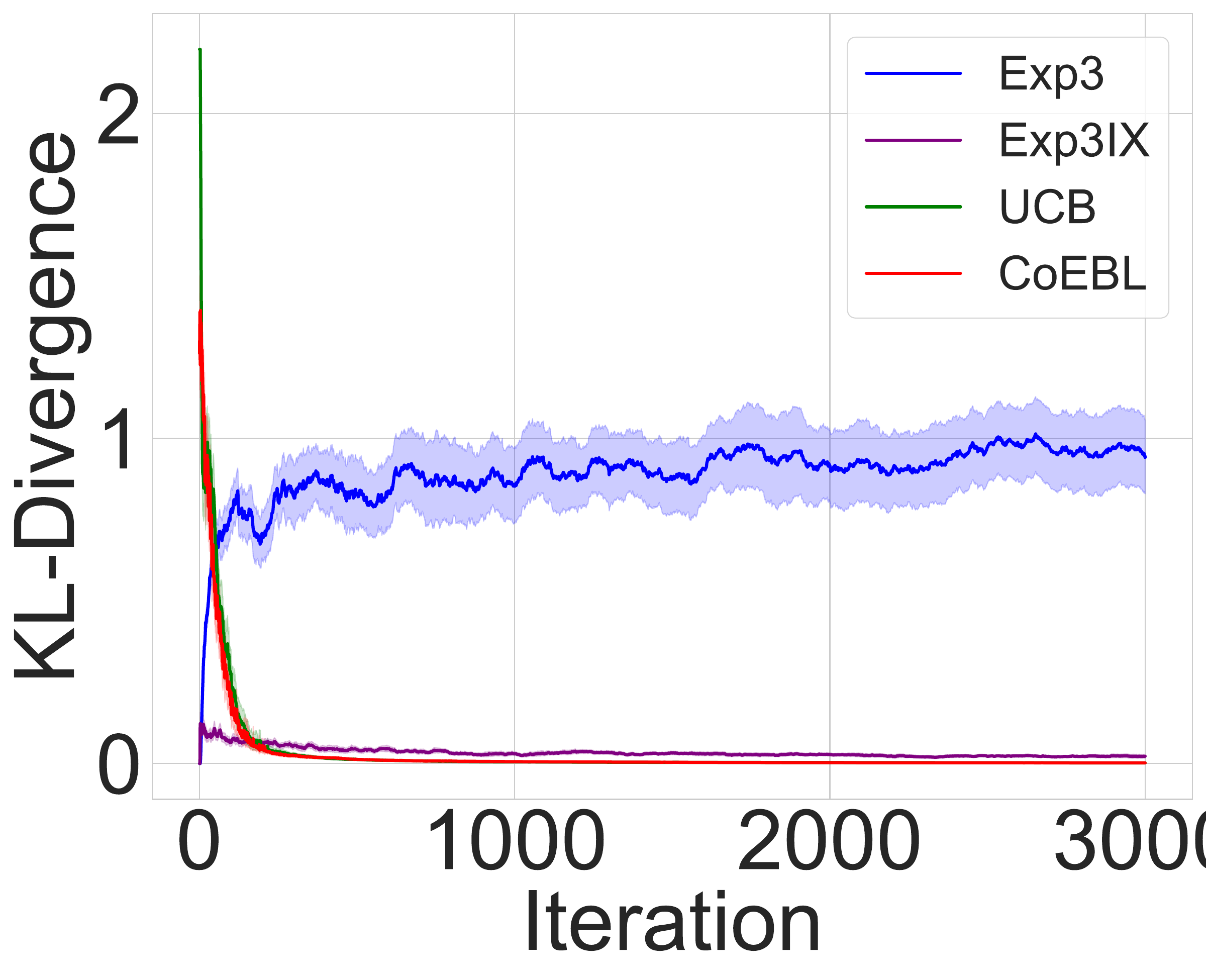}}
               \hfill
             \subfloat
             % [Zoom-in for the detailed comparison]
             {%
                   \includegraphics[width=0.32\linewidth]{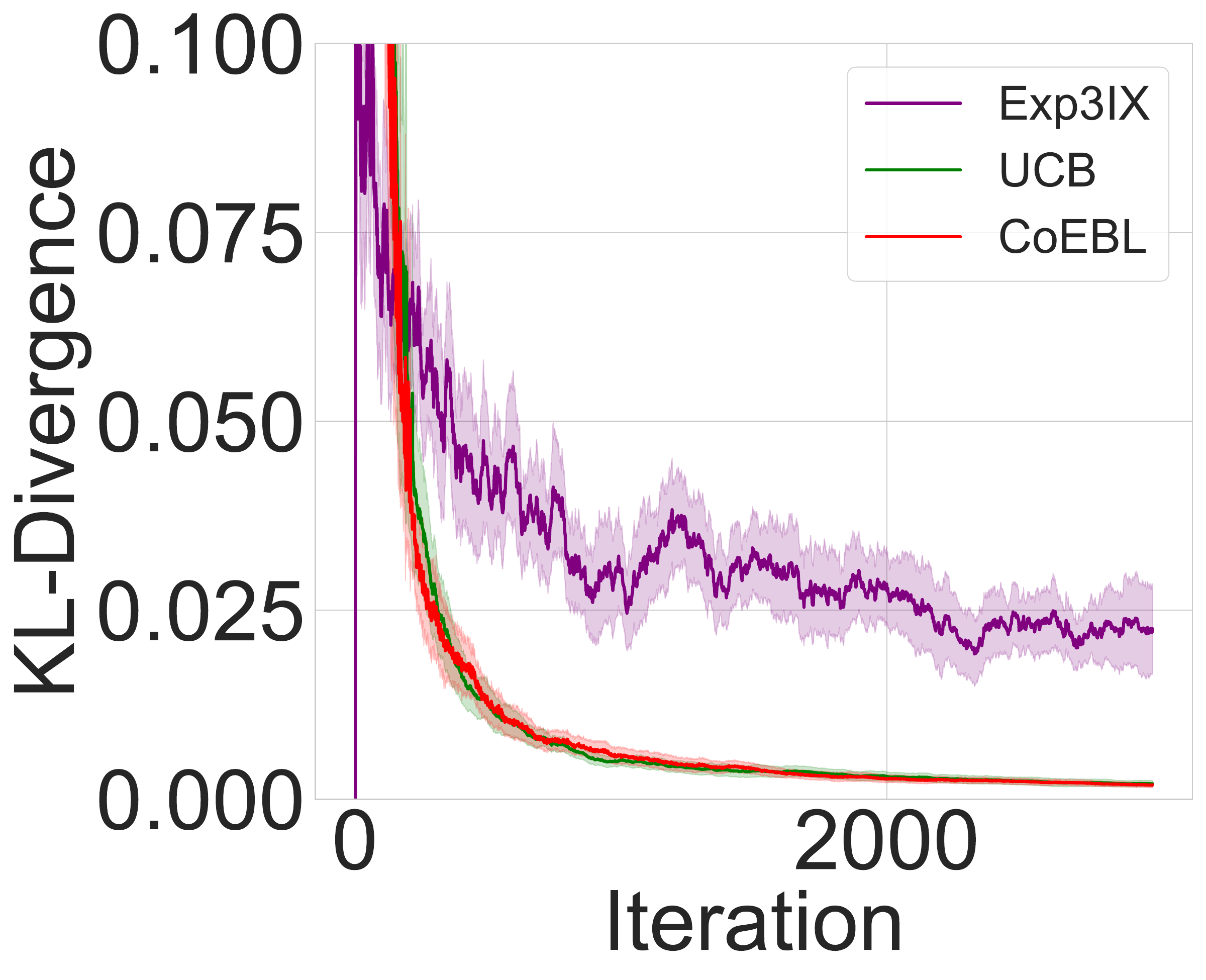}}
             \caption{Regret and KL-divergence for Self-Plays on RPS games}
             \label{fig:KLdivergence_RPS} 
           \end{figure}

           \par In Figure~\ref{fig:KLdivergence_RPS}, we present the self-play results of each algorithm. 
           We can observe that \coebl also exhibits sublinear regret in the RPS game, similar to other bandit baselines, and matches our theoretical bound.
           In terms of the KL-divergence, \exptr, as reported in \citep{o2021matrix,cai2024uncoupled}, diverges from the Nash equilibrium.
           By zooming in on the KL-divergence plot, we can observe that \coebl and \ucb converges to the Nash equilibrium faster than the other algorithms; especially, \EXPTrNi has a much slower convergence rate.

           \par Next, we compare the performance of the algorithms by examining their regret bounds and KL-divergence from the Nash equilibrium when algorithms compete with each other using the same information. 
           % The legend indicates ‘\alg~1 vs \alg~2’ for various choices of \alg~1 and \alg~2, where \alg~1 acts as the maximiser and \alg~2 as the minimiser. 
           % As the same setting in \citep{o2021matrix}, 
           % the plots below show the regret (not absolute regret) from the maximiser's (\alg~1) perspective. 
           % A positive regret value means that the minimiser (\alg~2) is, on average, winning and vice versa.
           In Figure~\ref{fig:Regret_RPS2}, we can clearly observe that \coebl outperforms the \exptr family, including \exptr and \exptrni, in terms of regret.
           On average, \coebl has a smaller advantage over \ucb in terms of regret, since the empirical mean of regret is above $5$ but below $10$ after iteration $2000$.
           
           \begin{figure}[!htpt]
               \centering
             % \subfloat[Exp3 vs Exp3-IX]{%
             %      \includegraphics[width=0.48\linewidth]{figures/RPS/regret_plot_exp3_vs_exp3_ix_rps.pdf}}
             % \subfloat[Exp3 vs UCB]{%
             %       \includegraphics[width=0.47\linewidth]{figures/RPS/regret_plot_exp3_vs_ucb_rps.pdf}}
               % \hfill
             \subfloat[Exp3 vs \coebl]{%
                  \includegraphics[width=0.34\linewidth]{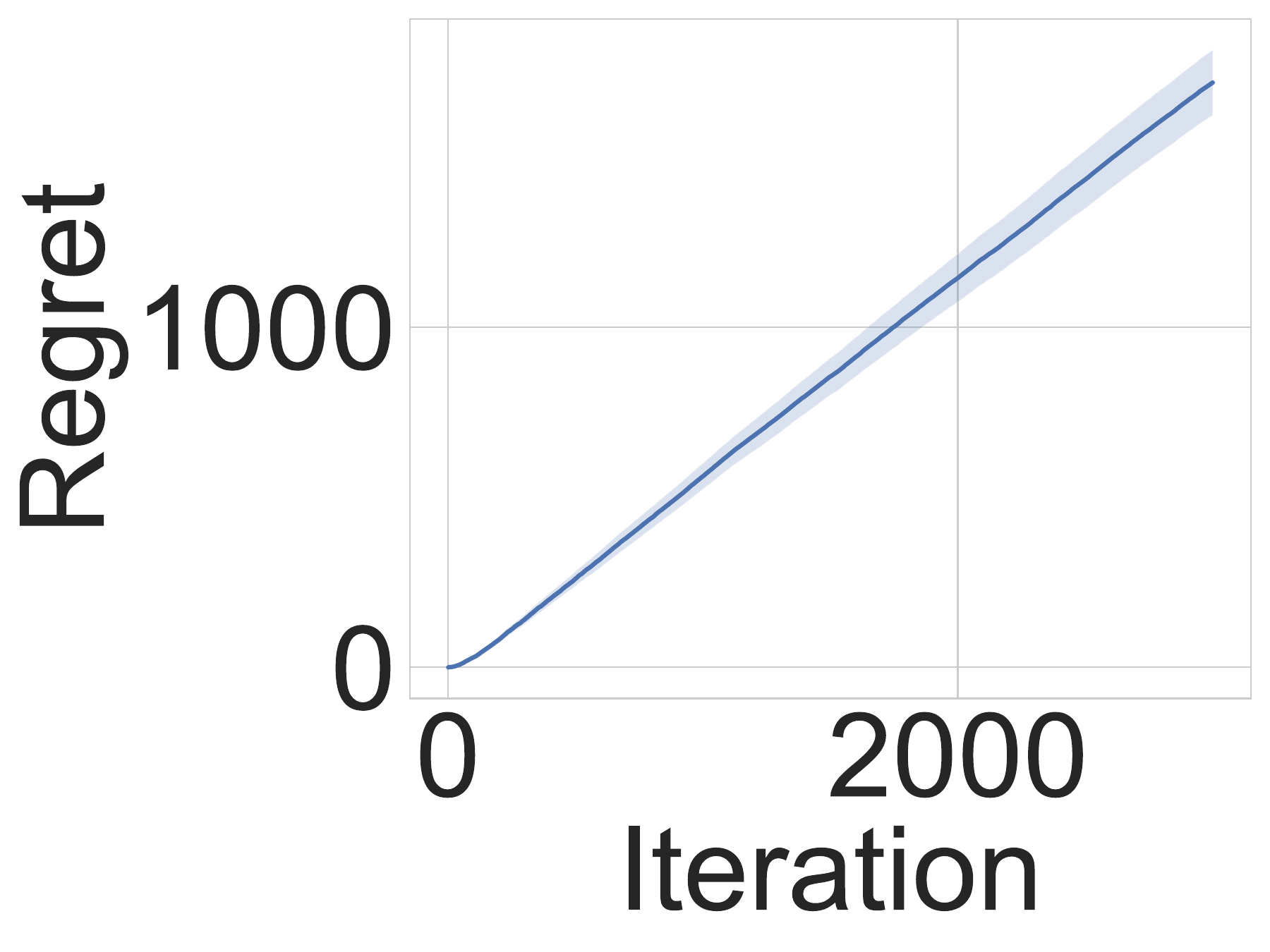}}
              %  \hfill
             \subfloat[Exp3-IX vs \coebl]{%
                   \includegraphics[width=0.34\linewidth]{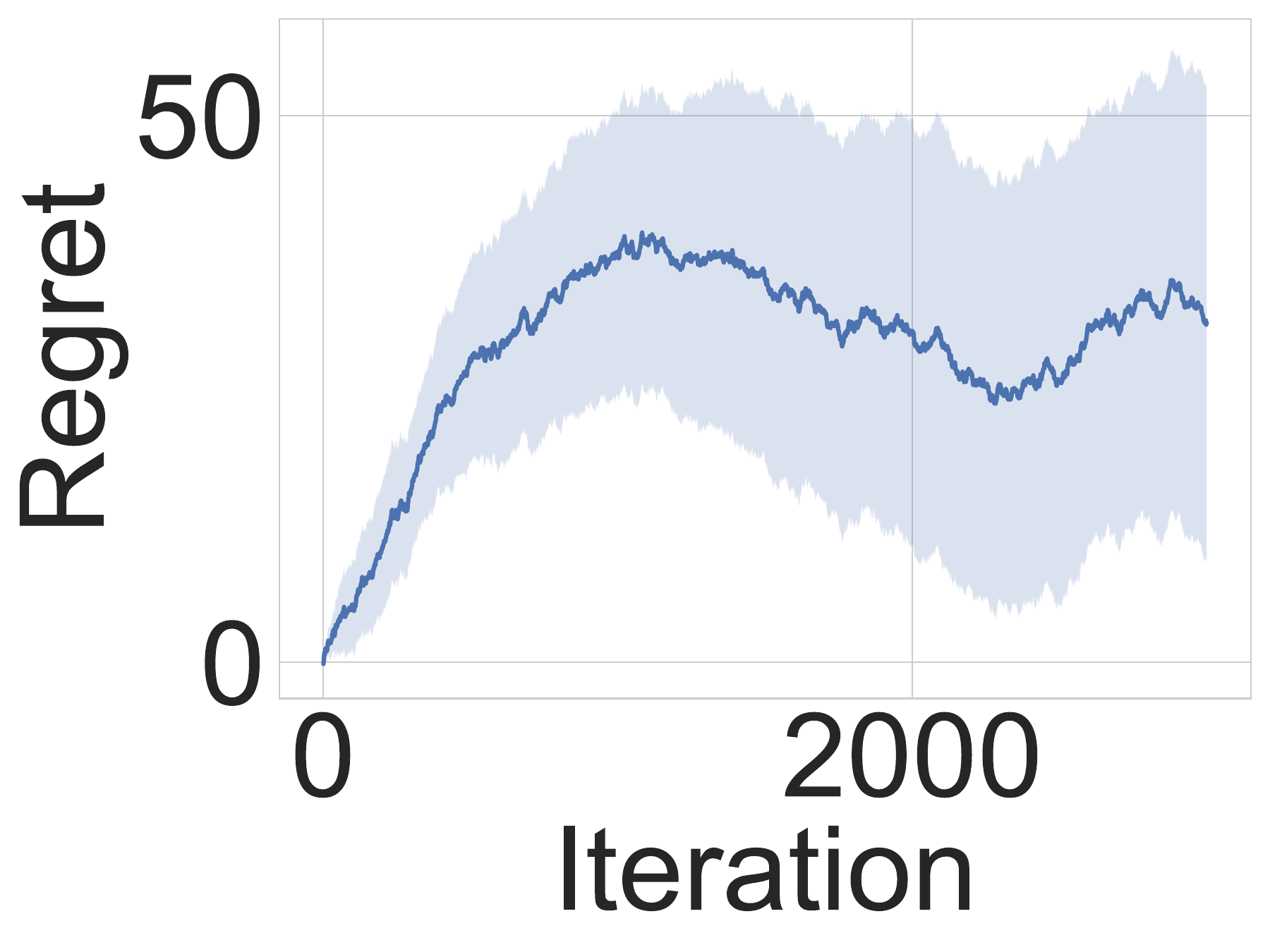}}
              %  \hfill
             \subfloat[UCB vs \coebl]{%
                   \includegraphics[width=0.34\linewidth]{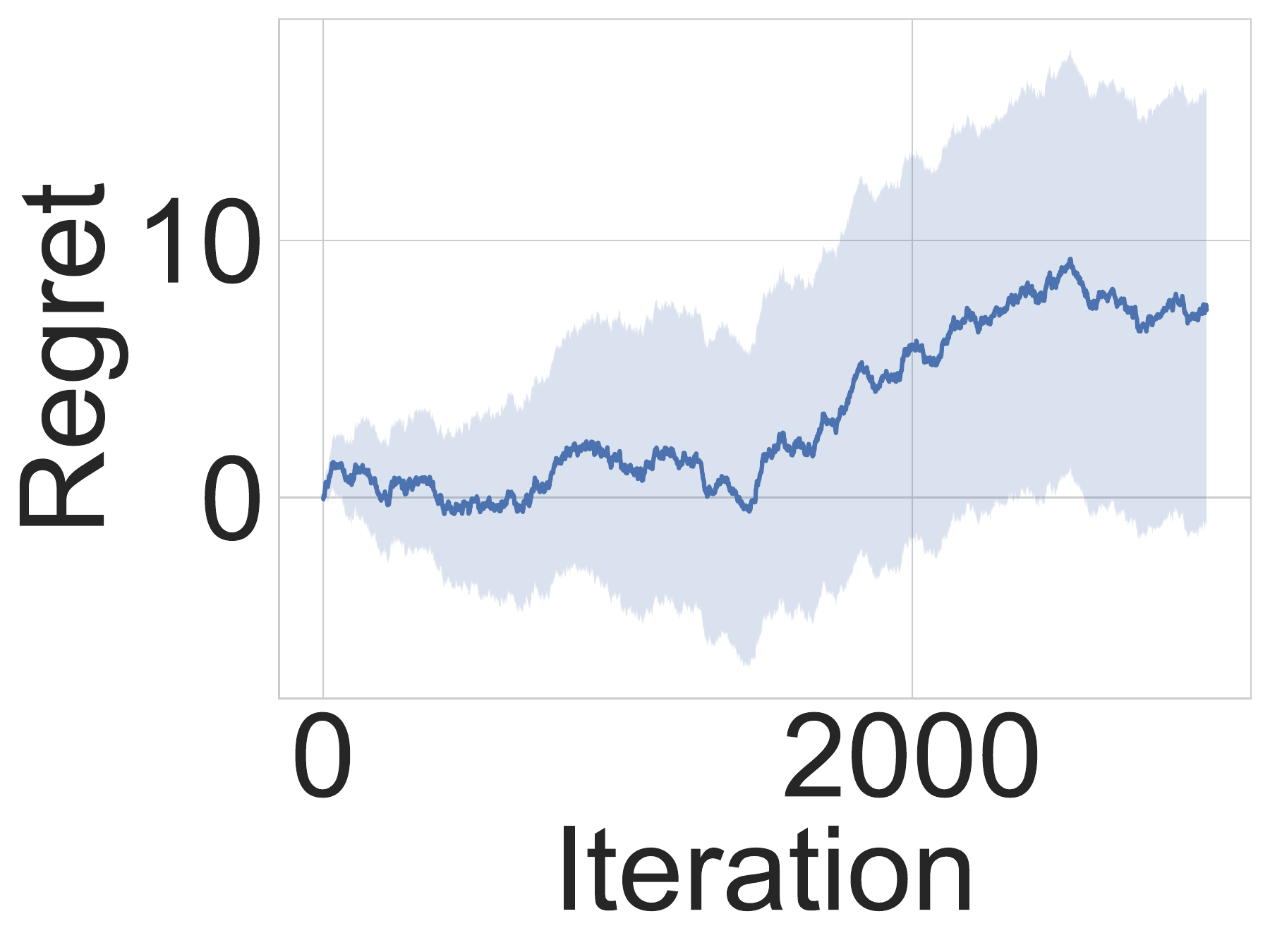}}
           %     \hfill
           %     \subfloat[Exp3 vs \coebl]{%
           %     \includegraphics[width=0.33\linewidth]{figures/RPS/kl_divergence_from_ne_plot_exp3_vs_coebl_rps.pdf}}
           %    % \hfill
           %    \subfloat[Exp3-IX vs \coebl]{%
           %      \includegraphics[width=0.33\linewidth]{figures/RPS/kl_divergence_from_ne_plot_exp3_ix_vs_coebl_rps.pdf}}
           %  % \hfill
           %    \subfloat[UCB vs \coebl]{%
           %      \includegraphics[width=0.33\linewidth]{figures/RPS/kl_divergence_from_ne_plot_ucb_vs_coebl_rps.pdf}}
            % \hfill
             \caption{Regret for $\alg~1$-vs-$\alg~2$ on RPS games}
             \label{fig:Regret_RPS2} 
           \end{figure}
       
           The RPS game with a small number of actions is relative simple for these algorithms to play.
           Moreover, although \coebl completely outperforms the \exptr family, it does not have an overwhelming advantage over the \ucb. 
           How do these algorithms behave on more complex games with exponentially many actions?
           Can \coebl still take over the game? 
           Next, we answer these questions by considering \Diagonal and \BigNum games.
       
           \subsection{\Diagonal Game}
           \Diagonal is a pseudo-Boolean maximin-benchmark on which \citet{lin2024overcoming} conducted runtime analysis of coevolutionary algorithms.
           Both players have an exponential number (i.e. $2^n$) of pure strategies.
           To distinguish between pure strategies that consist of the same number of $1$, 
           we modify the original \Diagonal by introducing a `draw' outcome.
           For $\U=\{0,1\}^n$ and $\V=\{0,1\}^n$, the payoff function $\Diagonal:\U \times \V \rightarrow \{0,1\}$ is defined by
           % \begin{align*}
           %     \Diagonal(u,v)
           %     := \begin{cases} 
           %       1 & |v|_1\leq |u|_1   \\
           %       0 & \text{otherwise}
           %    \end{cases}.
           % \end{align*}
           \begin{align*}
               \Diagonal(u,v)
               := \begin{cases} 
                 1 & |v|_1 < |u|_1   \\
                 0 & |v|_1 = |u|_1 \\
                 -1 & \text{otherwise}
              \end{cases}.
           \end{align*}
       
           As shown by \citet{lin2024overcoming}, this game (we provide a simple example in the appendix) exhibits a unique pure Nash equilibrium where both players choose $1^n$. 
           This corresponds to the mixed Nash equilibrium where $x^*=(0, \cdots ,1)$ and $y^*=(0,\cdots ,1)$.
           We conduct experiments using
           Algorithms~\ref{alg:Exp3} to~\ref{alg:MGBF} and compare them with our proposed Algorithm~\ref{alg:CoEBL} 
           (i.e. \coebl) on another matrix game benchmark: the \Diagonal game.
           We set the mutation constant $c=8$ for \coebl and consider $n=2, 3, 4, 5, 6, 7$ in the experiments.
           % We compute the empirical mean of 
       
           \begin{figure}[!ht]
               \centering
               \subfloat[$n=2$
               ]{%
                  \includegraphics[width=0.33\linewidth]{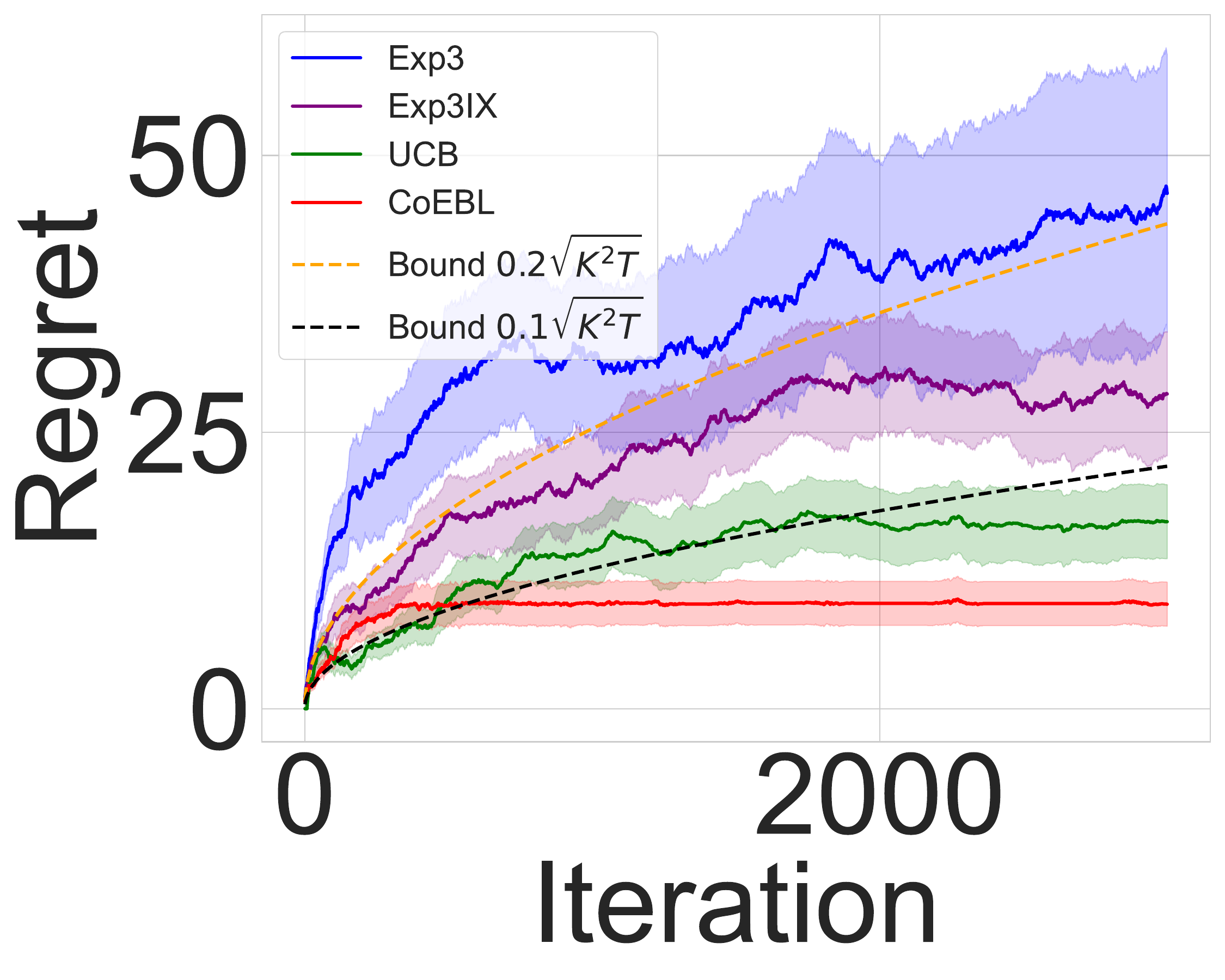}}
               \hfill
               \subfloat[$n=3$
               ]{%
                  \includegraphics[width=0.33\linewidth]{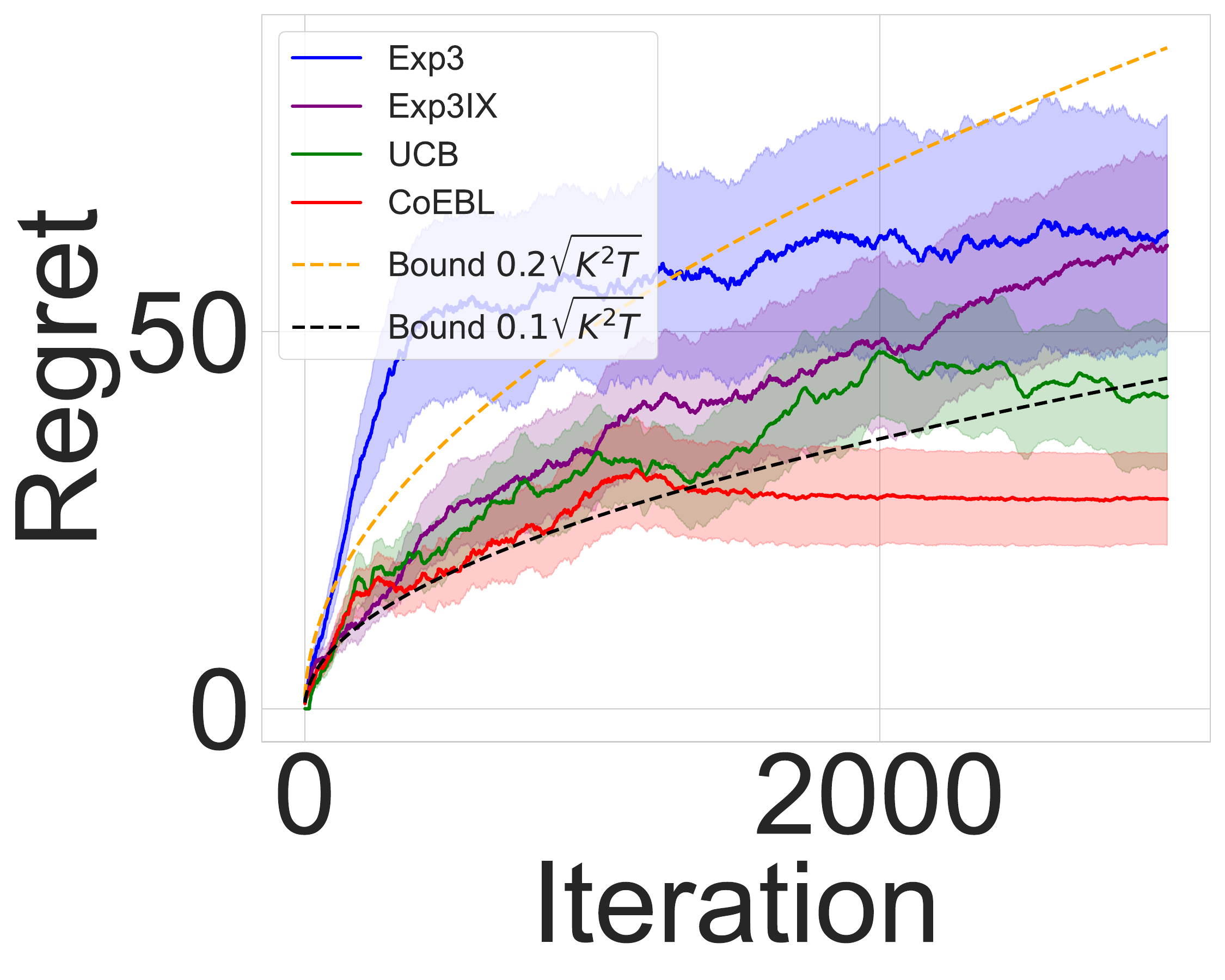}}
               \hfill
               \subfloat[$n=4$
               ]{%
                  \includegraphics[width=0.33\linewidth]{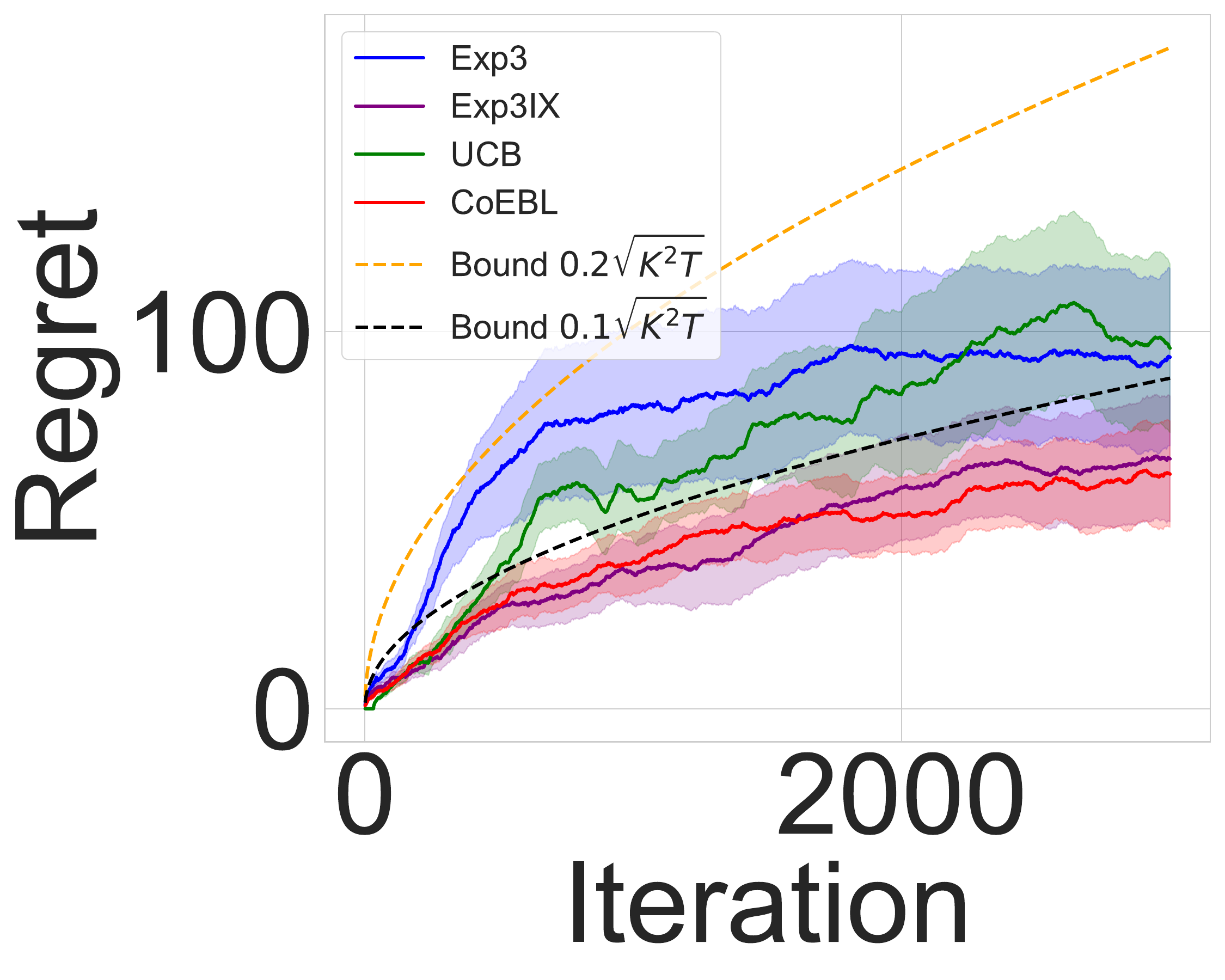}}
               \hfill
              %  \subfloat[$n=5$
              %  ]{%
              %     \includegraphics[width=0.33\linewidth]{figures/Diagonal/SelfPlay/regret_comparison_n5.pdf}}
              %  \hfill
              %  \subfloat[$n=6$
              %  ]{%
              %     \includegraphics[width=0.33\linewidth]{figures/Diagonal/SelfPlay/regret_comparison_n6.pdf}}
              %  \hfill
              %  \subfloat[$n=7$
              %  ]{%
              %     \includegraphics[width=0.33\linewidth]{figures/Diagonal/SelfPlay/regret_comparison_n7.pdf}}
              %  \hfill
               \subfloat[$n=2$
               ]{%
                  \includegraphics[width=0.33\linewidth]{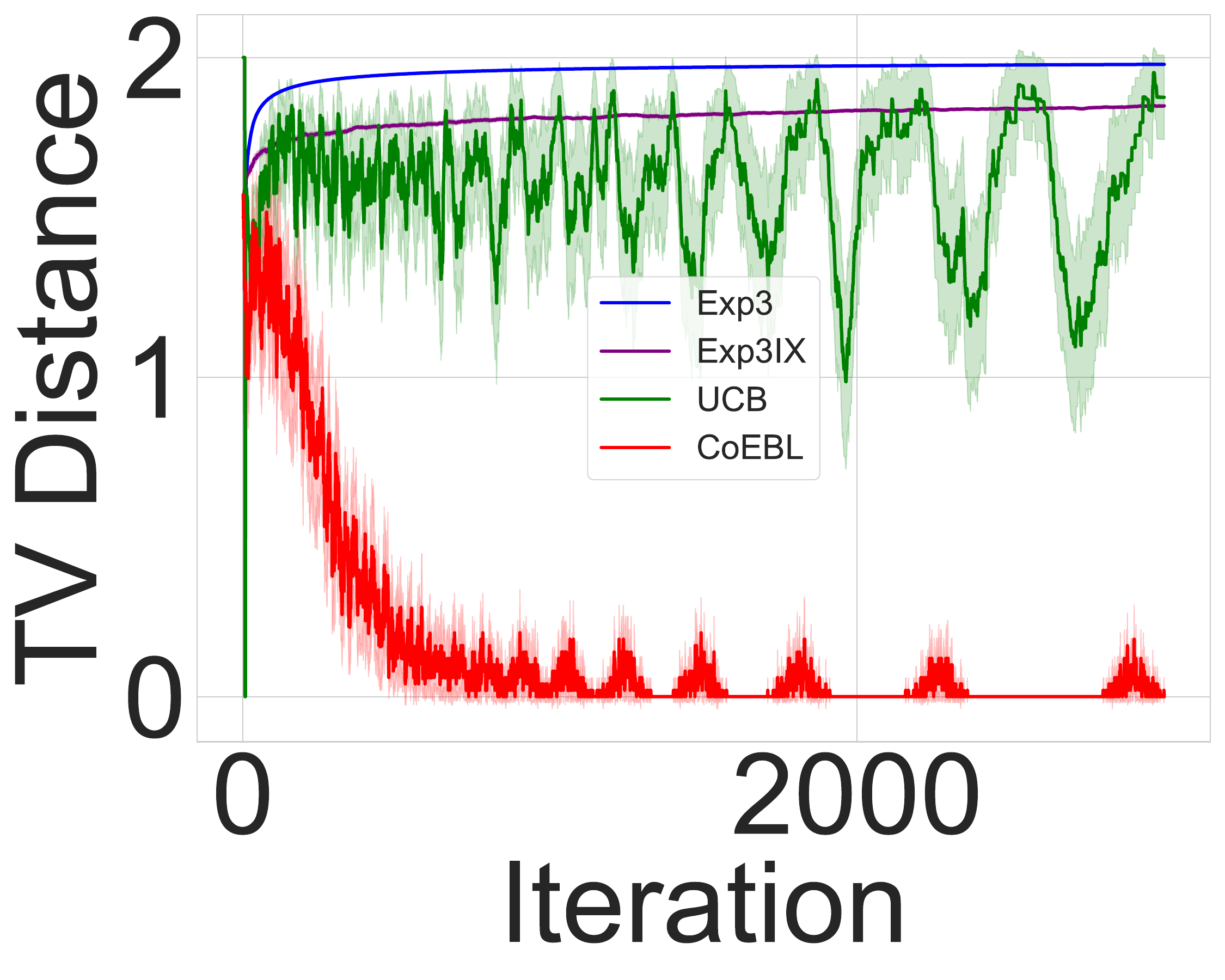}}
               \hfill
               \subfloat[$n=3$
               ]{%
                  \includegraphics[width=0.33\linewidth]{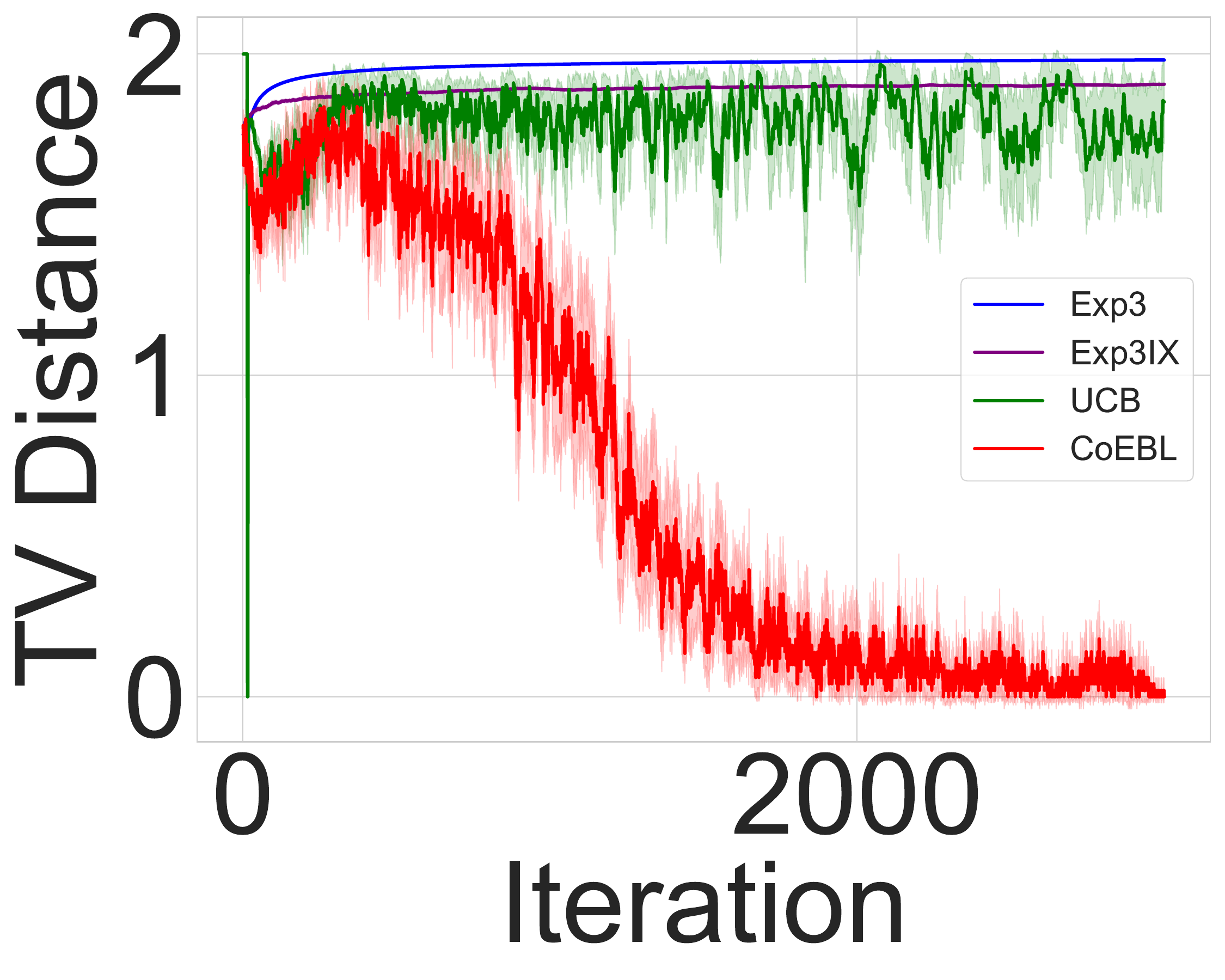}}
               \hfill
               \subfloat[$n=4$
               ]{%
                  \includegraphics[width=0.33\linewidth]{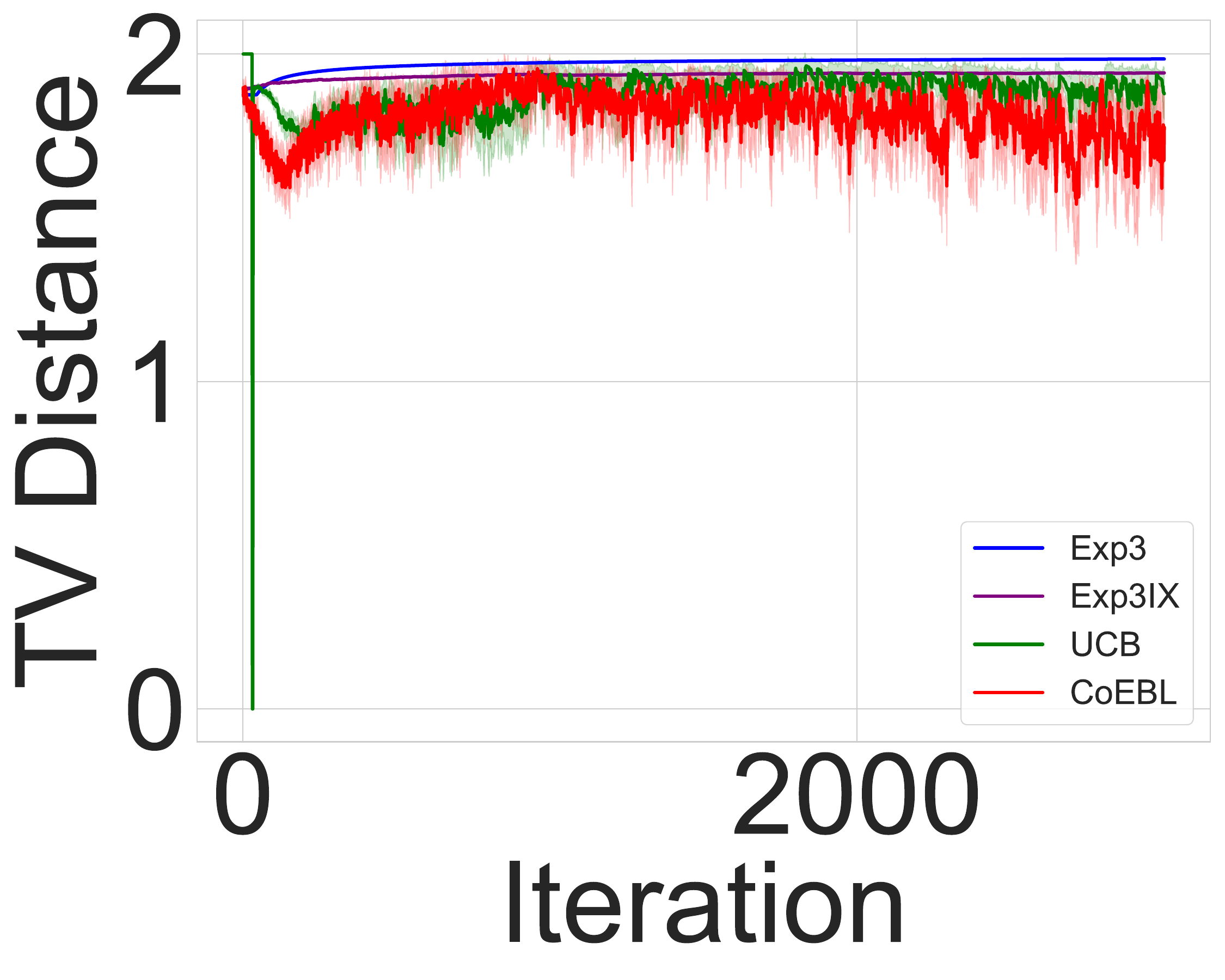}}
              %  \hfill
              %  \subfloat[$n=5$
              %  ]{%
              %     \includegraphics[width=0.33\linewidth]{figures/Diagonal/SelfPlay/tvd_comparison_n5.pdf}}
              %  \hfill
              %  \subfloat[$n=6$
              %  ]{%
              %     \includegraphics[width=0.33\linewidth]{figures/Diagonal/SelfPlay/tvd_comparison_n6.pdf}}
              %  \hfill
              %  \subfloat[$n=7$
              %  ]{%
              %     \includegraphics[width=0.33\linewidth]{figures/Diagonal/SelfPlay/tvd_comparison_n7.pdf}}
               \hfill
               \caption{Regret and TV Distance for Self-Plays on \Diagonal}
               \label{fig:Comparison_Regret_TV_Diagonal} 
           \end{figure} 
       
           In Figures~\ref{fig:Comparison_Regret_TV_Diagonal} and \ref{fig:Comparison_Regret_TV_Diagonal_2}\footnote{Figure~\ref{fig:Comparison_Regret_TV_Diagonal_2} is in the appendix.}, we present the self-play results of each algorithm on \Diagonal game for various values of $n$.
           Our results show that \coebl consistently exhibits sublinear regret in the \Diagonal game, aligning with our theoretical bounds and similar to other bandit algorithms.
           As $n$ increases, the regret of the baseline algorithms grows as expected.   
           \coebl remains more adaptive and robust in more challenging games, maintaining sublinear regret beneath the theoretical bound ($0.1\sqrt{K^2T}$), as indicated by the black dotted line.
           We also observe that the regrets of all algorithms increases as $n$ grows, which is expected due to the exponential increase in the number of pure strategies and the corresponding complexity of the game.
           % Our theoretical bound seems not to be tight for large $n$.
           In terms of convergence measured by TV-distance, \coebl converges to the Nash equilibrium for $n = 2, 3$, while the baseline algorithms do not converge.
           However, for $n \geq 4$, as the number of strategies grows exponentially, \coebl also struggles to converge to the Nash equilibrium.
           In Figures~\ref{fig:Regret_Diagonal2} and \ref{fig:Regret_Diagonal2_2}\footnote{Figure~\ref{fig:Regret_Diagonal2_2} is in the appendix.}, we present the regrets for $\alg~1$-vs-$\alg~2$ on \Diagonal.
           The empirical regrets across all algorithms exceed $16.2$, with a maximum of $389.8$ for $n=6$, indicating that the minimiser is dominant.
           In other words, \coebl outperforms the other bandit algorithms across all values of $n$, from $2$ to $7$.
        
              % Due to space limit, we defer the results on the \BigNum game to the appendix.

          \begin{figure}[!ht]
           \centering
           \subfloat[
            $n=2$
           ]{%
              \includegraphics[width=0.45\linewidth]{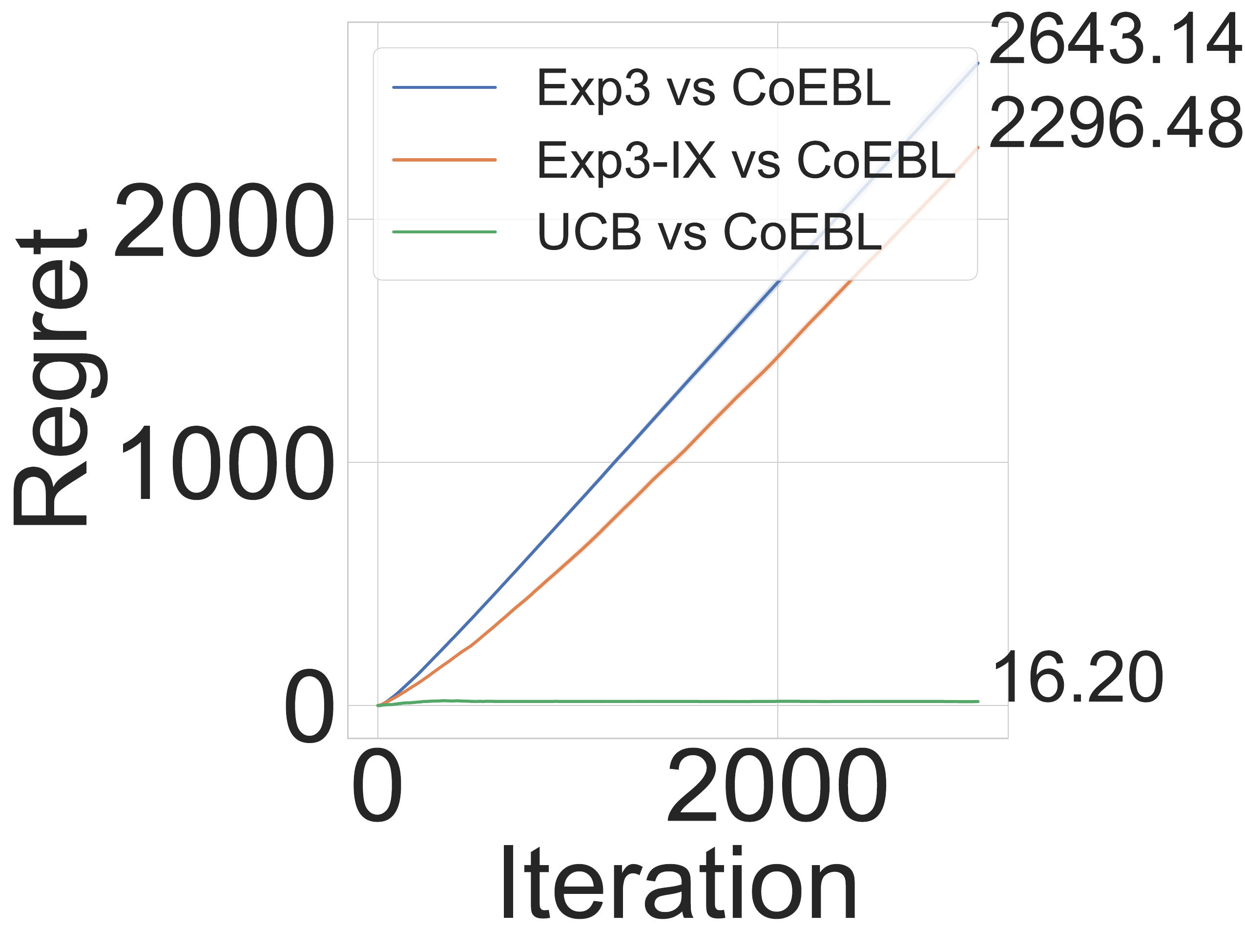}}
           \hfill
           \subfloat[
           $ n=3$
           ]{%
              \includegraphics[width=0.45\linewidth]{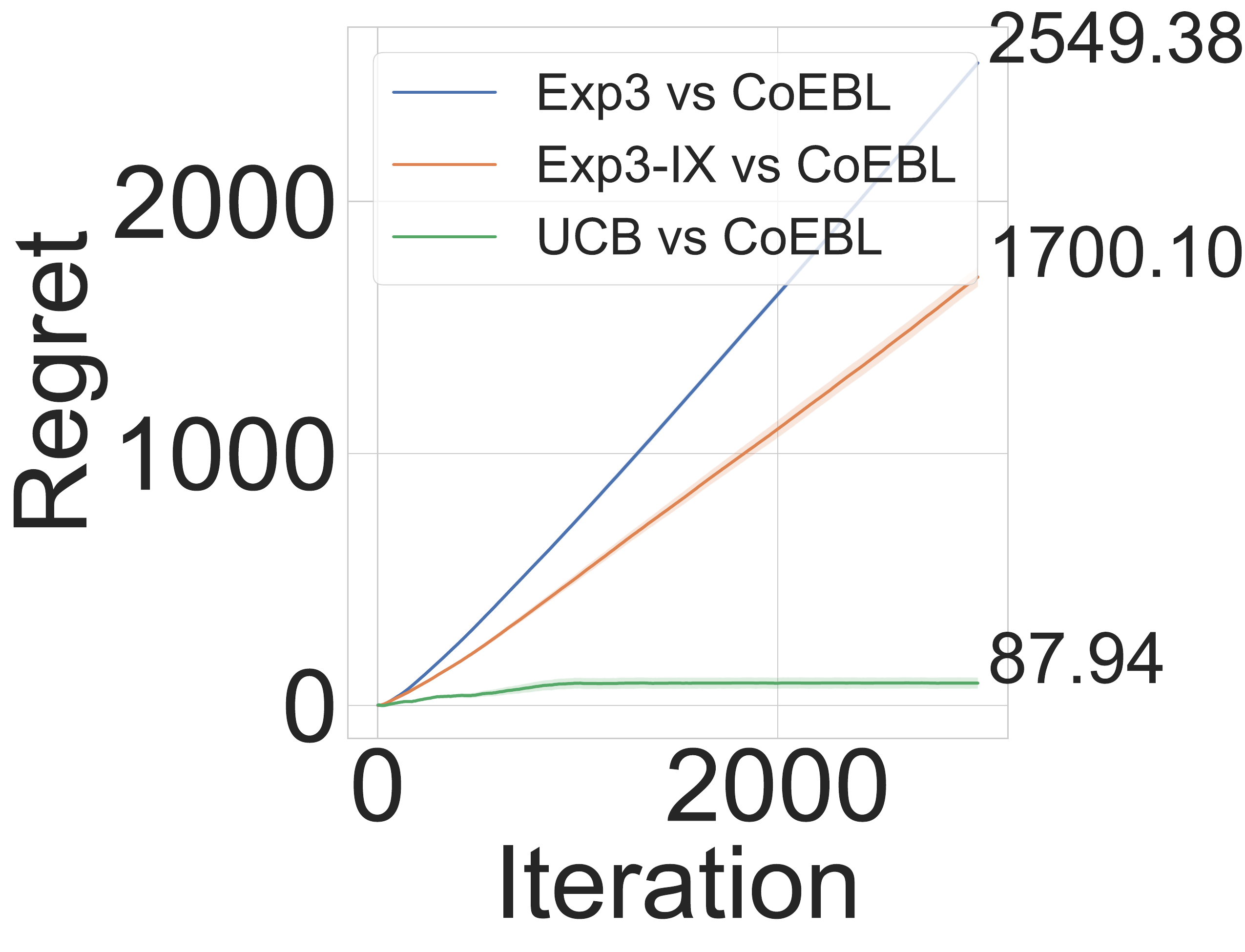}}
           \hfill
           \subfloat[
          $ n=4$
           ]{%
              \includegraphics[width=0.45\linewidth]{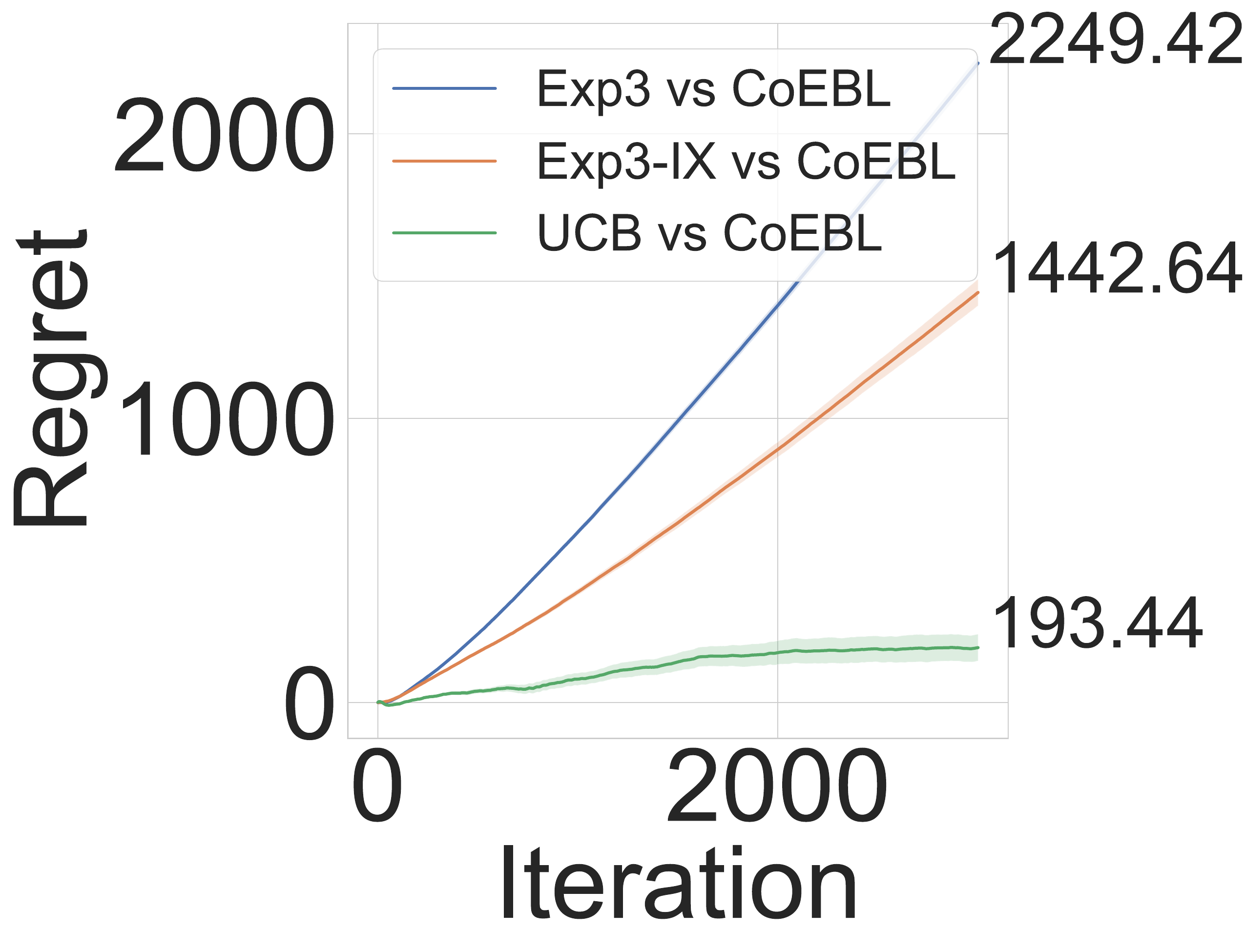}}
          %  \hfill
          %  \subfloat[
          %  % Regret (n=5)
          %  ]{%
          %     \includegraphics[width=0.33\linewidth]{figures/Diagonal/Competition/combined_regret_Competition_plot_n5.pdf}}
          %  \hfill
          %  \subfloat[
          %  % Regret (n=6)
          %  ]{%
          %     \includegraphics[width=0.33\linewidth]{figures/Diagonal/Competition/combined_regret_Competition_plot_n6.pdf}}
          %  \hfill
          %  \subfloat[
          %  % Regret (n=7)
          %  ]{%
          %     \includegraphics[width=0.33\linewidth]{figures/Diagonal/Competition/combined_regret_Competition_plot_n7.pdf}}
          %  \hfill
          %  \subfloat[
          %  % TV (n=2)
          %  ]{%
          %     \includegraphics[width=0.33\linewidth]{figures/Diagonal/Competition/combined_tv_distance_Competition_plot_n2.pdf}}
          %  \hfill
          %  \subfloat[
          %  % TV (n=3)
          %  ]{%
          %     \includegraphics[width=0.33\linewidth]{figures/Diagonal/Competition/combined_tv_distance_Competition_plot_n3.pdf}}
          %  \hfill
          %  \subfloat[
          %  % TV (n=4)
          %  ]{%
          %     \includegraphics[width=0.33\linewidth]{figures/Diagonal/Competition/combined_tv_distance_Competition_plot_n4.pdf}}
          %  \hfill
          %  \subfloat[
          %  % TV (n=5)
          %  ]{%
          %     \includegraphics[width=0.33\linewidth]{figures/Diagonal/Competition/combined_tv_distance_Competition_plot_n5.pdf}}
          %  \hfill
          %  \subfloat[
          %  % TV (n=6)
          %  ]{%
          %     \includegraphics[width=0.33\linewidth]{figures/Diagonal/Competition/combined_tv_distance_Competition_plot_n6.pdf}}
          %  \hfill
          %  \subfloat[
          %  % TV (n=7)
          %  ]{
          %  \includegraphics[width=0.33\linewidth]{figures/Diagonal/Competition/combined_tv_distance_Competition_plot_n7.pdf}}
           \caption{Regret for $\alg~1$-vs-$\alg~2$ on \Diagonal.}
           \hfill
           \label{fig:Regret_Diagonal2} 
       \end{figure}
       
       \subsection{\BigNum Game}
       
       \BigNum is another challenging two-player zero-sum game proposed and analysed by \citep{ijcai2024p336}.
       In this game, each player aims to select a number that is larger than their opponent's. 
       The players' action space is $\X = \{0,1\}^n$, representing binary bitstrings of length $n$ corresponding to natural numbers in the range $[0, 2^n-1]$. 
       A formal definition and the complete results can be found in the appendix.
       We present part of the results here.
       \begin{figure}[H]
          \centering
          \subfloat[$n=2$
          ]{%
             \includegraphics[width=0.33\linewidth]{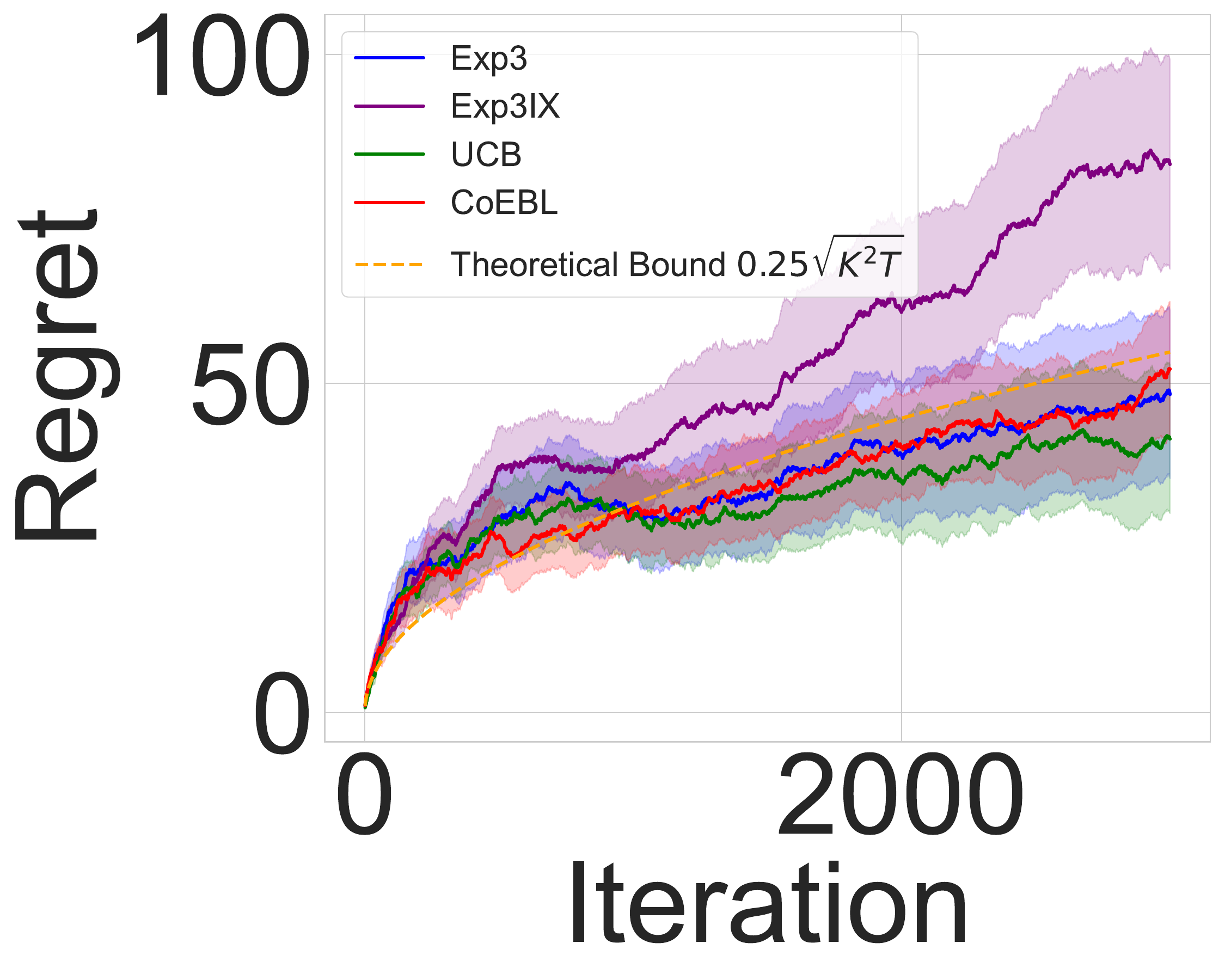}}
          \hfill
          \subfloat[$n=3$
          ]{%
             \includegraphics[width=0.33\linewidth]{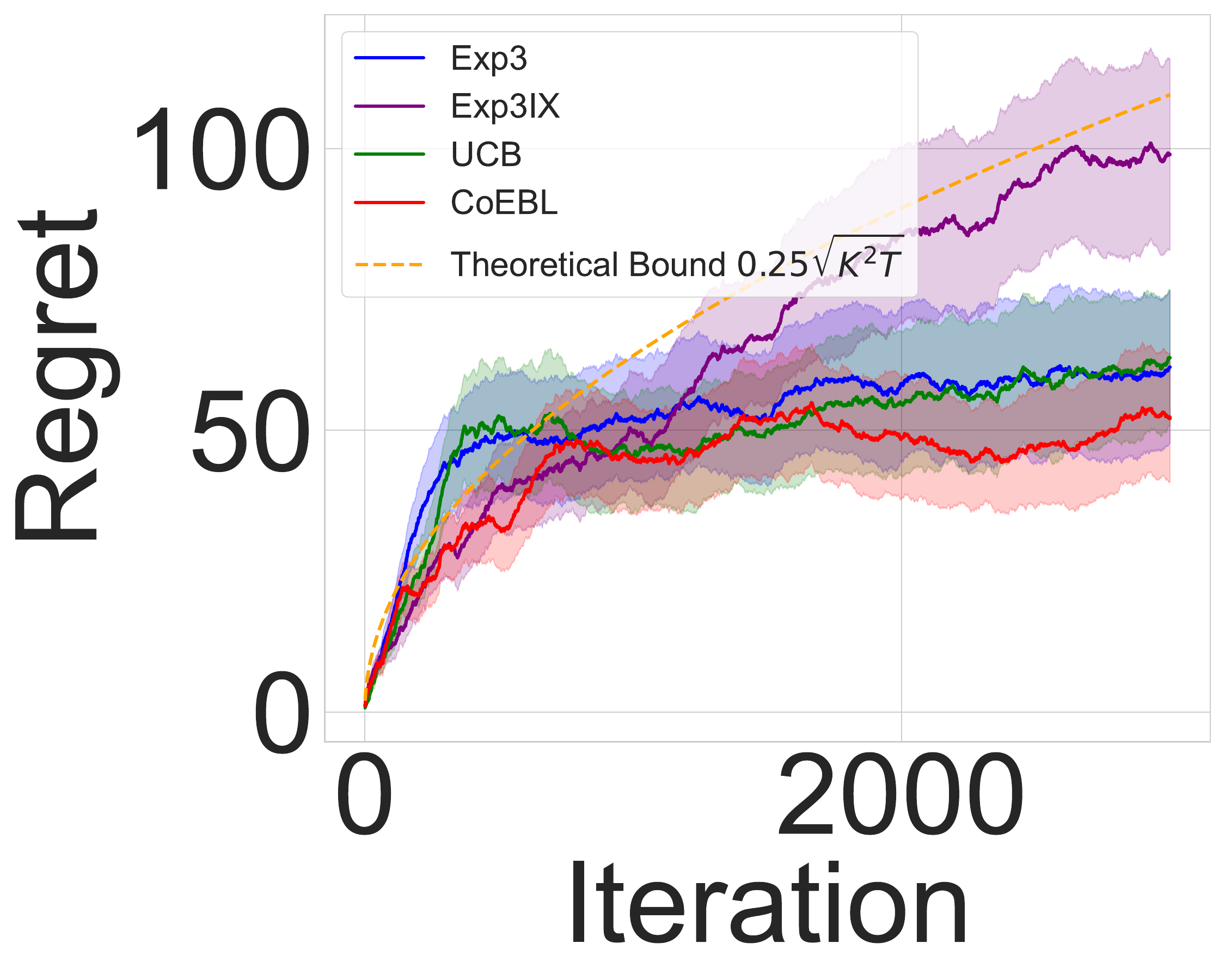}}
          \hfill
          \subfloat[$n=4$
          ]{%
             \includegraphics[width=0.33\linewidth]{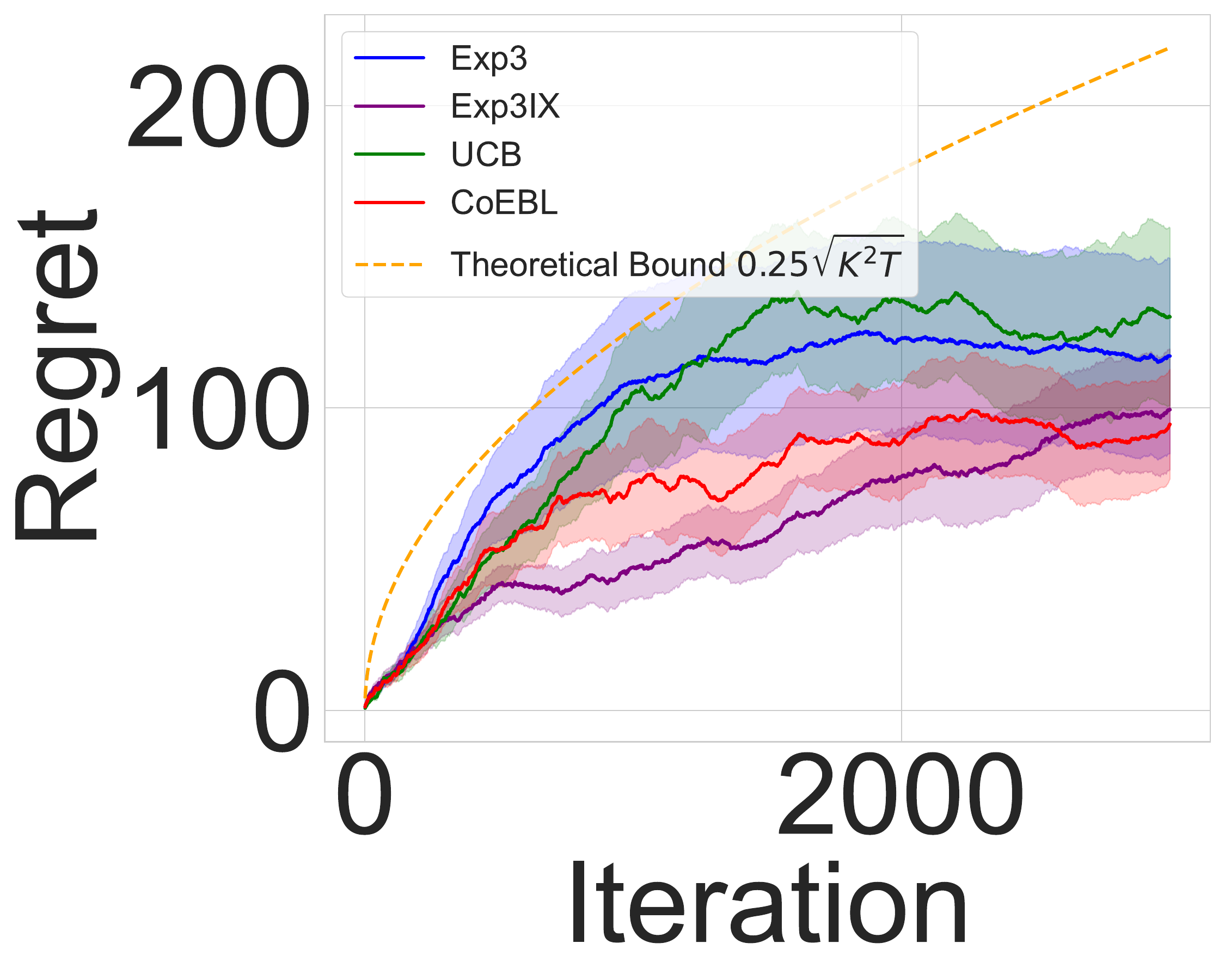}}
          \hfill
          \subfloat[$n=2$
          ]{%
             \includegraphics[width=0.33\linewidth]{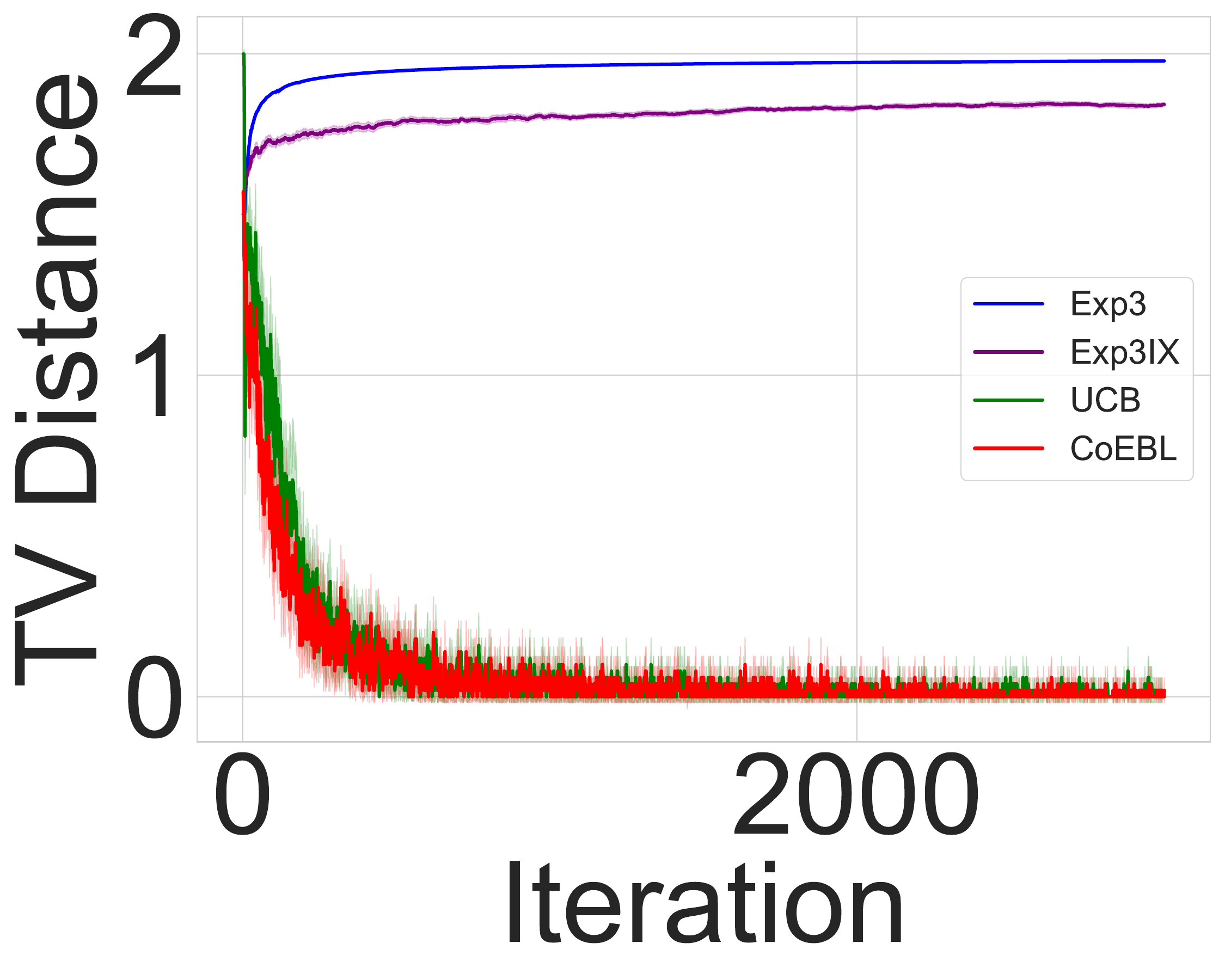}}
          \hfill
          \subfloat[$n=3$
          ]{%
             \includegraphics[width=0.33\linewidth]{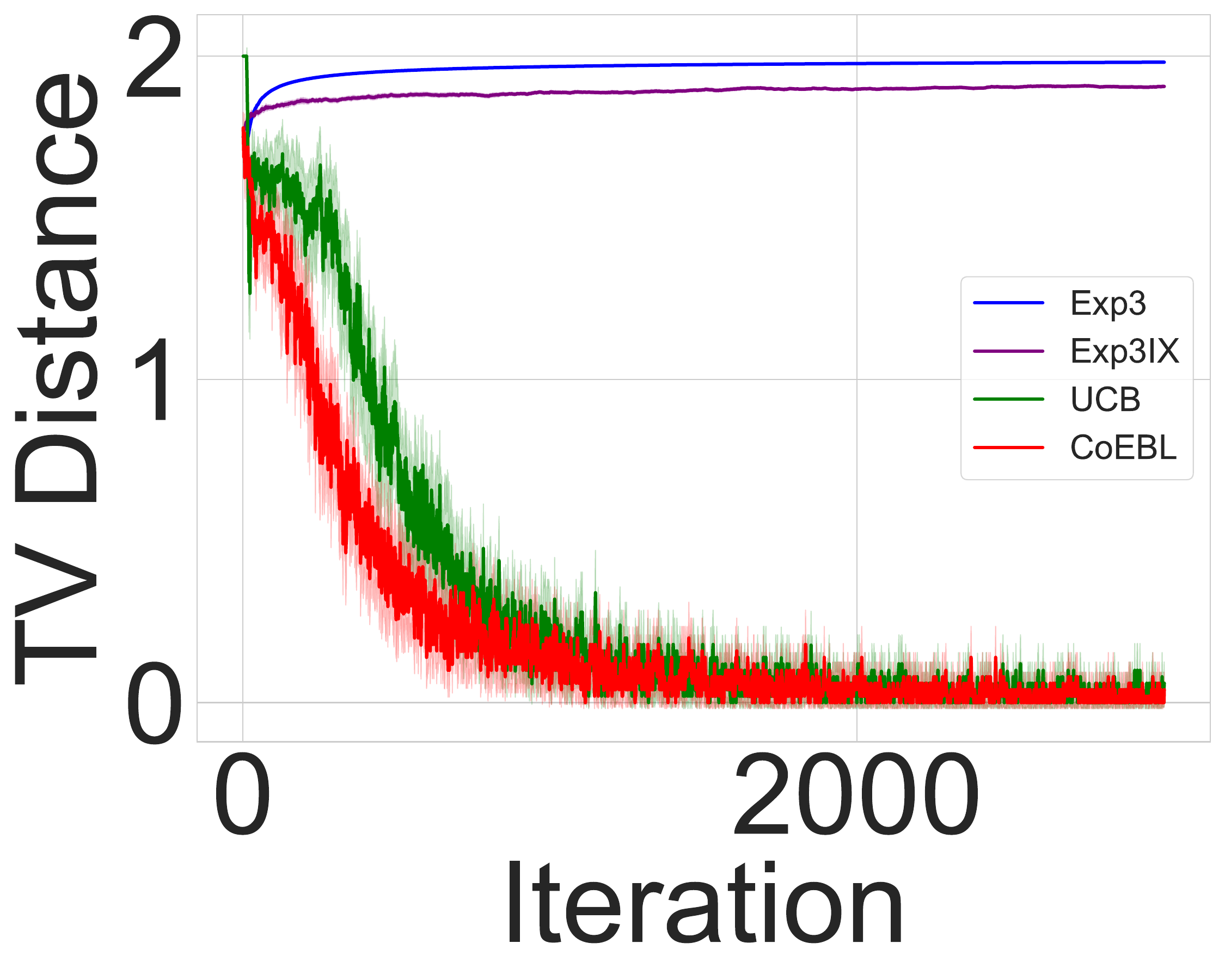}}
          \hfill
          \subfloat[$n=4$
          ]{%
             \includegraphics[width=0.33\linewidth]{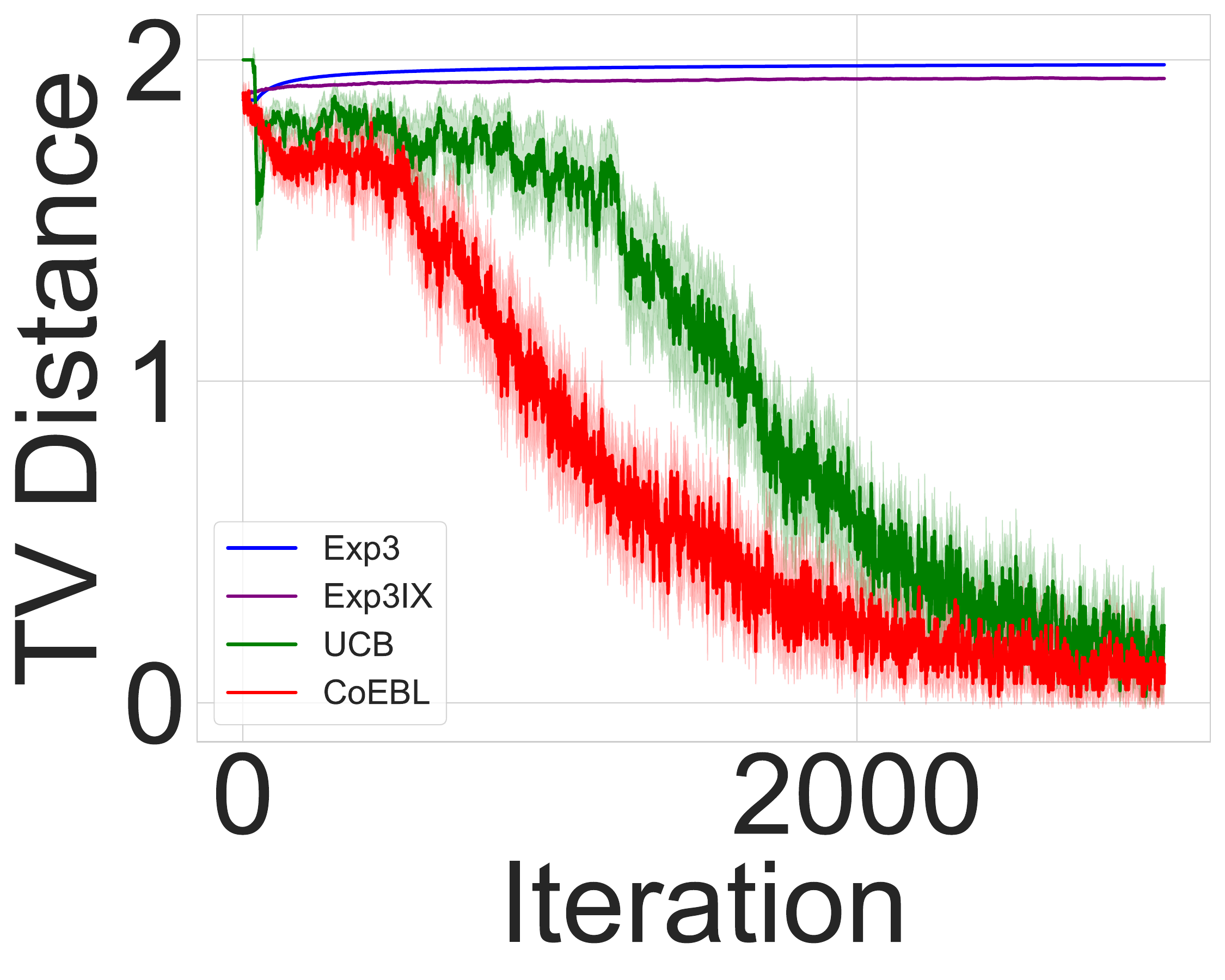}}
          \hfill
          \caption{Regret and TV Distance for Self-Plays on \BigNum}
          \label{fig:Comparison_Regret_TV_BigNum_0} 
       \end{figure}
       
       \begin{figure}[!ht]
          \centering
          \subfloat[
           $n=2$
          ]{%
             \includegraphics[width=0.45\linewidth]{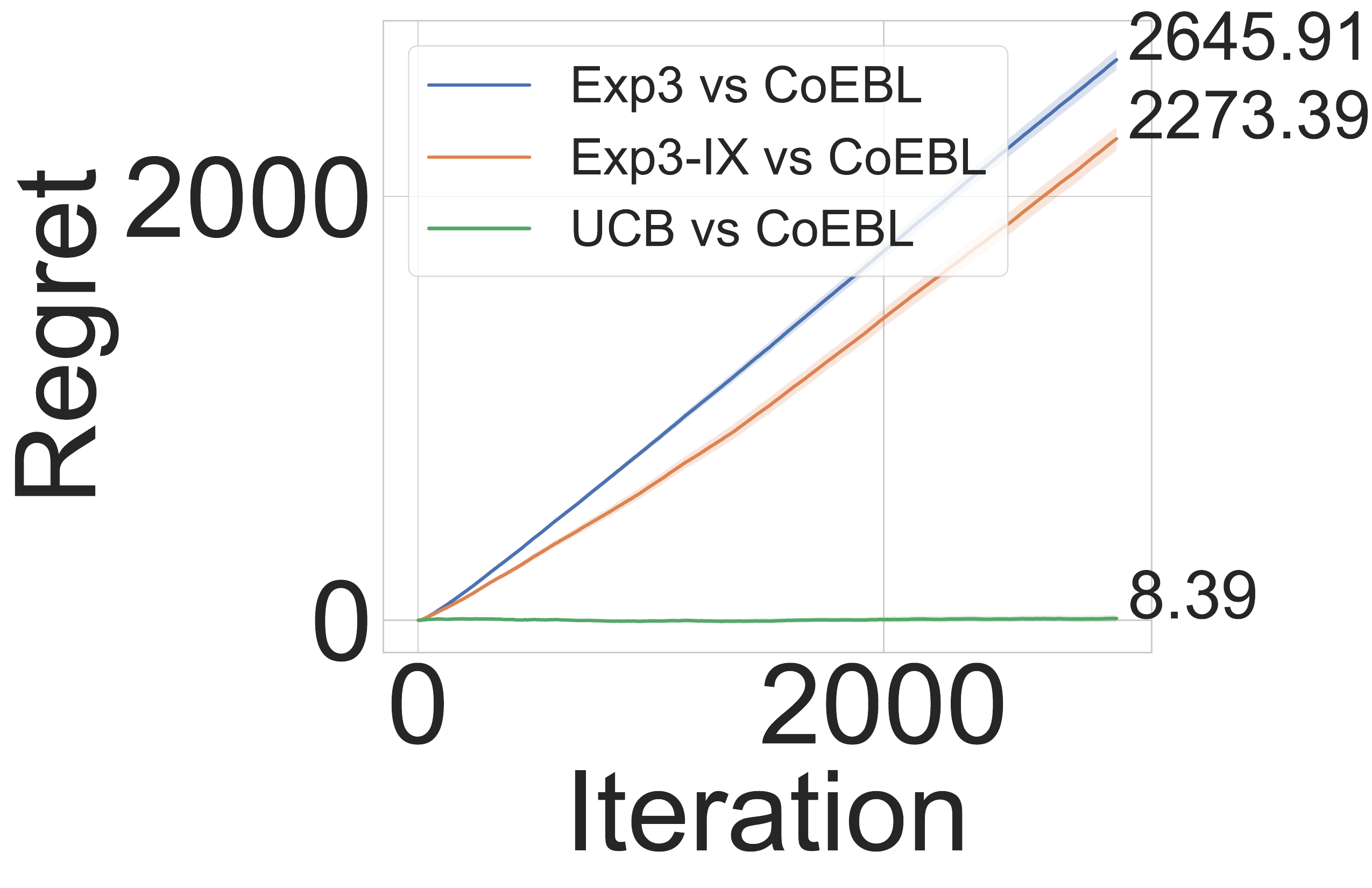}}
          \hfill
          \subfloat[
          $n=3$
          ]{%
             \includegraphics[width=0.45\linewidth]{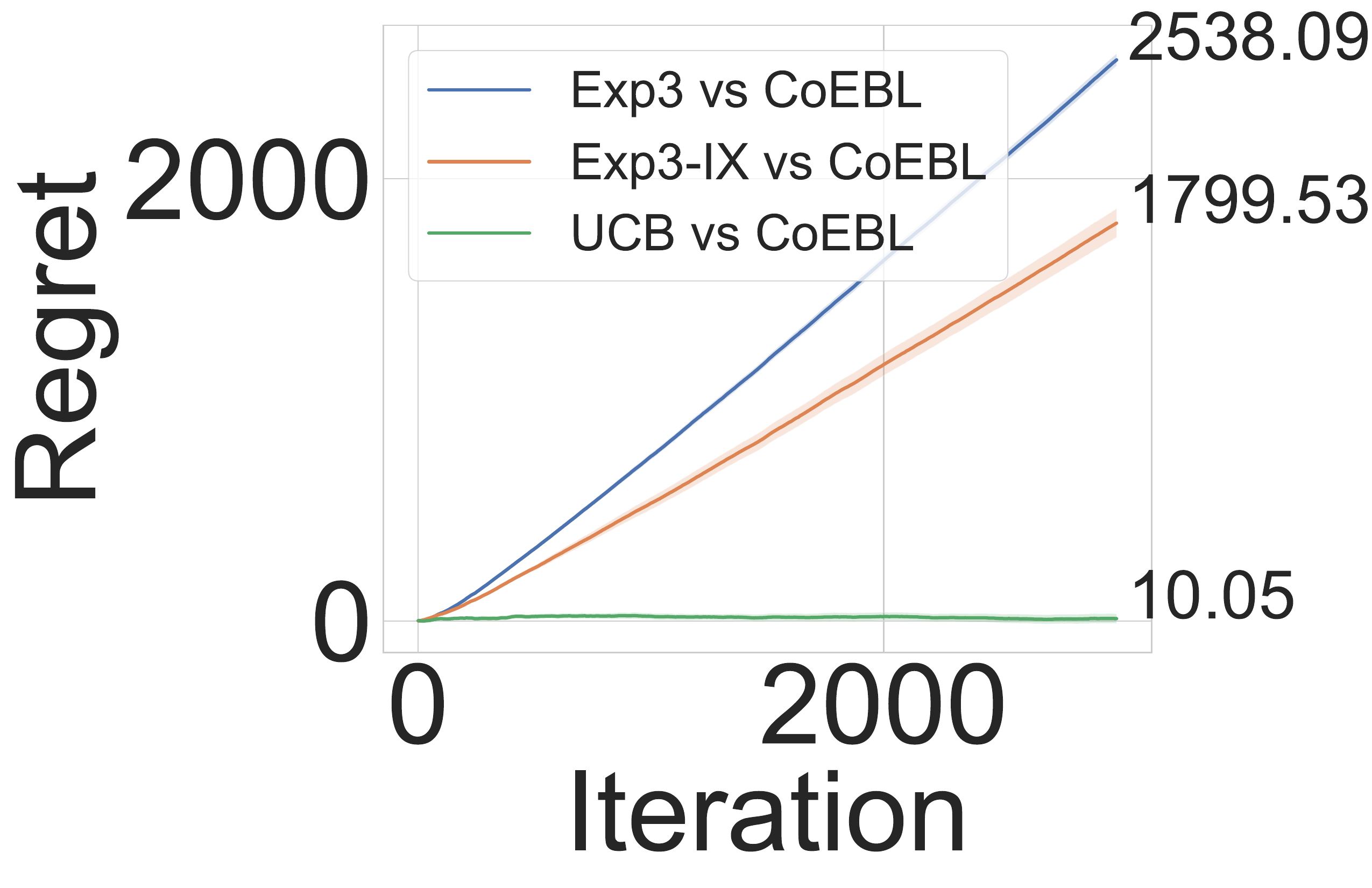}}
          \hfill
          \subfloat[
          $n=4$
          ]{%
             \includegraphics[width=0.45\linewidth]{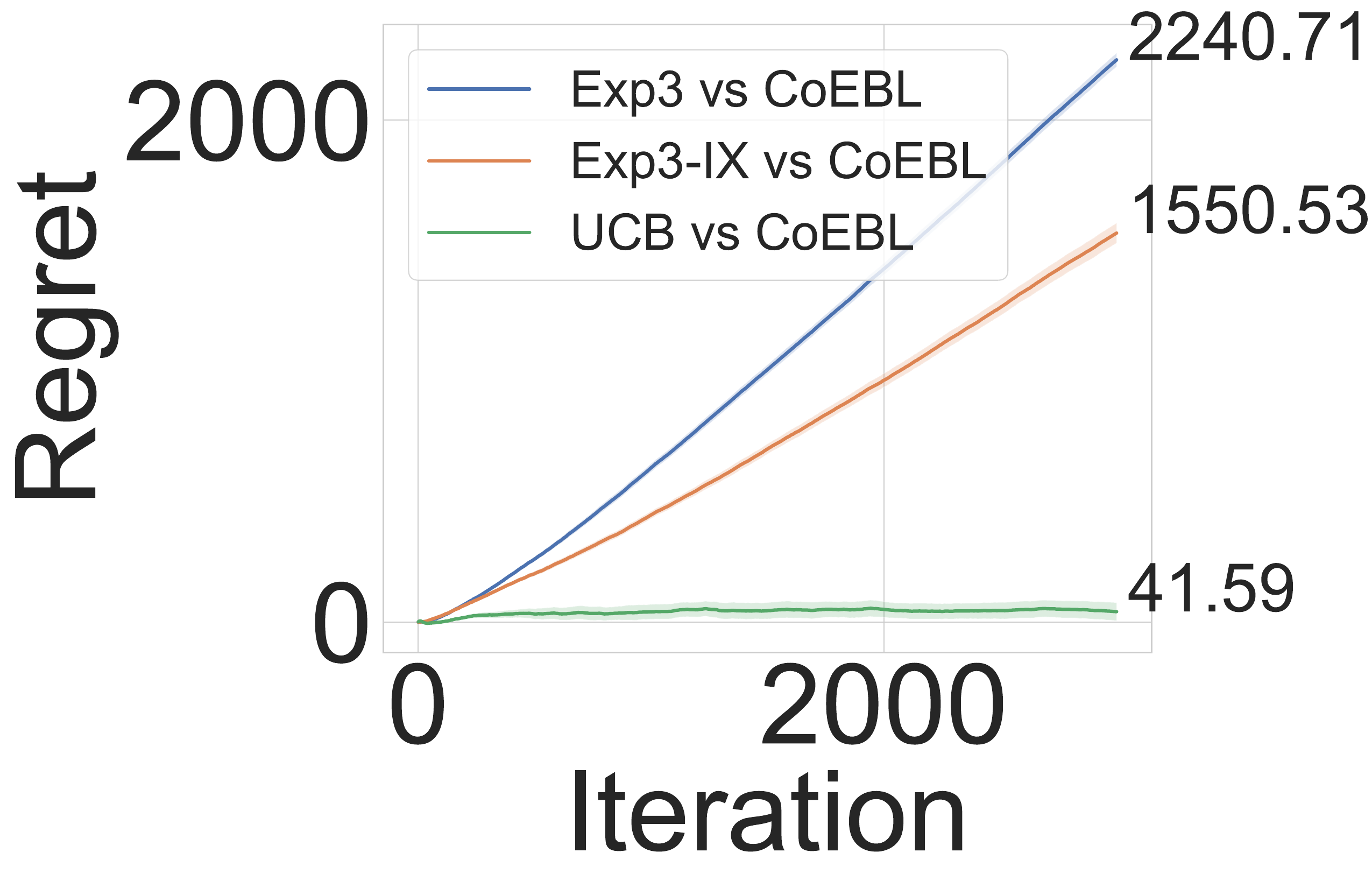}}
          \caption{Regret for $\alg~1$-vs-$\alg~2$ on \BigNum.}
          \hfill
          \label{fig:Regret_BigNum2_0} 
       \end{figure}
       
       % If the players select the same number, they receive a payoff of $0$. 
       % If the difference between the players' numbers is $1$, the player with the larger number receives a payoff of $2$, and the player with the smaller number receives a payoff of $-2$.
       % Otherwise, the player with the larger number receives a payoff of $1$, and the player with the smaller number receives a payoff of $-1$.
       % To simplify the game and make it consistent with ternary games, we modify the payoff function $\BigNum:\X \times \X \rightarrow \{ -1, 0, 1 \}$ defined by:
       % \begin{align*}
       %    \BigNum(x, y)
       %    := \begin{cases} 
       %      0 &  x = y \\
       %      1 &  x > y \\
       %     -1 &  x < y 
       %   \end{cases}.
       % \end{align*}
       
       In summary, we conducted extensive experiments on three matrix games: the RPS game, the \Diagonal game, and the \BigNum game.
       In terms of regret performance, \coebl in self-play aligns with our theoretical bounds.
       Moreover, \coebl consistently outperforms other bandit baselines when competing across various matrix game benchmarks, as shown in Figures~\ref{fig:Regret_Diagonal2} and \ref{fig:Regret_BigNum2_0}.
       \coebl matches the performance of \ucb and converges more quickly than \exptrni in the RPS game.
       \coebl converges to the Nash equilibrium for $n=2,3$ and for $n=2,3,4$, respectively, while the other baselines do not converge, as shown in Figures~\ref{fig:Comparison_Regret_TV_Diagonal} and \ref{fig:Comparison_Regret_TV_BigNum_0}.
       Therefore, we conclude that \coebl is a promising algorithm for matrix games, demonstrating sublinear regret, outperforming other bandit baselines, and achieving convergence to the Nash equilibrium in several matrix game instances.
       However, as the number of strategies grows exponentially, \coebl, like other algorithms, fails to converge to the Nash equilibrium. 
       This observation points out the current limitation of existing algorithms in exponentially large matrix games, and it will be an exciting path for future research.

       \section{Conclusion and Discussion}    
       This paper addresses the unsolved problem of learning in unknown two-player zero-sum matrix games with bandit feedback, proposing a novel algorithm, \coebl, which integrates evolutionary algorithms with bandit learning. 
       To the best of our knowledge, this is the first work that combines evolutionary algorithms and bandit learning for matrix games and provides regret analysis of evolutionary bandit learning (EBL) algorithms in this context. 
       This paper demonstrates that randomised or stochastic optimism, particularly through evolutionary algorithms, can also enjoy a sublinear regret in matrix games, offering a more robust and adaptive solution compared to traditional methods.
       
       Theoretically, we prove that \coebl exhibits sublinear regret in matrix games, extending the rigorous understanding of evolutionary approaches in bandit learning. 
       Practically, we show through extensive experiments on various matrix games, including the RPS, \Diagonal, and \BigNum games that \coebl outperforms existing bandit baselines, offering practitioners a new tool (randomised optimism via evolution) for handling matrix games playing with bandit feedback.
       
       Despite these promising results, our work has some limitations. 
       Theoretically, we only consider two-player zero-sum games, which is consistent with prior studies such as \citep{o2021matrix,cai2024uncoupled}. Extending \coebl to general-sum games with more players or to Markov games represents an exciting and challenging avenue for future research. 
       More technically, we conjecture whether Theorem~\ref{thm:main} could also hold for smaller value of $c<8$ with certain threshold.
       Additionally, our analysis assumes sub-Gaussian noise; investigating the algorithm's performance under different noise distributions, such as sub-exponential noise, could yield further insights. 
       From an experimental perspective, testing on more diverse problem instances would strengthen the empirical analysis.
       
       Future work could focus on both theoretical and practical extensions of evolutionary bandit learning. 
       From a theoretical perspective, it would be worthwhile to explore how \coebl or other evolutionary bandit learning algorithms can be adapted to more complex game structures, such as multi-player or general-sum games. 
       On the practical side, improving \coebl by incorporating more sophisticated mutation operators,  additional crossover operator, non-elitist selection mechanisms, or population-based evolutionary algorithms could enhance its performance in more complex settings. 
       % Exploring these avenues would further solidify the role of evolutionary approaches in bandit learning and matrix games.
             
\section*{Ethical Statement}

There are no ethical issues.

\section*{Acknowledgements}
We would like to thank Professor Per Kristian Lehre, Dr Alistair Benford and Dr Mario Alejandro Hevia Fajardo for the fruitful discussion and comments on an earlier draft of this paper and also thank the anonymous reviewers for their helpful reviews.
Shishen is supported by International PhD Studentship at School of Computer Science, the University of Birmingham.
The computations were performed using the University of Birmingham’s BlueBEAR high performance computing (HPC) service.

% \clearpage
\bibliographystyle{named}
% \bibliography{ijcai25.bib}

% APPENDIX
%%%%%%%%%%%%%%%%%%%%%%%%%%%%%%%%%%%%%%%%%%%%%%%%%%%%%%%%%%%%%%%%%%%%%%%%%%%%%%%
%%%%%%%%%%%%%%%%%%%%%%%%%%%%%%%%%%%%%%%%%%%%%%%%%%%%%%%%%%%%%%%%%%%%%%%%%%%%%%%

\newpage
\appendix
\onecolumn
% \appendixtableofcontents

\section{Appendix}

\subsection{Summary of Regret Bounds}
We provide a table to summarise the relevant regret bounds for different algorithms in matrix games with bandit feedback.

\begin{table*}[!ht]
   \centering
   \resizebox{\textwidth}{!} {
   \begin{tabular}{l@{\hspace{0.3cm}}c@{\hspace{0.3cm}}c@{\hspace{0.3cm}}c@{\hspace{0.3cm}}c@{\hspace{0.3cm}}c@{\hspace{0.3cm}}c@{\hspace{0.3cm}}}
   \toprule[1.5pt]
   \multirow{2}{*}{\textbf{Algorithms}} & OTS-\hedge / OTS-RM   & GP-MW & \exptr & \exptrni & \ucb  / K-Learning & \coebl \\
   & \citep{li2024optimistic} & \citep{sessa2019no} & \citep{auer2002nonstochastic} & \citep{neu2015explore,cai2024uncoupled} & \citep{o2021matrix} & \textbf{Theorem~\ref{thm:main}} \\
   \midrule[0.5pt]
   \multirow{2}{*}{Feedback} & \multirow{2}{*}{\shortstack{obtained reward \\ + opponents' actions}} & \multirow{2}{*}{\shortstack{obtained reward \\ + opponents' actions}} & \multirow{2}{*}{obtained reward} & \multirow{2}{*}{obtained reward} & \multirow{2}{*}{obtained reward} & \multirow{2}{*}{obtained reward} \\
   & & & & & & \\
   \midrule[0.5pt]
   \multirow{4}{*}{Regret} &
   \multirow{4}{*}{\shortstack{\Or{\sqrt{T\log K} + \sqrt{\gamma_T \beta T}} /\\ \Or{\sqrt{T K} + \sqrt{\gamma_T \beta T}}}} 
    &  \multirow{4}{*}{\Or{\sqrt{T\log K}} + $\gamma_T \sqrt{T}$}
    &  \multirow{4}{*}{\Or{\sqrt{TK\log K}}}
    &  \multirow{4}{*}{\Or{\sqrt{TK\log K}} }
    &  \multirow{4}{*}{\AsyO{\sqrt{K^2T}} }
    &  \multirow{4}{*}{\AsyO{\sqrt{K^2T}}} \\
    & & & & & & \\
    & & & & & & \\
    & & & & & & \\
   \bottomrule[1.5pt]
   \end{tabular}
   }
  %  \bigskip
    \caption{Regret bounds for different algorithms in matrix games. $K$ denotes the number of actions for each player, $T$ denotes the time horizon, and $\gamma_T$ in the bounds for the OTS-\hedge / OTS-RM   denotes the maximum information gain and in the bound for for the GP-MW algorithm denotes a kernel-dependent quantity. 
    $\beta$ in the bounds for the OTS-\hedge / OTS-RM is $\log (KT)$.
    In this table, we assume both players have the same number of strategies. This can be generalised to the case where both players have different numbers of strategies. For the regret bound of \coebl, we consider the worst-case scenario (i.e., the opponent uses the best-response strategy) and the Nash regret (Def.~(\ref{def:regret})), the same as in \citep{o2021matrix}.
   }
   \label{table:regret_table}
\end{table*}

\subsection{Pseudo-code of Algorithms}

   \par As follows, we summarise a general framework of algorithms for matrix games with bandit feedback considered in this paper.
   We only present the algorithm for the row player, and the algorithm for the column player is symmetric.
   \begin{algorithm}[!ht]
   \caption{General framework for matrix games with bandit feedback \citep{o2021matrix}}
   \label{alg:BBA}
   \begin{algorithmic}[1]
       \REQUIRE  Policy space of player: $\mathcal{X}  \subseteq \Delta_m$; 
       \REQUIRE Initial probability distribution $P_1 \in \mathcal{X}$;
       \FOR{$t = 1$ to $T$}
         %   \STATE Sample a pair of policies $(x,y) \sim \D(P_t, Q_t)$
           % \STATE Set $P_{t+1}(i):=x$ and $Q_{t+1}(i):=y$.
           % \STATE Use policy $x$ for the row player
           % \STATE Use policy $y$ for the column player
           \STATE The row player chooses action $i_t$ from $P_t$ 
           \STATE The column player chooses action $j_t$ from $Q_t$
           \STATE Observe reward $r_t$ based on $i_t,j_t$
           \STATE Update probability distribution $P_t$ based on $\F_{t+1}$, 
           where $\F_{t+1}:=(i_1,j_1,r_1 \cdots, i_t, j_t,r_t)$
       \ENDFOR
   \end{algorithmic}
   \end{algorithm}

\begin{algorithm}[!ht]
   \caption{\exptr for matrix games \citep{Auer1995bandit,o2021matrix}}
   \label{alg:Exp3}
   \begin{algorithmic}[1]
   \STATE \textbf{Input:} Number of actions $K$, number of iterations $T$, learning rate $\eta$ and exploration parameter $\gamma$.
   \STATE \textbf{Initialise:} 
   \STATE $\quad \hat{S_{0, i}} \gets 0 \text{ for all } i \in [K]$
   \FOR{$t = 1,2 , \cdots, T$}
       \STATE Calculate the sampling distribution $P_t$: \text{ for all } $i$
       \STATE $\quad P_{ti} \gets (1-\gamma) {\exp(\eta \hat{S}_{t-1,i}})/{\sum_{j=1}^{K} \exp(\eta \hat{S}_{t-1,j})} + \gamma / K$
       \STATE Sample $A_t \sim P_t$ and observe reward $X_t \in [0,1]$
       \STATE Update $\hat{S}_{ti}$: \text{ for all } $i$
       \STATE $\quad \hat{S}_{ti} \gets \hat{S}_{t-1,i} + 1 - {1\{A_t = i\}(1 - X_t)}/{P_{ti}}$
   \ENDFOR
   
   \end{algorithmic}
   \end{algorithm}

   \begin{algorithm}[!ht]
   \caption{\ucb for matrix games \citep{o2021matrix}}
   \label{alg:UCB}
   \begin{algorithmic}[1]
   \FOR{round $t = 1, 2, \ldots, T$}
       \FOR{all $i,j \in [m]$}
       \STATE compute $\tilde{A}^t_{ij} = \Bar{A^t_{ij}} + \sqrt{\frac{2 \log (2T^2 m^2)}{1 \lor n^t_{ij}}}$
       \ENDFOR
       \STATE use policy $x \in \arg\max_{x \in \Delta_m} \min_{y \in \Delta_m} y^T \tilde{A}^t x$
   \ENDFOR
   \end{algorithmic}
   \end{algorithm}

   \begin{algorithm}[H]
   \caption{\exptrni variant for matrix games \citep{cai2024uncoupled}}
   \label{alg:MGBF}
   \begin{algorithmic}[1]
   \REQUIRE \text{Define} $\eta_t = t^{-k_\eta}$, $\beta_t = t^{-k_\beta}$, $\epsilon_t = t^{-k_\epsilon}$ where $k_\eta = \frac{5}{8}$, $k_\beta = \frac{3}{8}$, $k_\epsilon = \frac{1}{8}$. $\A$ is the set of actions.
   \REQUIRE $\Omega_t = \{ x \in \Delta_m : x_a \geq \frac{1}{mt^2}, \, \forall a \in \mathcal{A} \}$.
   \STATE \textbf{Initialisation:} $x_1 = \frac{1}{m} {(1,\cdots ,1)}$.
   \FOR{$t = 1, 2, \ldots$}
       \STATE Sample $a_t \sim x_t$, and receive $\sigma_t \in [0, 1]$ with $\sigma_t = G_{a_t, b_t}$ where $b_t$ is the action by the opponent.
       \STATE Compute $g_t$ where $g_{t,a} = {1[a_t = a] \sigma_t}/{(x_{t,a} + \beta_t)} + \epsilon_t \ln x_{t,a}, \, \forall a \in \mathcal{A}$.
       \STATE Update $x_{t+1} = \arg\min_{x \in \Omega_t} \left\{ x^\top g_t + \frac{1}{\eta_t} \text{KL}(x, x_t) \right\}$.
   \ENDFOR
   \end{algorithmic}
   \end{algorithm}

\subsection{Algorithm Implementation}
   % \lss{Add one more baseline: Thompson Sampling}
   Previous works, including \citep{o2021matrix,cai2024uncoupled}, have not released the source code for their algorithms.
   Therefore, we provide our own implementation for \coebl, and other bandit baselines used in this paper.
   The source code is available at the link \url{https://github.com/Lss1242/CoEvo_Bandit_Learning}.
   
   % We will release the code later once the paper is accepted.

\subsection{Technical Lemmas}
\begin{restatable}{lemma}{SecThreeLemOne}
   \label{lem:identityOne}
   Given $x,y \in \Delta_m$, for all $i,j \in [m]$, $A_{ij} \in \mathbb{R}$,
   then $y^TAx=  \sum_{i,j \in [m]} y_j x_i A_{ij}$.
\end{restatable}

\begin{proof}[Proof of Lemma~\ref{lem:identityOne}]
   We compute $y^TAx$ as follows.
  %  \begin{align*}
  %  y^T A x 
  %  &= \begin{pmatrix} y_1, & \ldots &, y_m \end{pmatrix} 
  %  \begin{pmatrix} 
  %  \alpha & \ldots & \alpha \\ 
  %  \vdots & \ddots & \vdots \\ 
  %  \alpha & \ldots & \alpha 
  %  \end{pmatrix} 
  %  \begin{pmatrix} 
  %  x_1 \\ 
  %  \vdots \\ 
  %  x_m 
  %  \end{pmatrix} \\ 
  %  &= \begin{pmatrix} y_1, & \ldots  &, y_m \end{pmatrix} 
  %  \begin{pmatrix} 
  %  \alpha (x_1 + \ldots + x_m) \\ 
  %  \vdots \\ 
  %  \alpha (x_1 + \ldots + x_m) 
  %  \end{pmatrix} \\ \intertext{Using $x \in \Delta_m$ and linearity of matrix multiplication gives}
  %  &= \alpha \begin{pmatrix} y_1, & \ldots &, y_m \end{pmatrix} 
  %  \begin{pmatrix} 
  %  1 \\ 
  %  \vdots \\ 
  %  1 
  %  \end{pmatrix} \\
  %  &= \alpha (y_1 + y_2 + \ldots + y_m) \\ 
  %  \intertext{Using $y \in \Delta_m$ gives}
  %  &= \alpha \cdot 1 = \alpha .
  %  \end{align*}
   \begin{align*}
   y^T A x &= \begin{pmatrix} y_1 & \ldots & y_m \end{pmatrix} 
   \begin{pmatrix} 
   A_{11} & \ldots & A_{1m} \\ 
   \vdots & \ddots & \vdots \\ 
   A_{m1} & \ldots & A_{mm} 
   \end{pmatrix} 
   \begin{pmatrix} 
   x_1 \\ 
   \vdots \\ 
   x_m 
   \end{pmatrix} 
    \\  \intertext{Note that simple algebra gives}
   &= 
    \begin{pmatrix} y_1 & \ldots & y_m \end{pmatrix}
    \begin{pmatrix} 
    A_{11}x_1  + A_{1m}x_m \\ 
   \vdots \\ 
    A_{m1}x_1  + A_{mm}x_m
   \end{pmatrix} \\
   &=  \sum_{j=1}^{m} y_j \left( \sum_{i=1}^{m} x_i A_{ij} \right) \\
   & = \sum_{i,j \in [m]} y_j x_i A_{ij}.
   \end{align*}
   
   \end{proof}

\begin{restatable}{lemma}{SecThreeLemTwo}
\label{lem:identityTwo}
The following inequalities hold for any $n \in \N$:
\begin{enumerate}
   \item[(1)]
       \begin{align*}
       1+ \frac{1}{\sqrt{2}} + \cdots + \frac{1}{\sqrt{n}} \leq 2 \sqrt{n}.
   \end{align*}

   \item [(2)] Given $x_i \geq 0$ for all $i \in [n]$,
       \begin{align*}
       \frac{1}{n}\sum_{i=1}^n x_i  \leq \sqrt{\frac{\sum_{i}^n x_i^2}{n}}
   \end{align*}

   \item[(3)] Hoeffding's inequality for $\sigma^2$-sub-Gaussian random variables with zero-mean \citep{vershynin2018high}:
   let \(X_1, \ldots, X_n\) be \(n\) independent random variables such that \(X_i \) is  $\sigma^2$-sub-Gaussian.
   Then for any \(\mathbf{a} \in \mathbb{R}^n\), we have
   \begin{align*}
   \Pr\left( \sum_{i=1}^n a_i X_i > t \right) \leq \exp\left(-\frac{t^2}{2\sigma^2 \|\mathbf{a}\|_2^2}\right), 
   \Pr\left(\sum_{i=1}^n a_i X_i < -t \right) \leq \exp\left(-\frac{t^2}{2\sigma^2 \|\mathbf{a}\|_2^2}\right).
   \end{align*}
   
   Of special interest is the case where \(a_i = 1/n\) for all \(i\). Then, we get that the average \(\bar{X} = \frac{1}{n} \sum_{i=1}^n X_i\) satisfies
   \begin{align*}
   \Pr(\bar{X} > t) \leq \exp\left(-\frac{nt^2}{2\sigma^2}\right), 
   \Pr(\bar{X} < -t) \leq \exp\left(-\frac{nt^2}{2\sigma^2}\right).
   \end{align*}
\end{enumerate}

\end{restatable}

\begin{proof}[Proof of Lemma~\ref{lem:identityTwo}]
   Proof of (3) can be found in \citep{vershynin2018high}.
   So, we only provide the proofs of other two inequalities here.
    \begin{enumerate}
        \item[(1)] We proceed by induction.
         For $n=1$, the inequality is trivial, i.e. $1\leq 2\sqrt{1}$.
         Now, assume the inequality holds for $n=k \geq 2$.
         For the case $n=k+1$, applying the induction hypothesis step gives,
         \begin{align*}
            1 + \frac{1}{\sqrt{2}} + \cdots + \frac{1}{\sqrt{k}} + \frac{1}{\sqrt{k+1}}
            &= 2\sqrt{k} + \frac{1}{\sqrt{k+1}} \\ \intertext{Rearranging the terms gives}
            &\leq 2\sqrt{k+1} +2\sqrt{k} -2\sqrt{k+1} + \frac{1}{\sqrt{k+1}} \\ \intertext{Notice that $2\sqrt{k} -2\sqrt{k+1} = \frac{-2}{\sqrt{k}+\sqrt{k+1}}$. Thus, we have}
            &= 2 \sqrt{k+1 } + \frac{\sqrt{k}+ \sqrt{k+1}-2\sqrt{k+1}}{\sqrt{k+1} \left(\sqrt{k+1}+\sqrt{k} \right)} \\ \intertext{Note that $\sqrt{k}+ \sqrt{k+1}-2\sqrt{k+1} = \sqrt{k}- \sqrt{k+1}<0$ gives}
            &< 2\sqrt{k+1}.
         \end{align*}
         Thus, we complete the induction step and can complete the proof i.e. the inequality holds for all $n \in \mathbb{N}$.

         \item[(2)] 
         We proceed by induction.
         For $n=1$, the inequality is trivial, i.e. $1\leq \sqrt{1^2}$.
         Now, assume the inequality holds for $n=k \geq 2$.
         For the case $n=k+1$, applying the induction hypothesis step gives,
         \begin{align*}
            \frac{1}{(k+1)^2} \left(x_1 + \cdots + x_k + x_{k+1} \right)^2 
            &\leq   \frac{1}{(k+1)^2} \left(\sqrt{k \sum_{i=1}^k x_i^2}+x_{k+1} \right)^2 \\
            &=\frac{1}{(k+1)^2} \left(k \sum_{i=1}^k x_i^2 + x_{k+1}^2 +2 x_{k+1} \sqrt{k \sum_{i=1}^k x_i^2}  \right) \\ \intertext{Notice that $2ab \leq a^2 +b^2$ for $a,b \geq 0$ gives $2x_{k+1} \sqrt{k \sum_{i=1}^k x_i^2} =2x_{k+1} \sqrt{k} \cdot \sqrt{\sum_{i=1}^k x_i^2}\leq kx_{k+1}^2 +\sum_{i=1}^k x_i^2$. }
            &\leq \frac{1}{(k+1)^2}   \left( k \sum_{i=1}^k x_i^2 + x_{k+1}^2 + kx_{k+1}^2 +\sum_{i=1}^k x_i^2  \right) \\ \intertext{Rearranging the terms gives}
            &=  \frac{1}{(k+1)^2}   \left( (k+1) \sum_{i=1}^{k+1} x_i^2 \right) = \frac{1}{k+1}   \sum_{i=1}^{k+1} x_i^2.
         \end{align*}
         Then, taking the square root of both sides gives the desired inequality for the case $n=k+1$.
         Thus, we complete the proof.
    \end{enumerate}

    \end{proof}
    \subsection{Omitted Proofs}
    Note that we restrict $A \in  [0, 1]^{m \times m}$ in the analysis for simplification. 
    However, the proof works for any bounded $A \in  [-b, b]^{m \times m}$ where $b$ is constant with respect to $T$ and $m$, by simply shifting from $[-b,b]$ to $[0,2b]$ and normalising the entries in $[0,1]$.
 
 % \SecThreeTheZero*
 
 % \begin{proof}[Proof of Theorem~\ref{thm:mainzero}]
    
 % \end{proof}

 \SecThreeLemThree*
 
 \begin{proof}[Proof of Lemma~\ref{lem:HoeffdingBound}]
     We consider the mutation rate $c>0$ in \coebl, where $c$ is a constant with respect to $T$ and $m$.
     We denote the empirical mean of the sample payoff $A_{ij}$ by $(\Bar{A_t})_{ij}$ and the number of times that row $i$ and column $j$ have been chosen by both players up to round $t$.
     Under Assumption (A), we compute the probability with $z_{ij} \sim \Nr(0,1)$ are i.i.d:
     \begin{align*}
       &  \Pr \left(A_{ij} \leq (\tilde{A_t})_{ij} \right) \\
     =&  \Pr \left((A_{ij} \leq (\Bar{A_t})_{ij}+ \sqrt{\frac{c \log (2T^2 m^2)}{1 \lor n^t_{ij} +1}} + \frac{1}{1\lor n_{ij}^t}\cdot z_{ij}\right) \\
     =& \Pr \left(A_{ij} - \frac{1}{1 \lor n_{ij}^t} \sum_{k=1}^{1 \lor n_{ij}^t} (A_k)_{ij} - \frac{z_{ij}}{1 \lor n_{ij}^t}  \leq \sqrt{\frac{c \log (2T^2 m^2)}{1 \lor n^t_{ij} +1}} \right) \\ \intertext{Recall that $(A_k)_{ij}=A_{ij}+\eta_k$ where $\eta_k$ are i.i.d. 1-sub-Gaussian with zero mean.
     Note that $ \eta_k':=-\eta_k$ is also 1-sub-Gaussian with zero mean and $z_{ij}':=-z_{ij} \sim \Nr(0,1)$. Thus, we can rewrite the inequality as follows.}
     =& \Pr \left( \frac{1}{1 \lor n_{ij}^t} \left(\sum_{k=1}^{1 \lor n_{ij}^t} \eta_k' + z_{ij}' \right)\leq \sqrt{\frac{c \log (2T^2 m^2)}{1 \lor n^t_{ij} +1}} \right)
     \end{align*}
     We consider the reverse quantity:
     \begin{align*}
          &\Pr \left( \frac{1}{1 \lor n_{ij}^t} \left(\sum_{k=1}^{1 \lor n_{ij}^t} \eta_k' + z_{ij}' \right)> \sqrt{\frac{c \log (2T^2 m^2)}{1 \lor n^t_{ij} +1}} \right) \\
     = & \Pr \left( \frac{1}{1 \lor n_{ij}^t+1} \left(\sum_{k=1}^{1 \lor n_{ij}^t} \eta_k' + z_{ij}'  \right)
     > \frac{1 \lor n_{ij}^t}{1 \lor n_{ij}^t+1} \sqrt{\frac{c \log (2T^2 m^2)}{1 \lor n^t_{ij} +1}} \right) \\ \intertext{Note that $\frac{1 \lor n_{ij}^t}{1 \lor n_{ij}^t+1} \geq \frac{1}{2}$. 
     Thus, we have}
     \leq & \Pr \left( \frac{1}{1 \lor n_{ij}^t+1} \left(\sum_{k=1}^{1 \lor n_{ij}^t} \eta_k' + z_{ij}' \right)> \frac{1}{2}\sqrt{\frac{c \log (2T^2 m^2)}{1 \lor n^t_{ij} +1}} \right) \\ \intertext{Using Hoeffding's inequality for i.i.d. sub-Gaussian random variables gives}
     \leq & \exp \left( - \frac{(1 \lor n_{ij}^t+1)\cdot \frac{1}{4}{\frac{c \log (2T^2 m^2)}{1 \lor n^t_{ij} +1}} }{2} 
     \right) \\ 
     =& \left( \frac{1}{2T^2m^2} \right)^{c/8}:=\delta
     \end{align*}
     Hence, we complete the proof.
     \end{proof}

 \SecThreeMain*
 
 \begin{proof}[Proof of Theorem~\ref{thm:main}]
       First, we follow the proof of Theorem 1 in \citep{o2021matrix} using the following events.
       Let $E_t$ be the event that $\exists i,j \in [m]$ such that $(\tilde{A}_t)_{ij} <A_{ij}$. 
       We know $E_t \in \F_t$ where $\F_t$ is defined in the preliminaries.
       Consider some iteration $E_t$ does not hold and let 
       $\tilde{y}_t:= \text{argmin}_{y \in \Delta_m} y^T\tilde{A_t} x_t$ be the best response of the column player where $x_t$ is the current policy of the row player.
       Since $E_t$ does not hold, then for $\forall i,j \in [m], A_{ij} \leq (\tilde{A}_t)_{ij}$.
       Thus, $V_{A}^* \leq V_{\tilde{A}_t}^*$.
       So, the regret in each round $t$ under the case that $E_t$ does not hold is bounded by the following,
       \begin{align}
           V_{A}^* -\Et{y_t^T Ax_t } 
           \leq \Et{V_{\tilde{A}_t}^* - y_t^T Ax_t} 
           &= \Et{\tilde{y_t}^T \tilde{A_t}x_t -y_t^T Ax_t} \nonumber\\ 
           \intertext{Note that in \coebl, we incorporate the selection mechanism: if $\min_{y\in \Delta_m} y^T \tilde{A_t} x' > \min_{y\in \Delta_m} y^T \tilde{A_t} x_{t-1}$, then we update $x_t=x'$ and $\min_{y\in \Delta_m}y^T \tilde{A_t}x_t = \min_{y\in \Delta_m}y^T \tilde{A_t} x'$; otherwise we update $x_t=x_{t-1}$ and $\min_{y\in \Delta_m}y^T \tilde{A_t}x_t>\min_{y\in \Delta_m}y^T \tilde{A_t}x'$. 
           In both cases, we can give a upper bound: $\tilde{y_t}^T \tilde{A_t}x' \leq \tilde{y_t}^T \tilde{A_t} x_t$. Moreover, $x' \in \arg\max_{x \in \Delta_m} \min_{y \in \Delta_m} y^T \tilde{A}^t x $ and thus $x_t \in \arg\max_{x \in \Delta_m} \min_{y \in \Delta_m} y^T \tilde{A}^t x $. In other words, \coebl only updates the current policy if the previous policy is not the maximin solution anymore. Recall that $\tilde{y_t}$ is the best response of the column player. We obtain an upper bound.}
           &\leq \Et{y_t^T\tilde{A_t}x_t-y_t^TAx_t} \nonumber\\
           &= \Et{y_t^T\left(\tilde{A_t}-A \right)x_t}  \nonumber \\ 
           \nonumber
           \intertext{Recall the estimated matrix in Algorithm~\ref{alg:CoEBL} and $x_t \in \arg\max_{x \in \Delta_m} \min_{y \in \Delta_m} y^T$. 
           We have $\left(\tilde{A_t}-A \right)_{ij} = \sqrt{\frac{c \log (2T^2 m^2)}{1 \lor n^t_{ij} +1}} + \frac{1}{1 \lor n^t_{ij} }\Nr(0,1) $.
           Note that $\log (2T^2m^2)= \log \left( (1/\delta)^{8/c} \right) = 8 \log(1/\delta) /c$.
           Using Lemma~\ref{lem:identityOne} gives} 
           &= \Et{\sqrt{\frac{8 \log (1/\delta)}{1 \lor n^t_{i_tj_t} +1}} + \sum_{j=1}^m y_j \sum_{i=1}^m x_i \frac{z_{ij}}{1 \lor n^t_{ij}}} \nonumber \\ \intertext{Note that $1 \lor n^t_{ij} \geq 1$. We can have the following inequality.}
           &\leq \Et{\sqrt{\frac{8\log (1/\delta)}{1 \lor n^t_{i_tj_t} +1}} + \sum_{j=1}^m y_j \sum_{i=1}^m x_i z_{ij}}
           \nonumber\\ \intertext{By linearity of expectation and $\Et{z_{ij}}=0$, we have}
           &= \Et{\sqrt{\frac{8 \log (1/\delta)}{1 \lor n^t_{i_tj_t} +1}}}.  \label{eq:upper}
       \end{align}
       Thus, we can bound the overall regret.
       Given the class of games $\forall A \in \A$ defined in (A), we have
       \begin{align*}
           \R \left(A, \coebl, T\right) 
           =& \E{\sum_{t=1}^T V_{A^*} - \Et{y_t^TAx_t}} \\ \intertext{Using law of total probability gives}
           =& \E{\sum_{t=1}^T V_{A^*} -\Et{y_t^TAx_t} \mid \cap_{t=1}^T E_t^c} \cdot \Pr \left(\cap_{t=1}^T E_t^c \right) \\
           \quad & + \E{\sum_{t=1}^T V_{A^*} -\Et{y_t^TAx_t} \mid  \left( \cap_{t=1}^T E_t^c \right)^c}  \times \Pr \left(\left( \cap_{t=1}^T E_t^c \right)^c  \right) \\ \intertext{Using De Morgan's Law gives $ \left( \cap_{t=1}^T E_t^c \right)^c = \cup_{t=1}^T E_t$.}
           =& \E{\sum_{t=1}^T V_{A^*} -\Et{y_t^TAx_t} \mid \cap_{t=1}^T E_t^c} \cdot \Pr \left(\cap_{t=1}^T E_t^c \right) \\
           \quad & + \E{\sum_{t=1}^T V_{A^*} -\Et{y_t^TAx_t} \mid   \cup_{t=1}^T E_t} \cdot \Pr \left( \cup_{t=1}^T E_t  \right) \\ \intertext{Using the upper bound in Eq.~\ref{eq:upper} and $\Pr \left(\cap_{t=1}^T E_t^c \right) \leq 1 $ gives}
           \leq & \E{\sum_{t=1}^T \sqrt{\frac{8\log (1/\delta)}{1 \lor n^t_{i_tj_t} +1}} } + \E{\sum_{t=1}^T 1} \cdot \Pr \left( \cup_{t=1}^T E_t  \right) \\ \intertext{Using the Union bound gives}
           \leq &  \E{\sum_{t=1}^T \sqrt{\frac{8 \log (1/\delta)}{1 \lor n^t_{i_tj_t} +1}} } + T \sum_{t=1}^T \Pr \left(E_t \right) \\ \intertext{Using Lemma~\ref{lem:HoeffdingBound} gives $\Pr \left(E_t \right) \leq \delta$. Thus, we have}
           \leq &  \E{\sum_{t=1}^T \sqrt{\frac{8 \log (1/\delta)}{1 \lor n^t_{i_tj_t} +1}} } + T^2 \delta \\ \intertext{Recall that $\delta=\left(1/2T^2m^2 \right)^{c/8}\leq 1/2T^2m^2$ for $c\geq 8$.
           Note that $\log (1/\delta) = \log \left(\left(2T^2m^2 \right)^{c/8} \right)=c \log(2T^2m^2)/8$.
           }
           \leq & \E{\sum_{t=1}^T \sqrt{\frac{c \log (2T^2m^2)}{1 \lor n^t_{i_tj_t} +1}} } + \frac{1}{2m^2} \\ \intertext{Rewrite the summation in the expectation. }
           \leq & \sum_{i,j \in [m]} 
           \E{\sum_{t=1}^T \sqrt{\frac{c \log (2T^2m^2)}{1 \lor n^t_{ij}+1}} \mathds{1}_{\{i_t=i,j_t=j}\}}
           + \frac{1}{2m^2} \\ \intertext{Let us denote the set $B_{ij}:=\{t \in \{0, \cdots, n_{ij}^T\} \mid i_t=i,j_t=j\}$ for $i, j \in [m]$. So we can rewrite the summand as follows.
           }
           =& \sqrt{c \log (2T^2m^2)}\sum_{i,j \in [m]} 
           \E{
             \sum_{t_k \in B_{ij}} \sqrt{\frac{1}{1 \lor n_{ij}^{t_k}+1}}
           } 
           + \frac{1}{2m^2} 
           \\ \intertext{Note that $n_{ij}^{t_k}$ is an increasing sequence in $t_k$. Thus, we can have}
           =& \sqrt{c \log (2T^2m^2)}\sum_{i,j \in [m]} 
           \E{
             \sum_{k=1}^{n_{ij}^{T}} \sqrt{\frac{1}{k+1}}
           } 
           + \frac{1}{2m^2} 
           \\ \intertext{Adding one more $1/\sqrt{1}$ in the inner sum and using Lemma~\ref{lem:identityTwo} (1) give}
           &\leq  \sqrt{c \log (2T^2m^2)}\sum_{i,j \in [m]} \E{2\sqrt{1\lor n_{ij}^T+1}}+ \frac{1}{2m^2} \\ \intertext{Using Lemma~\ref{lem:identityTwo} (2) with $x_k:=\sqrt{1\lor n_{ij}^T+1}$ where $k \in [m^2]$ gives}
           &\leq \sqrt{4c\log (2T^2m^2)}\cdot m^2 \sqrt{\frac{\sum_{i,j \in [m]}1\lor n_{ij}^T+1}{m^2}} \\ \intertext{Notice that $1 \lor n_{ij}^T \leq  n_{ij}^T+1$.}
           &\leq  \sqrt{4c\log (2T^2m^2)}  \sqrt{m^2{\sum_{i,j \in [m]} (n_{ij}^T+2)}} \\ \intertext{Notice that $\sum_{i,j \in [m]} n_{ij}^T=T$.}
           &= \sqrt{4c\log (2T^2m^2)}  \sqrt{m^2{ (T+2m^2)}} \\ \intertext{Since $T\geq 2m^2$, we have}
           &\leq \sqrt{4c\log (2T^2m^2)}  \sqrt{2m^2T} = 2 \sqrt{2c Tm^2\log (2T^2m^2)} =
           \tilde{O}(\sqrt{m^2T}).
        \end{align*} 
    Thus, we can conclude that  $\worstre\left(\A, \coebl, T\right) =\AsyO{\sqrt{m^2T}} $.
 \end{proof}

 \newpage
 \subsection{Complete Empirical Results}

 \subsubsection{Reasons for the choices of matrix games benchmarks}
 We choose the given matrix games benchmarks for the following reasons:
 
 \begin{enumerate}
    \item The RPS game is a classical benchmark widely used in the previous RL and game theory literature, and we want to compare the performance of \coebl with the existing algorithms.
    \item However, RPS consists of a small number of actions and the game is not complex enough to test the performances of the algorithms. Therefore, we included the \Diagonal and \BigNum games, which are more complex and feature exponentially larger action spaces
    \item We chose these matrix game benchmarks from multiple fields, including RL~\citep{littman1994markov,o2021matrix}, game theory~\citep{ijcai2024p336}, and evolutionary computation~\citep{lin2024overcoming}, to demonstrate the general applicability of the proposed algorithm.
 \end{enumerate}
 
 \subsubsection{Reasons for the choices of symmetric matrix games benchmarks}
    One might notice that all the matrix games considered in the experiments are symmetric, meaning that for the payoff matrix $A$,  $A_{ij}=-A_{ji}$ for all $i,j \in [m]$. 
    In such games, there is no advantage in being the first or second player, 
    the experimental studies provide fair head-to-head comparisons.
 
    \subsubsection{Reasons for not including OTS and TS in experiments}
    We chose not to include Optimistic Thompson Sampling (OTS) or Thompson Sampling (TS) in the experiments for the following reasons:
    
    \begin{enumerate}
        \item \citet{o2021matrix} have already evaluated the traditional TS as a baseline and demonstrated the superiority of UCB over traditional TS in the matrix game setting. Therefore, it makes sense to directly compare \coebl with the most robust algorithms reported in \citep{o2021matrix}.
        \item \coebl can be viewed as an evolutionary variant of TS, as it implements randomised optimism through the evolutionary operator and the selection mechanism. Since \coebl inherently extends the principles of TS, comparing with TS would not provide additional meaningful insights.
        \item OTS \citep{li2024optimistic} focuses on adversarial regret rather than Nash regret. Since our work focuses on Nash regret, it would be inappropriate and unfair to compare algorithms that rely on fundamentally different metrics.
        \item The current theoretical and empirical analysis is sufficient to demonstrate the potential and effectiveness of randomised optimism via evolutionary approach in matrix games against \ucb, which is the key message from our current work. Including additional experiments with TS or OTS is unnecessary at this stage. 
        % However, for future extensions, such comparisons could be included for completeness.
    \end{enumerate}
    
    The experimental design focuses on robust and directly comparable baselines to ensure that the evaluations are fair, relevant, and aligned with the objectives of this study.
    
%   \subsection*{Reasons for no ablation analysis}

%   An ablation study could, in principle, isolate the impact of stochastic optimism or the selection scheme. However, COEBL’s components—stochastic optimism, selection scheme, and evolutionary updates—are deeply intertwined and work synergistically. Isolating one component may not reflect its actual contribution in practice, as its effectiveness depends on interactions with other components.

%    Moreover, the regret analysis and empirical benchmarks already provide strong evidence of COEBL's effectiveness as a whole. Theorem~\ref{thm:main} rigorously establishes sublinear regret, demonstrating that the theoretical guarantees hold even with the inclusion of stochastic optimism. Empirical results further reinforce that COEBL outperforms baseline algorithms, validating the combined impact of its components. Conducting an ablation study might therefore introduce artificial scenarios that do not accurately reflect COEBL’s real-world performance.

 \subsubsection{\Diagonal Game}
 
 We defer the full experimental results on \Diagonal game to the appendix
 and provide the payoff matrix of \Diagonal game when $n=2$.
    
    \begin{table}[H]
    \centering
    \begin{tabular}{c|c c c c}
       & 00 & 01 & 10 & 11 \\ \hline
    00 & 0  & -1  & -1  & -1  \\
    01 & 1  & 0  & 0  & -1  \\
    10 & 1  & 0  & 0  & -1  \\
    11 & 1  & 1  & 1  & 0  \\
    \end{tabular}
    \caption{The payoff matrix of \Diagonal game $(n=2)$. 
    Binary bitstrings represent different pure strategies of each player.
    This game compares the number of $1$-bits of each player.
    }
    \end{table}
 
 In this case, both players have $2^n$ actions, which is way more complicated than the RPS. 
 In terms of the regret, all the algorithms in the self-play scenario, exhibit sublinear regrets.
 However, only \coebl converges for several problem instances.
 When $n$ increases to certain level, like $n\geq 4$, none of them can converge to the Nash equilibrium anymore.
 For the \alg-1 vs \alg-2 scenario, after iteration $2000$, \coebl has an overwhelming advantage over other bandit baselines in terms of regret performance.
 For the convergence of the the Nash equilibrium, surprisingly, 
 in Figure~\ref{fig:Regret_Diagonal2_2}, \ucb-vs-\coebl converges to or approximates the Nash equilibrium even when $n=4$.
 However, they also fail to converge to the Nash equilibrium when $n=5,6,7$.
 We can see that the opponent performance has certain impact to the overall dynamics towards the Nash equilibrium.
 \begin{figure}[!ht]
    \centering
    \subfloat[$n=2$
    ]{%
       \includegraphics[width=0.33\linewidth]{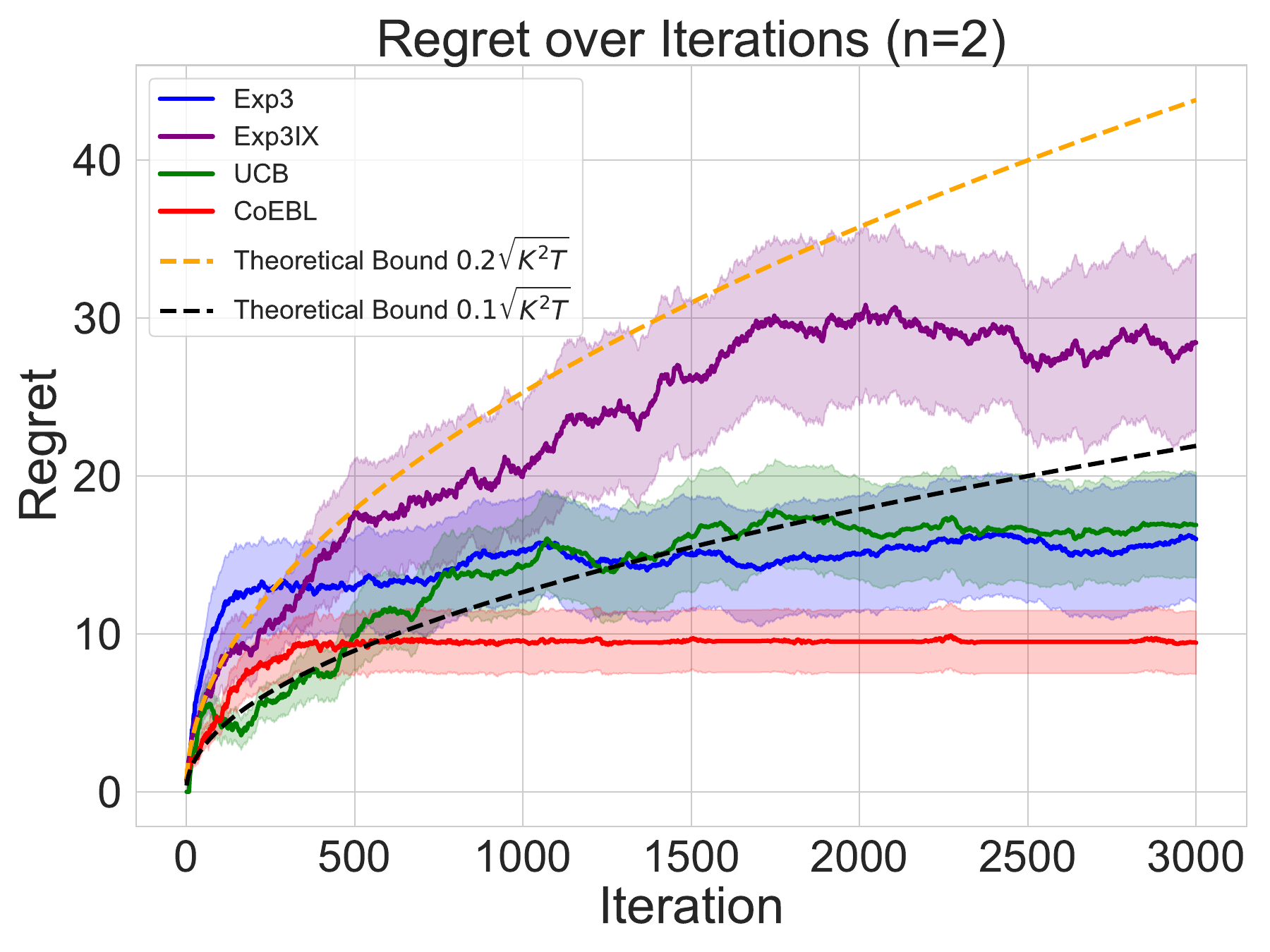}}
    \hfill
    \subfloat[$n=3$
    ]{%
       \includegraphics[width=0.33\linewidth]{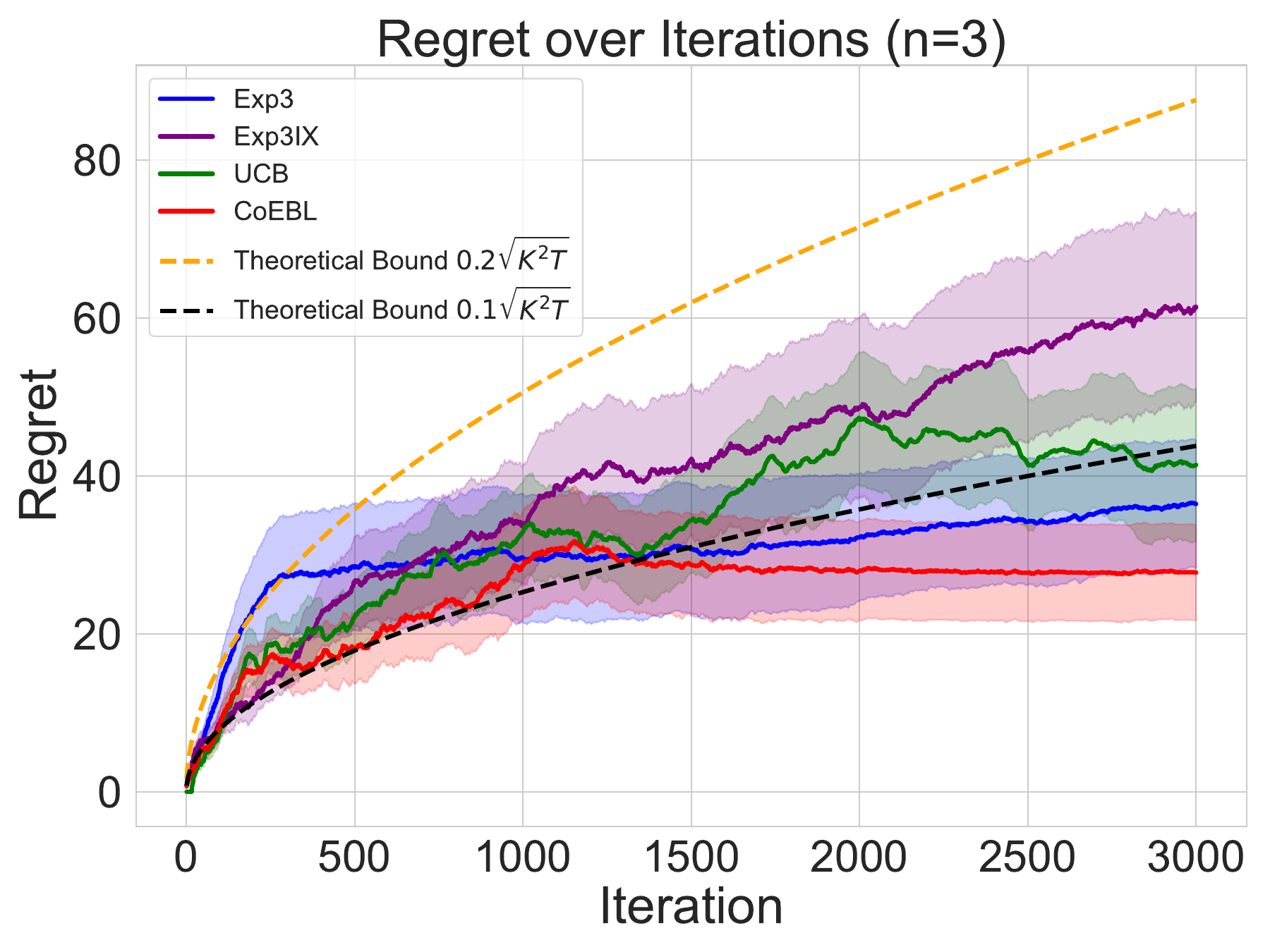}}
    \hfill
    \subfloat[$n=4$
    ]{%
       \includegraphics[width=0.33\linewidth]{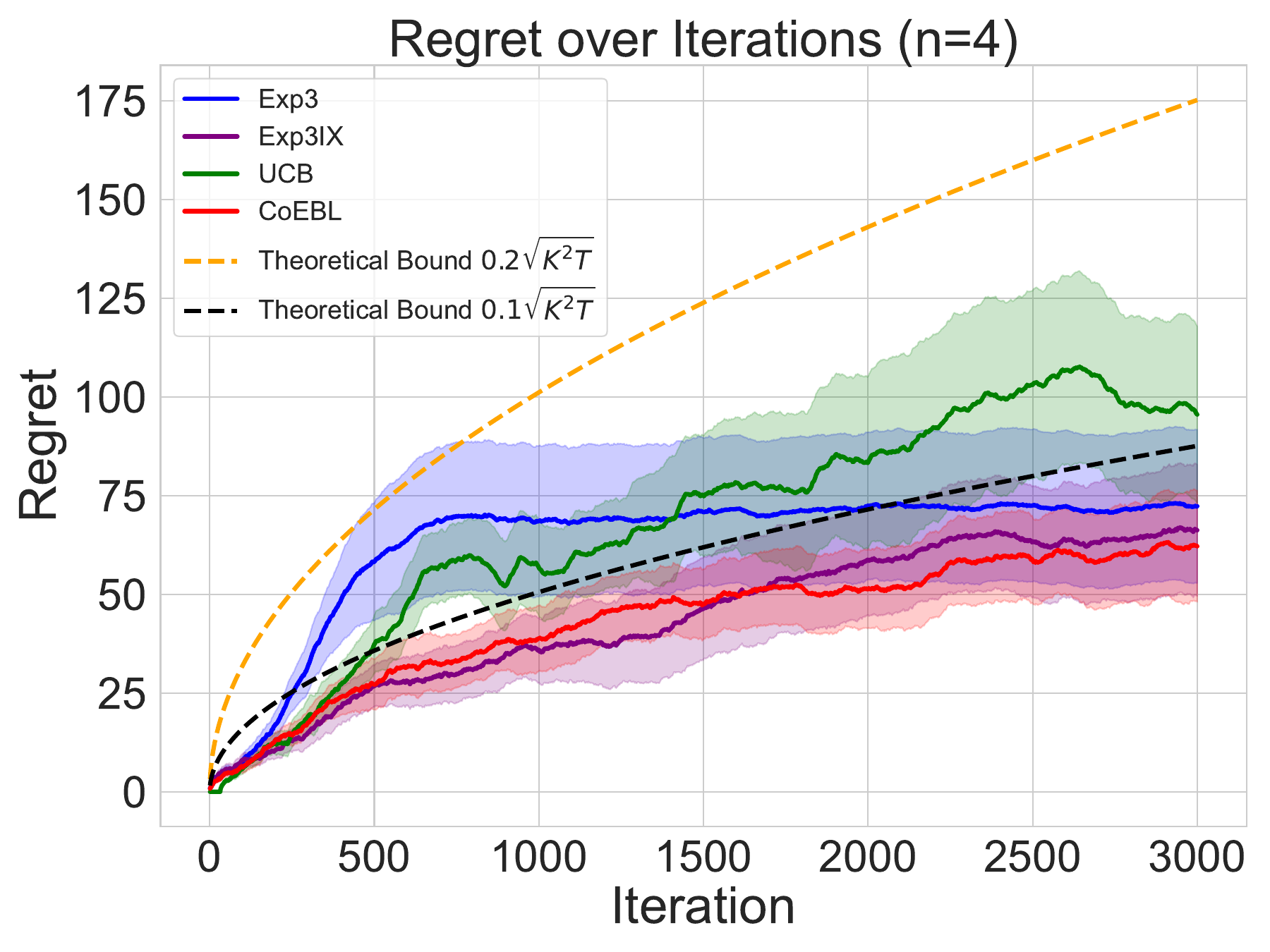}}
    \hfill
    \subfloat[$n=5$
    ]{%
       \includegraphics[width=0.33\linewidth]{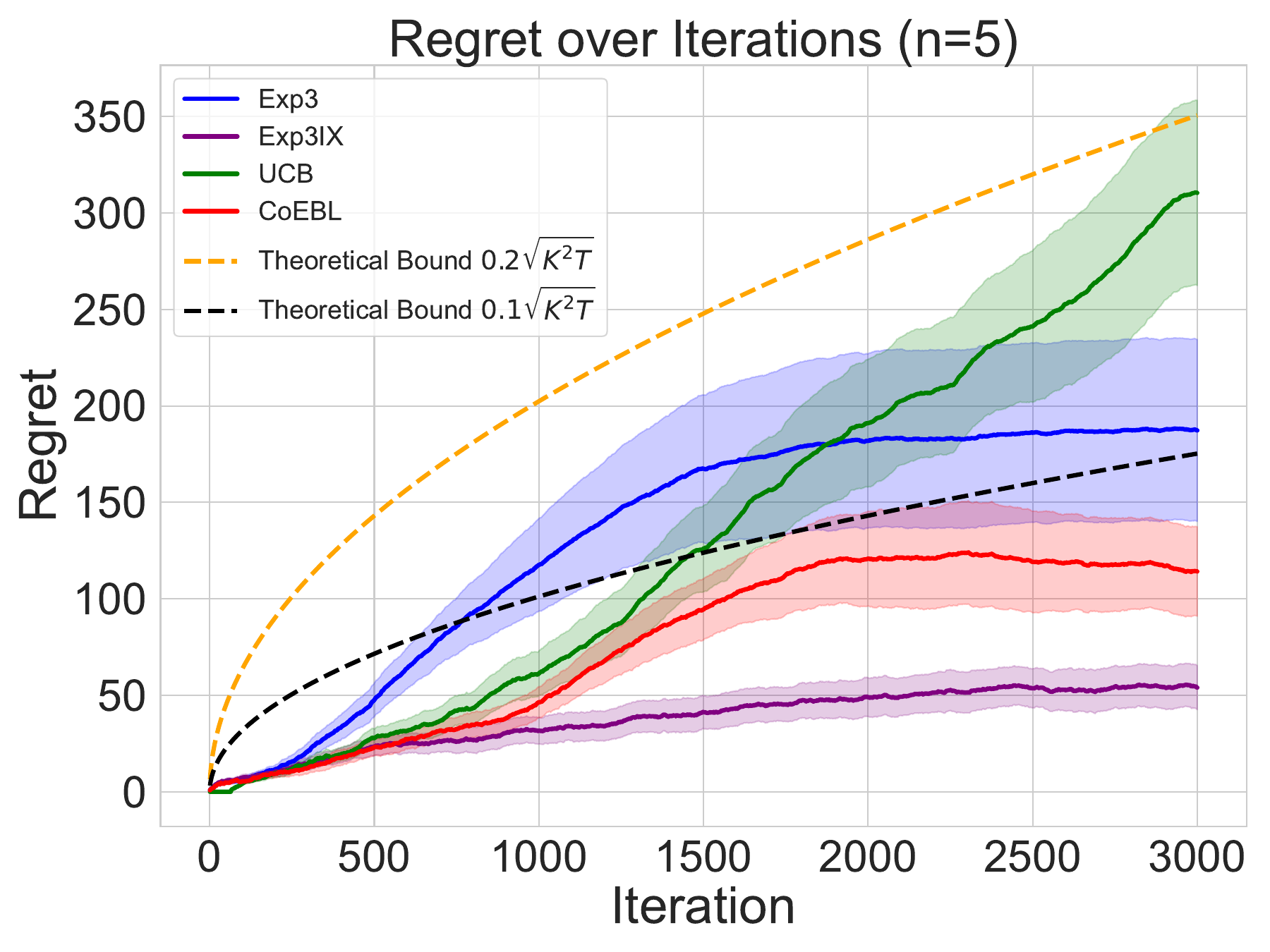}}
    \hfill
    \subfloat[$n=6$
    ]{%
       \includegraphics[width=0.33\linewidth]{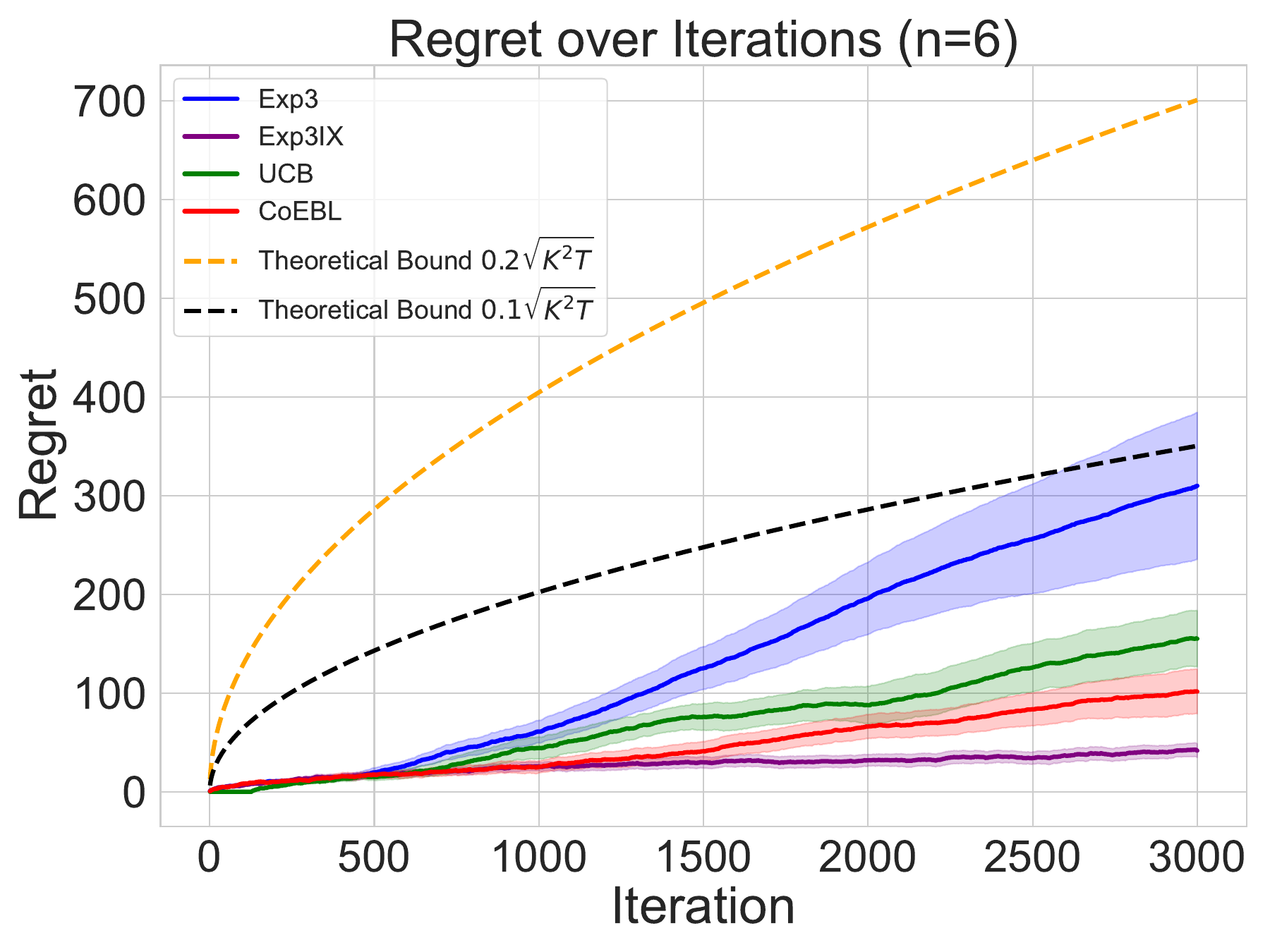}}
    \hfill
    \subfloat[$n=7$
    ]{%
       \includegraphics[width=0.33\linewidth]{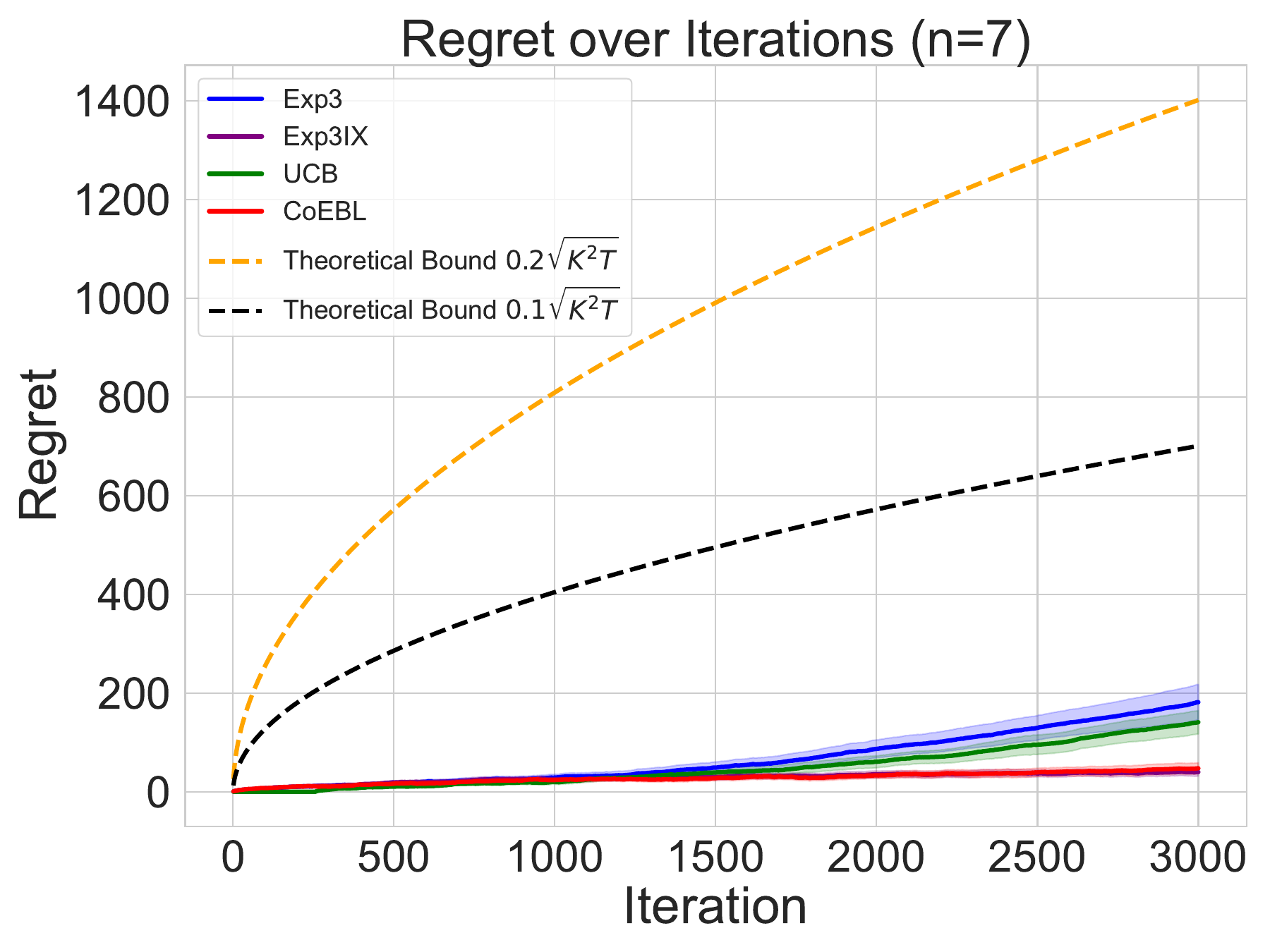}}
    \hfill
    \subfloat[$n=2$
    ]{%
       \includegraphics[width=0.33\linewidth]{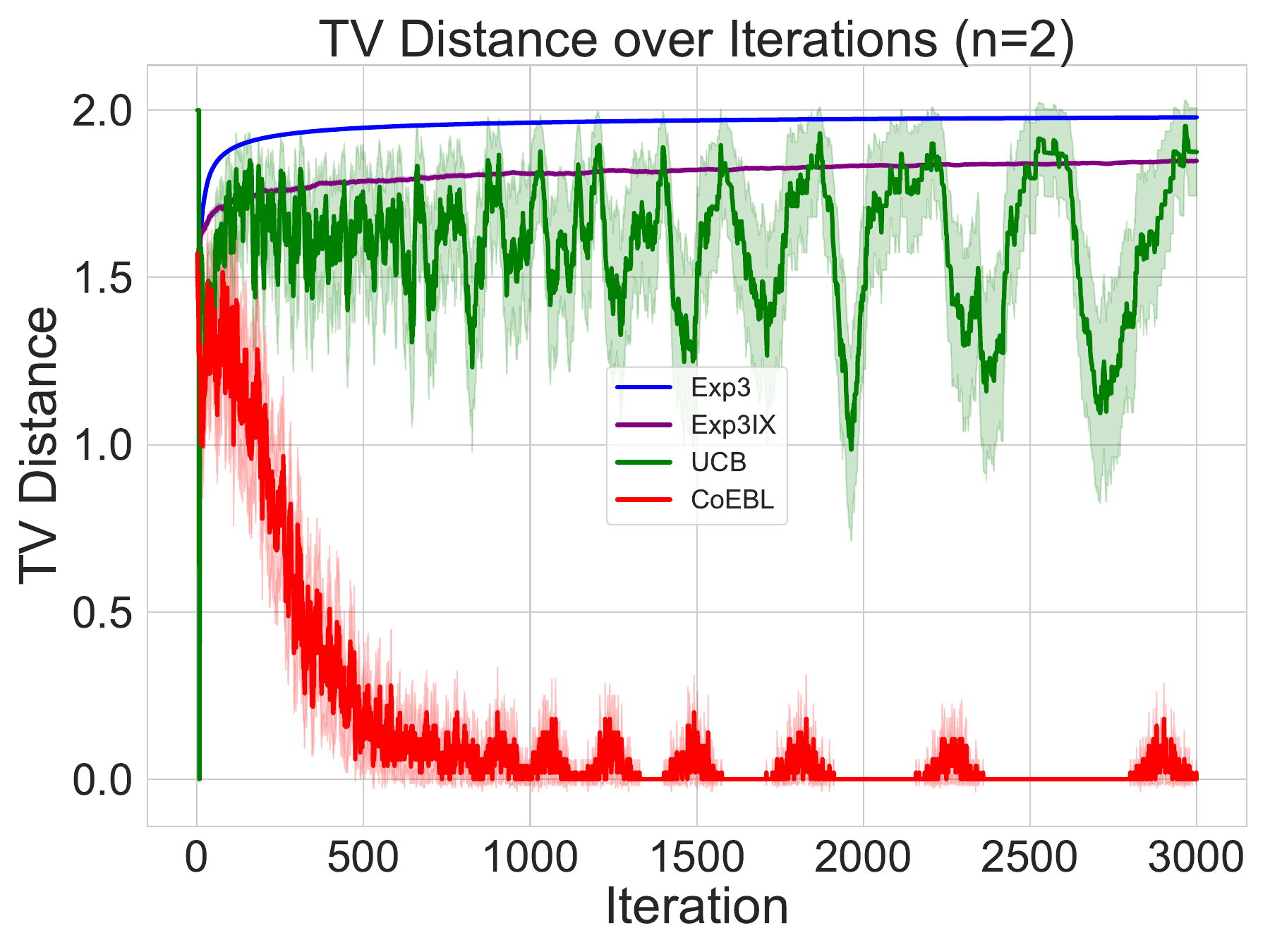}}
    \hfill
    \subfloat[$n=3$
    ]{%
       \includegraphics[width=0.33\linewidth]{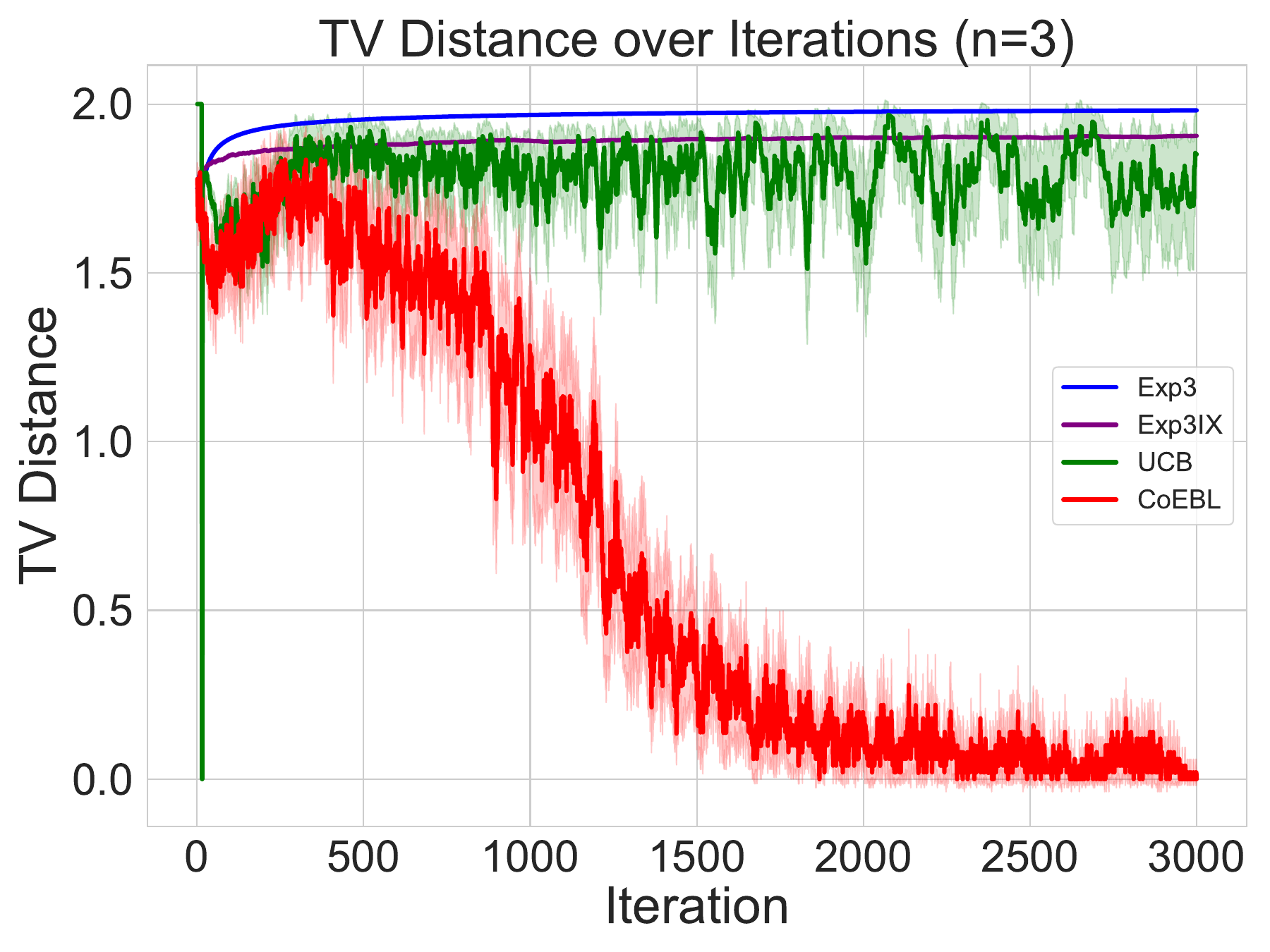}}
    \hfill
    \subfloat[$n=4$
    ]{%
       \includegraphics[width=0.33\linewidth]{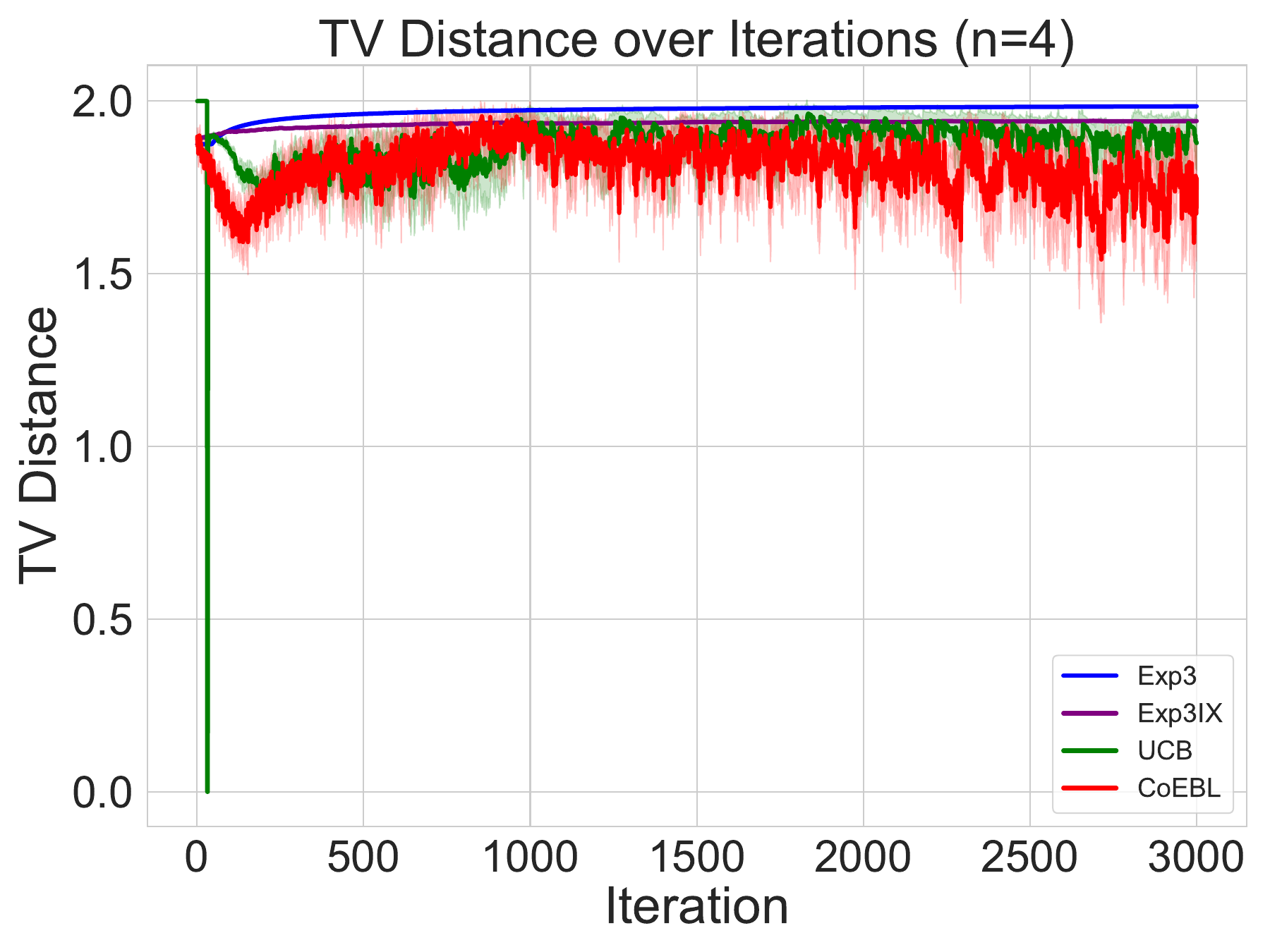}}
    \hfill
    \subfloat[$n=5$
    ]{%
       \includegraphics[width=0.33\linewidth]{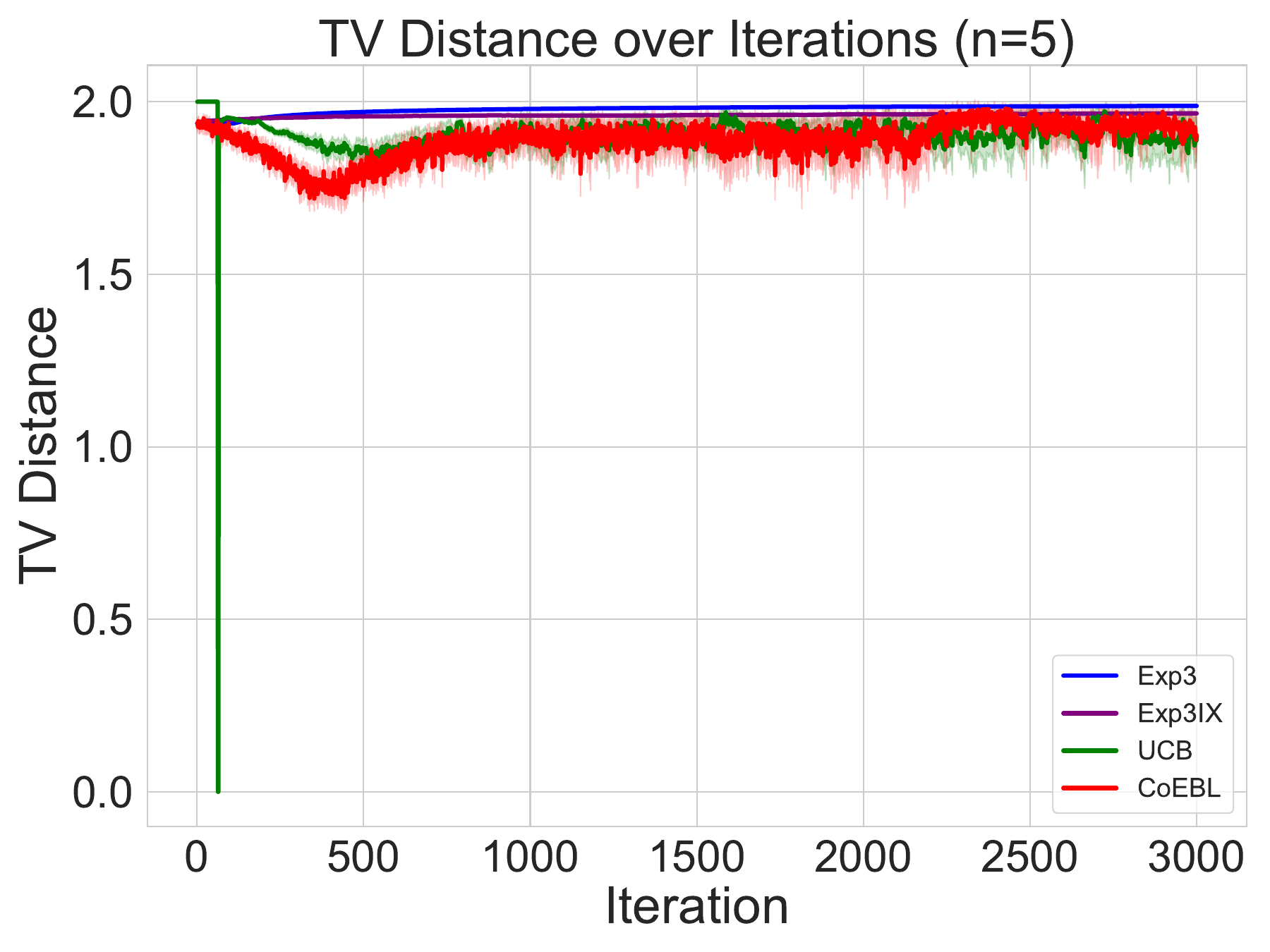}}
    \hfill
    \subfloat[$n=6$
    ]{%
       \includegraphics[width=0.33\linewidth]{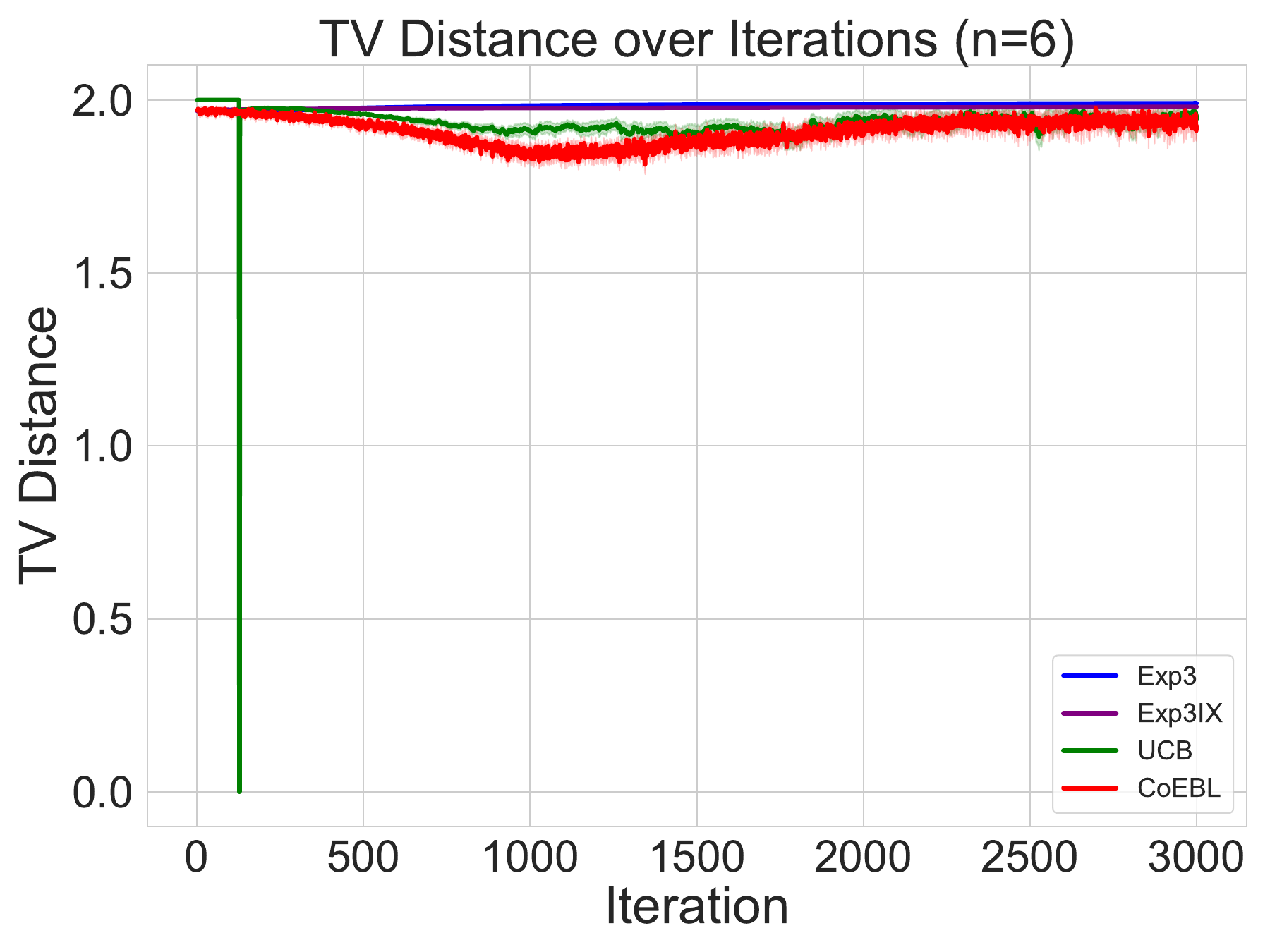}}
    \hfill
    \subfloat[$n=7$
    ]{%
       \includegraphics[width=0.33\linewidth]{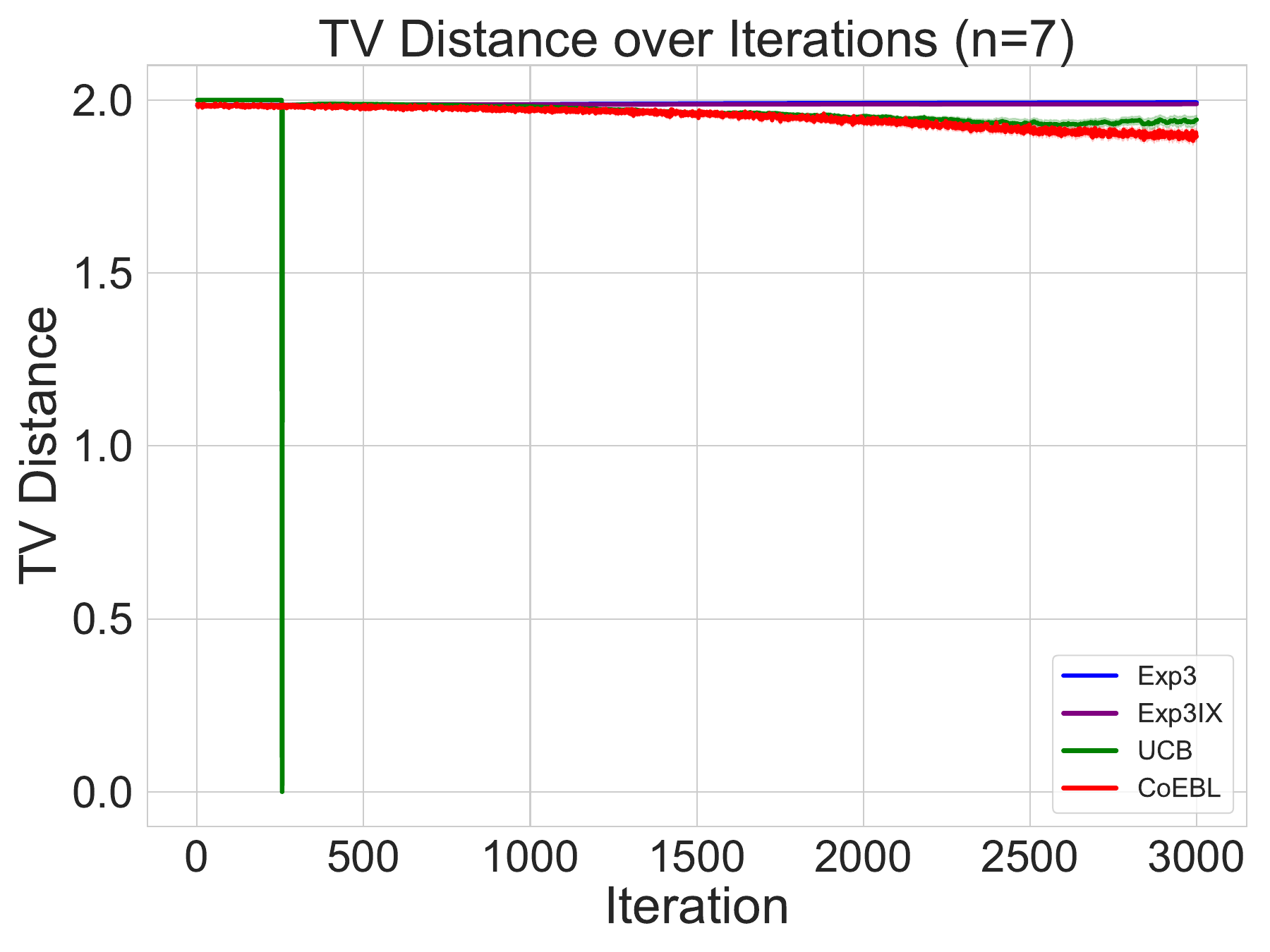}}
    \hfill
    \caption{Regret and TV Distance for Self-Plays on \Diagonal for $n=2, \ldots, 7$}
    \label{fig:Comparison_Regret_TV_Diagonal_2} 
 \end{figure}

 \begin{figure}[!ht]
    \centering
    \subfloat[
    % Regret (n=2)
    ]{%
       \includegraphics[width=0.33\linewidth]{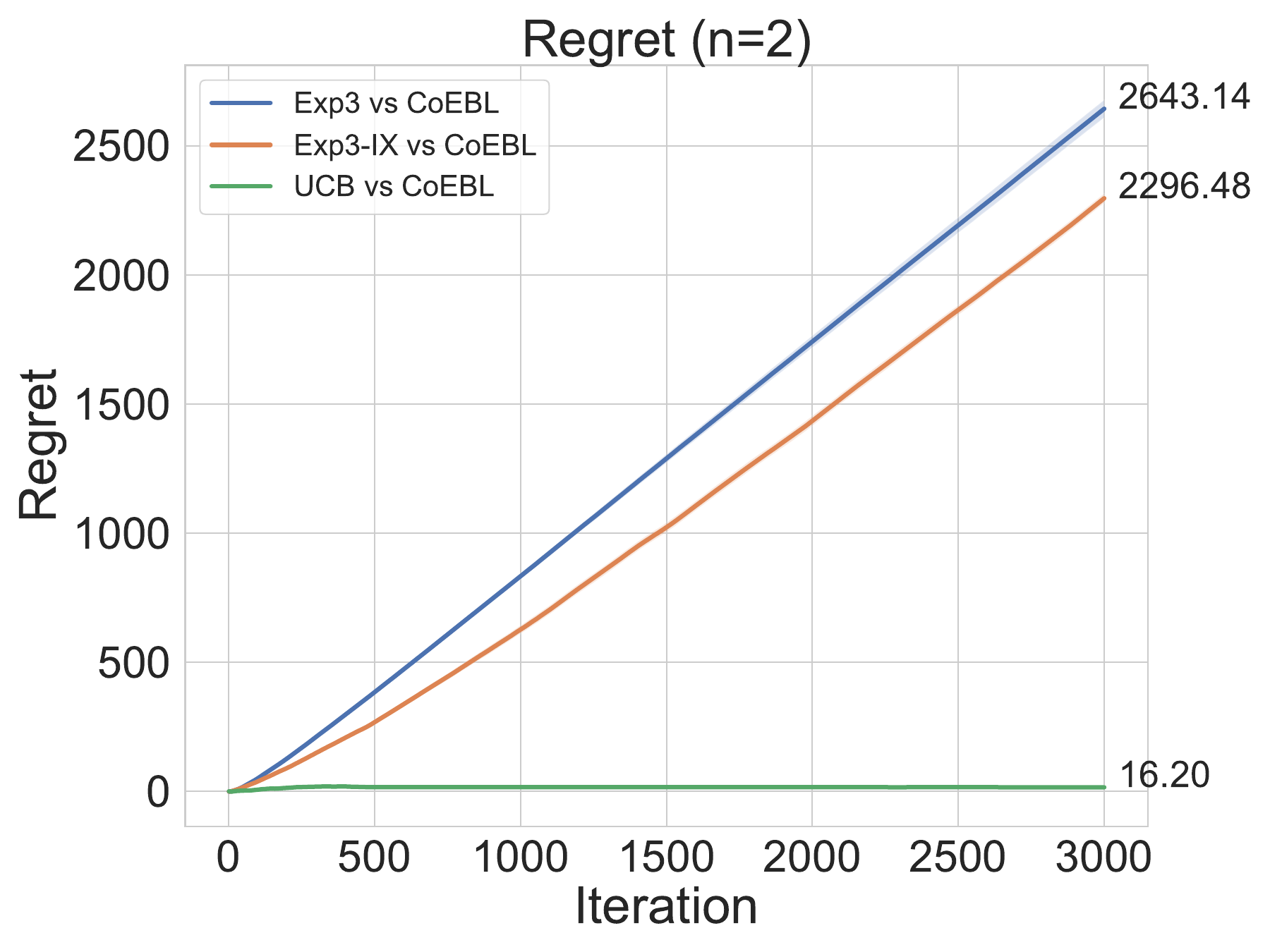}}
    \hfill
    \subfloat[
    % Regret (n=3)
    ]{%
       \includegraphics[width=0.33\linewidth]{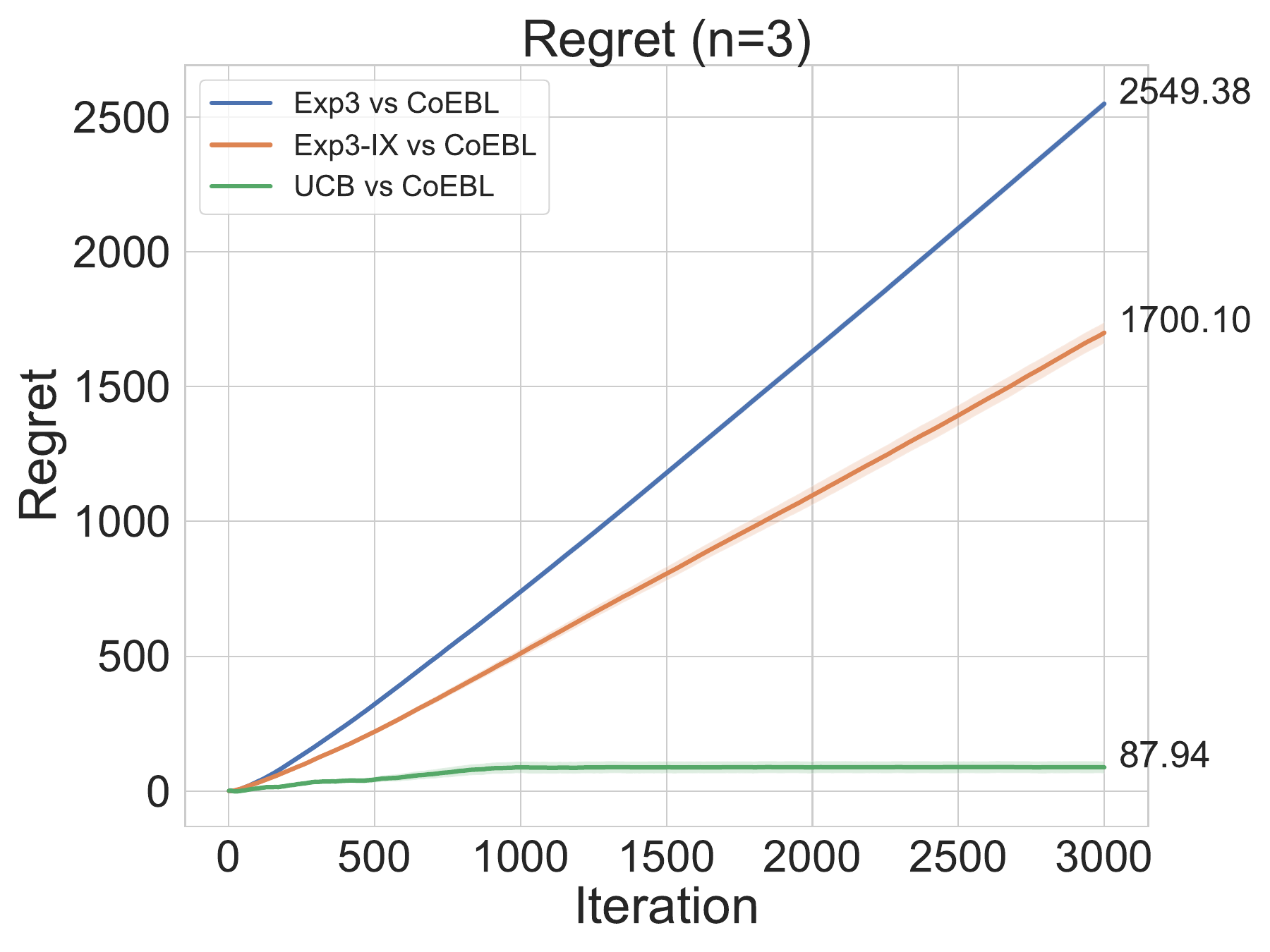}}
    \hfill
    \subfloat[
    % Regret (n=4)
    ]{%
       \includegraphics[width=0.33\linewidth]{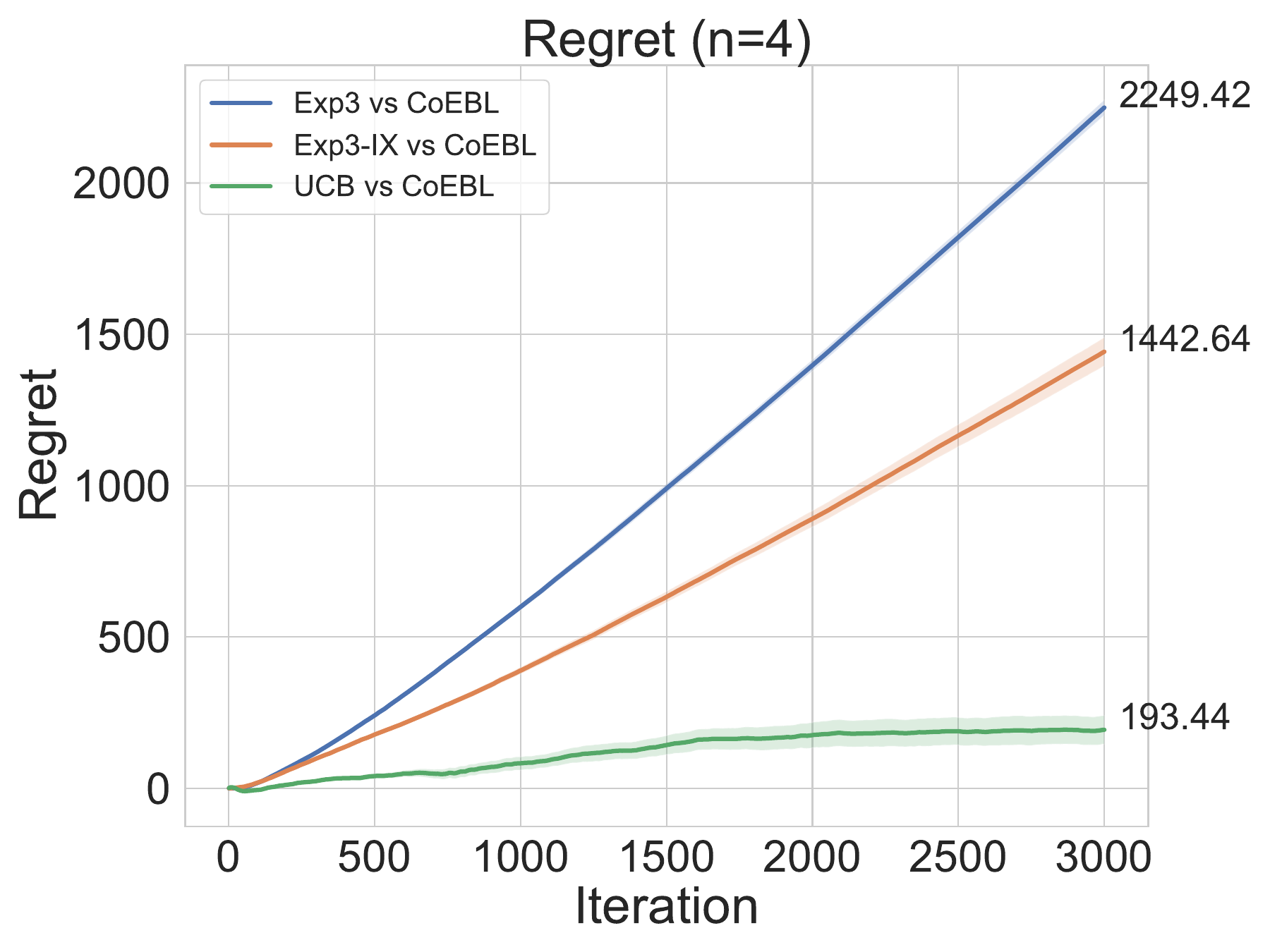}}
    \hfill
    \subfloat[
    % Regret (n=5)
    ]{%
       \includegraphics[width=0.33\linewidth]{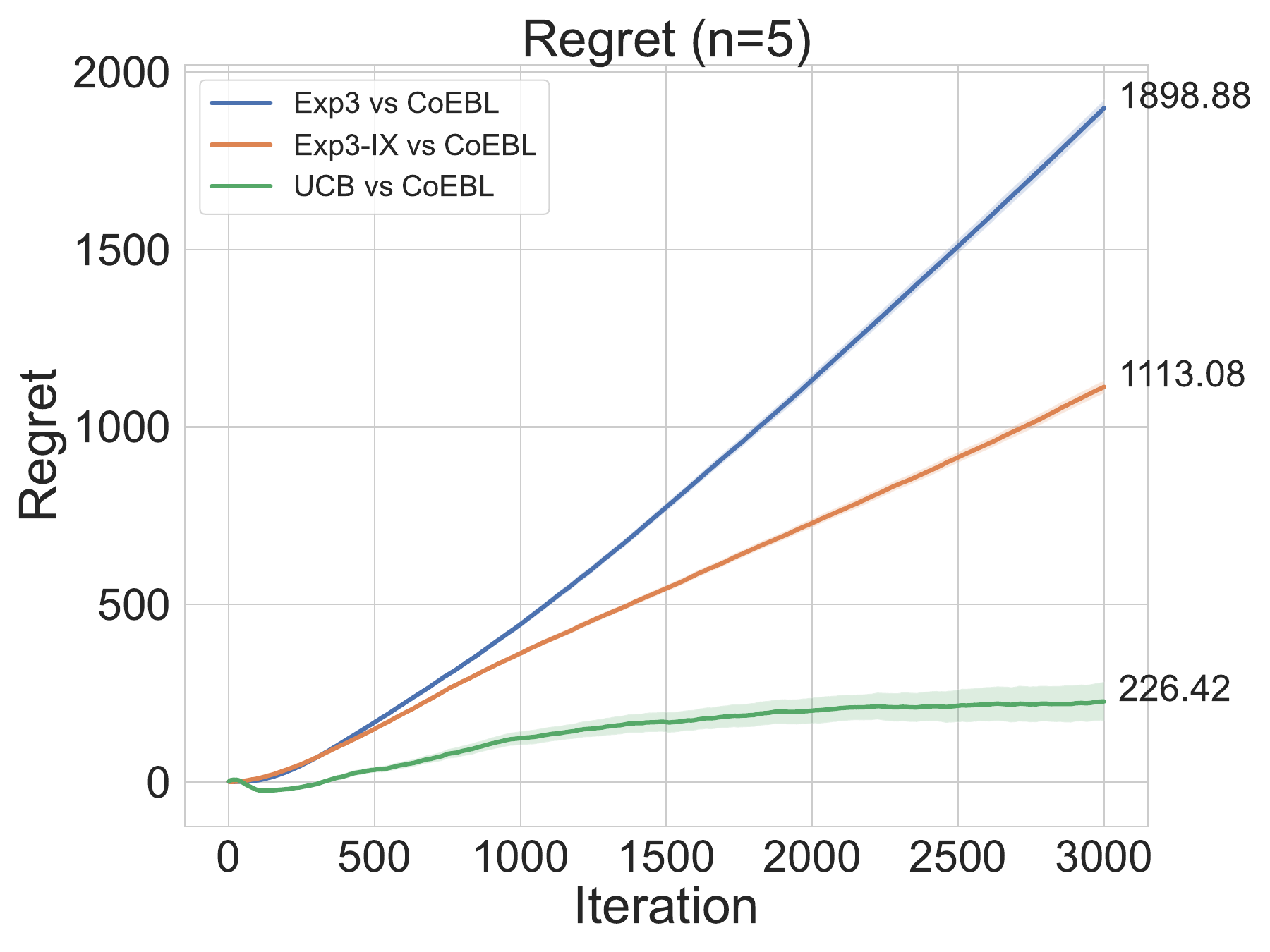}}
    \hfill
    \subfloat[
    % Regret (n=6)
    ]{%
       \includegraphics[width=0.33\linewidth]{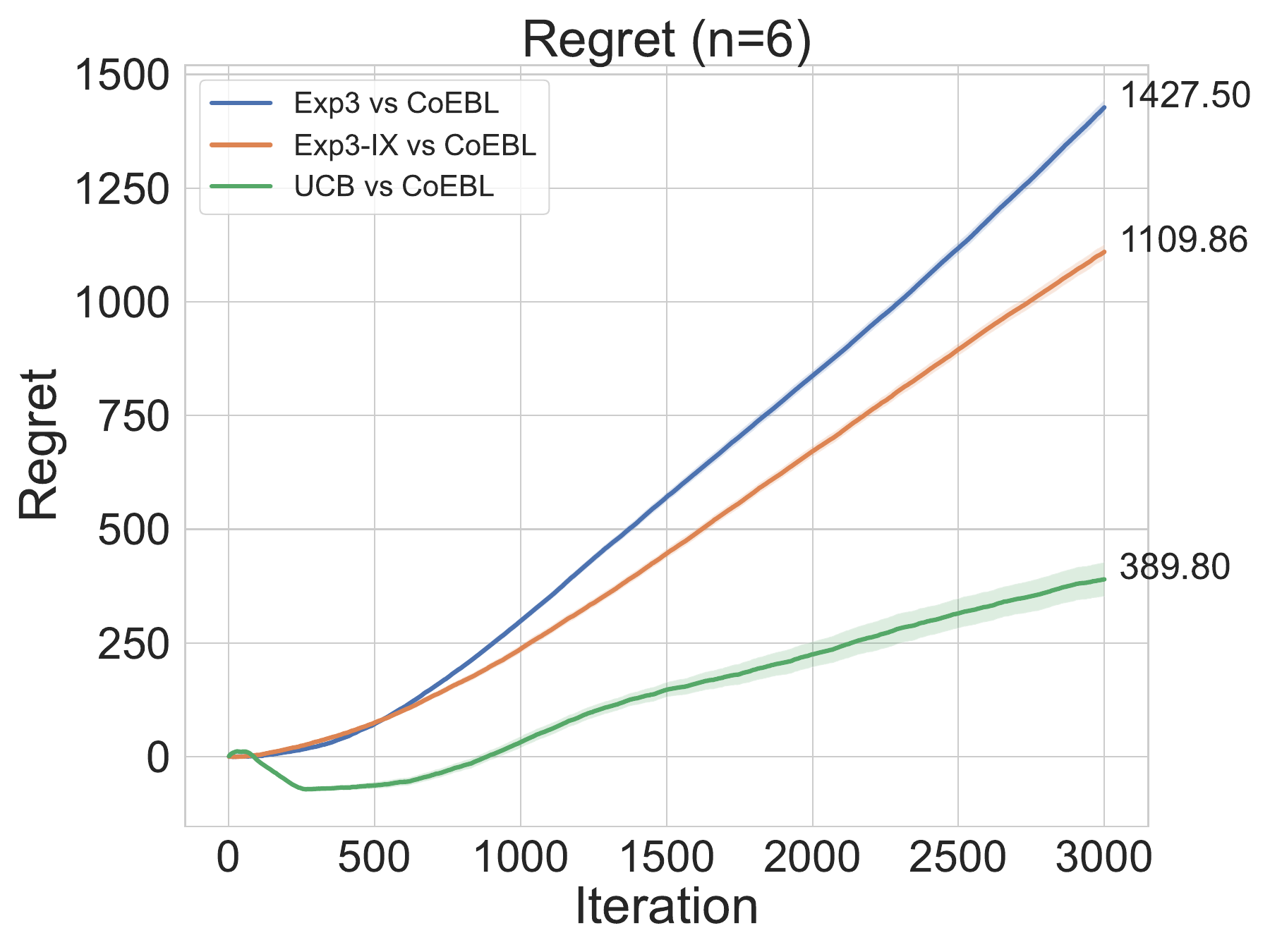}}
    \hfill
    \subfloat[
    % Regret (n=7)
    ]{%
       \includegraphics[width=0.33\linewidth]{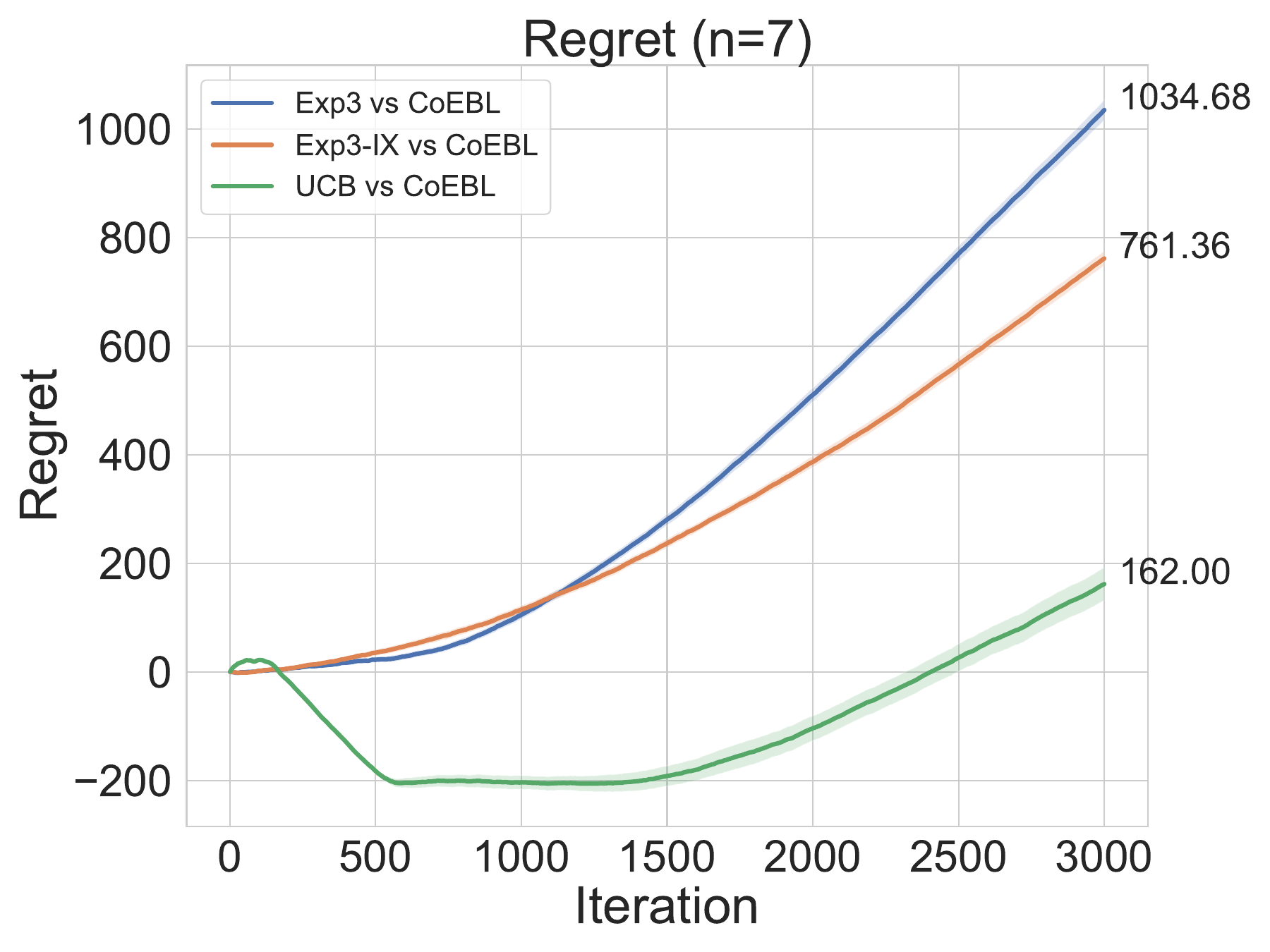}}
    \hfill
    \subfloat[
    % TV (n=2)
    ]{%
       \includegraphics[width=0.33\linewidth]{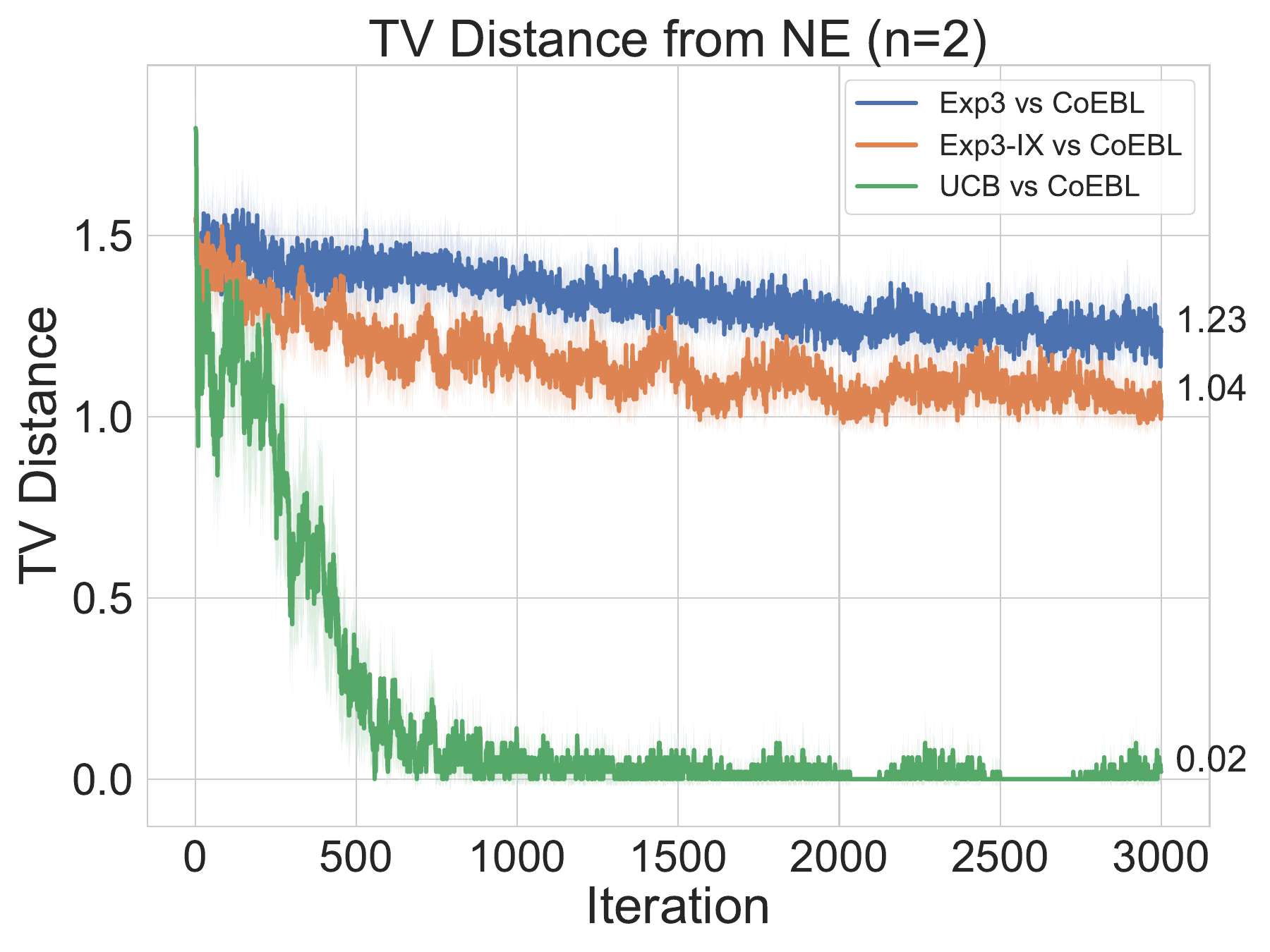}}
    \hfill
    \subfloat[
    % TV (n=3)
    ]{%
       \includegraphics[width=0.33\linewidth]{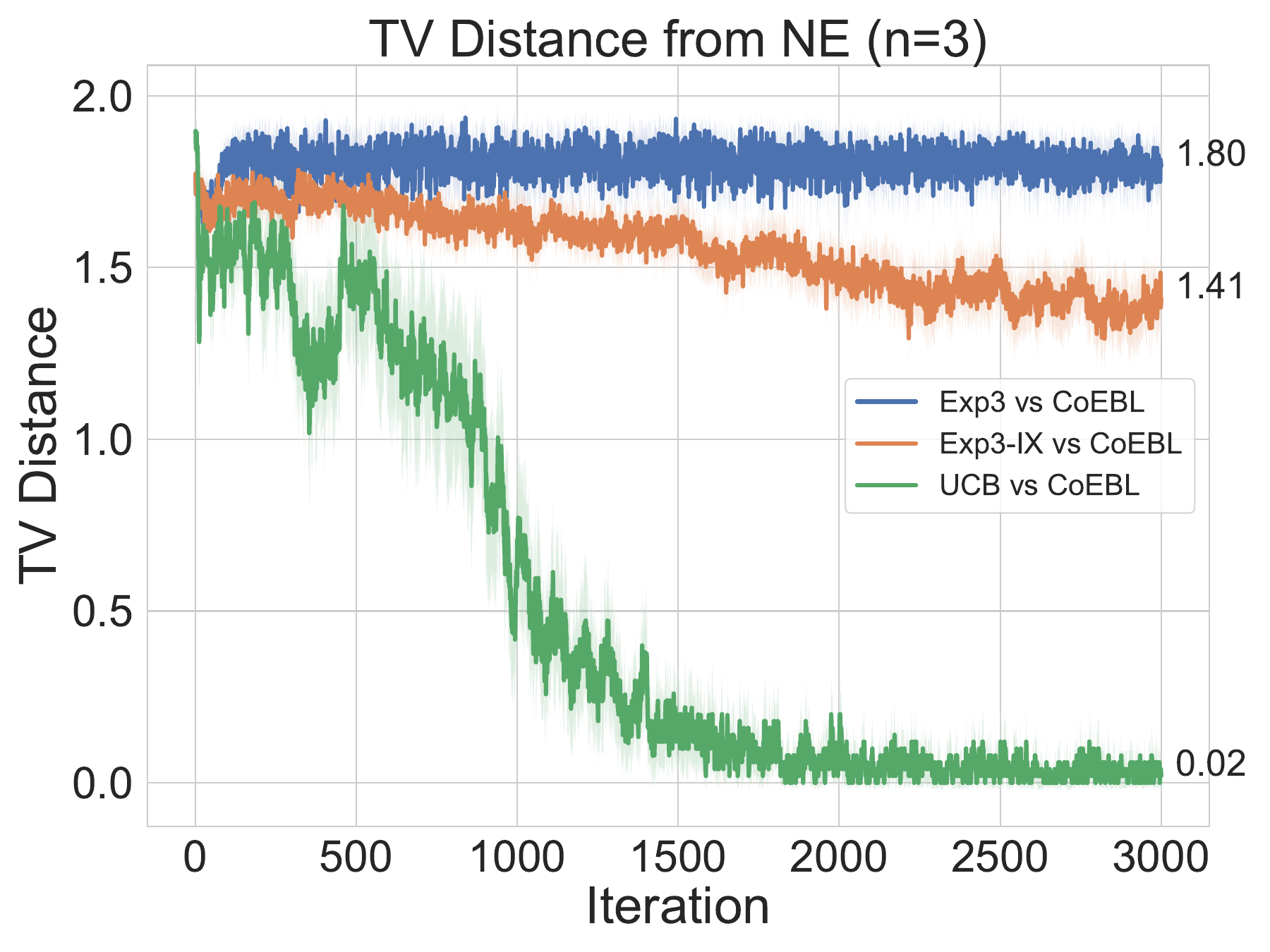}}
    \hfill
    \subfloat[
    % TV (n=4)
    ]{%
       \includegraphics[width=0.33\linewidth]{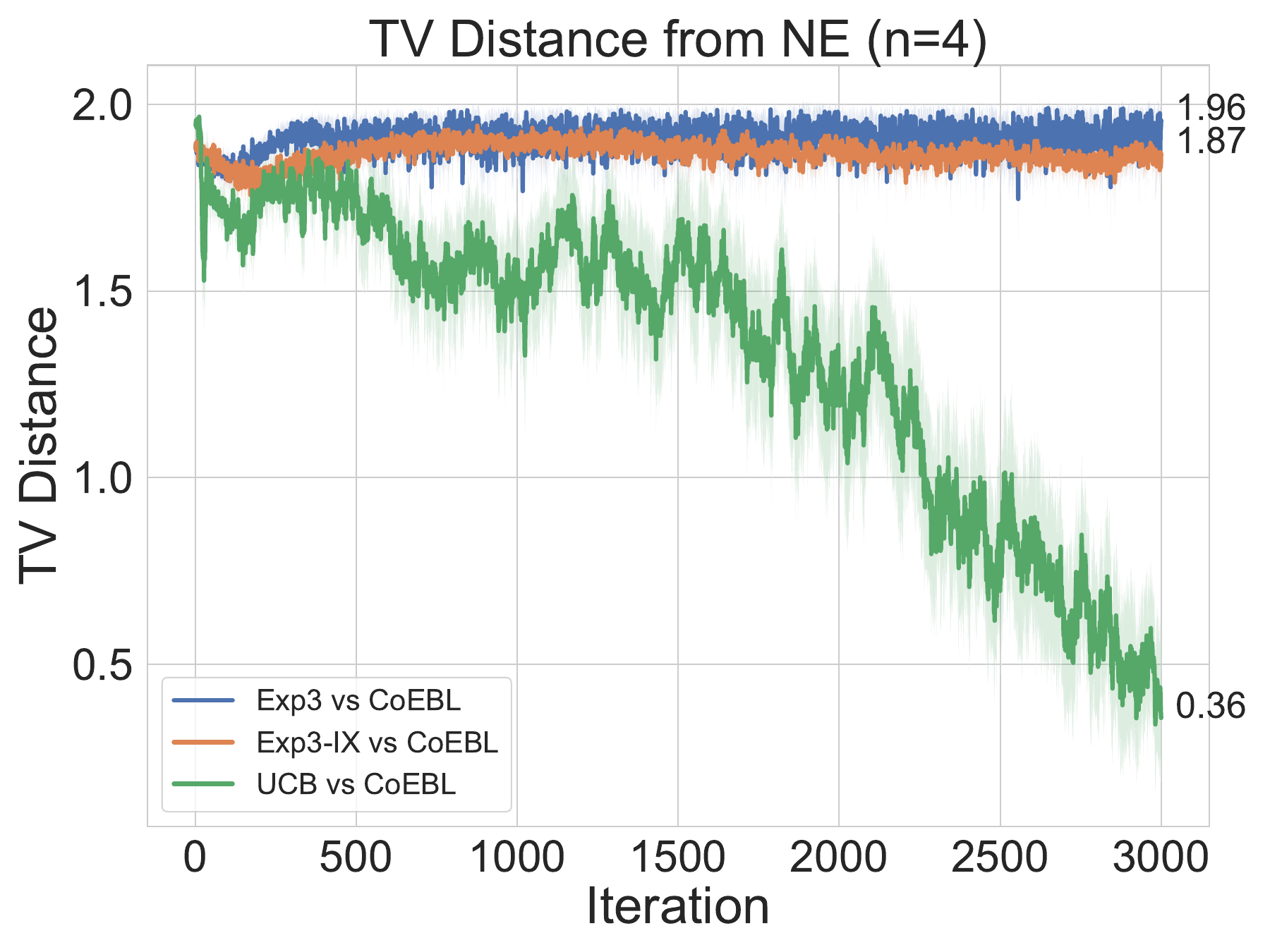}}
    \hfill
    \subfloat[
    % TV (n=5)
    ]{%
       \includegraphics[width=0.33\linewidth]{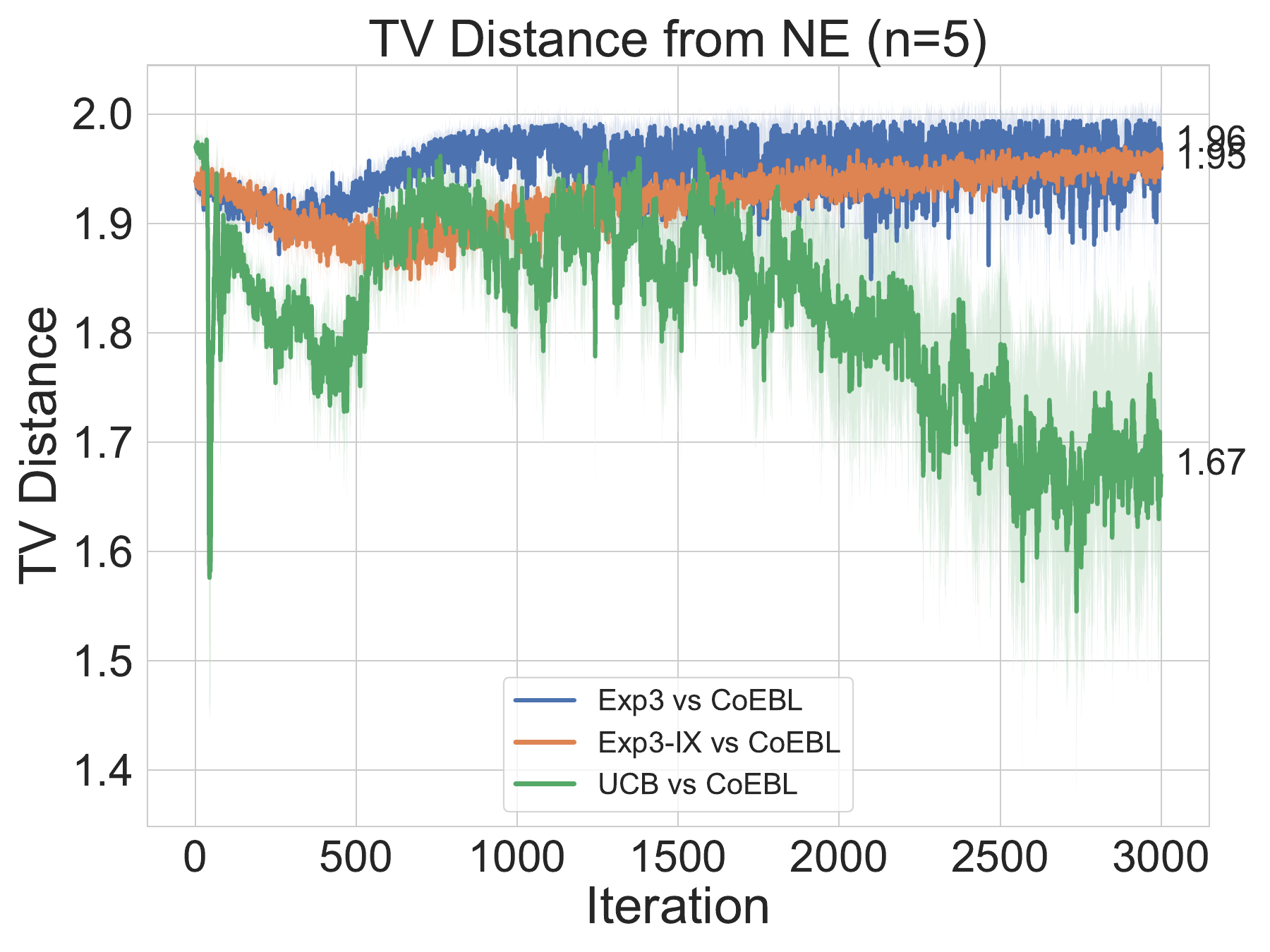}}
    \hfill
    \subfloat[
    % TV (n=6)
    ]{%
       \includegraphics[width=0.33\linewidth]{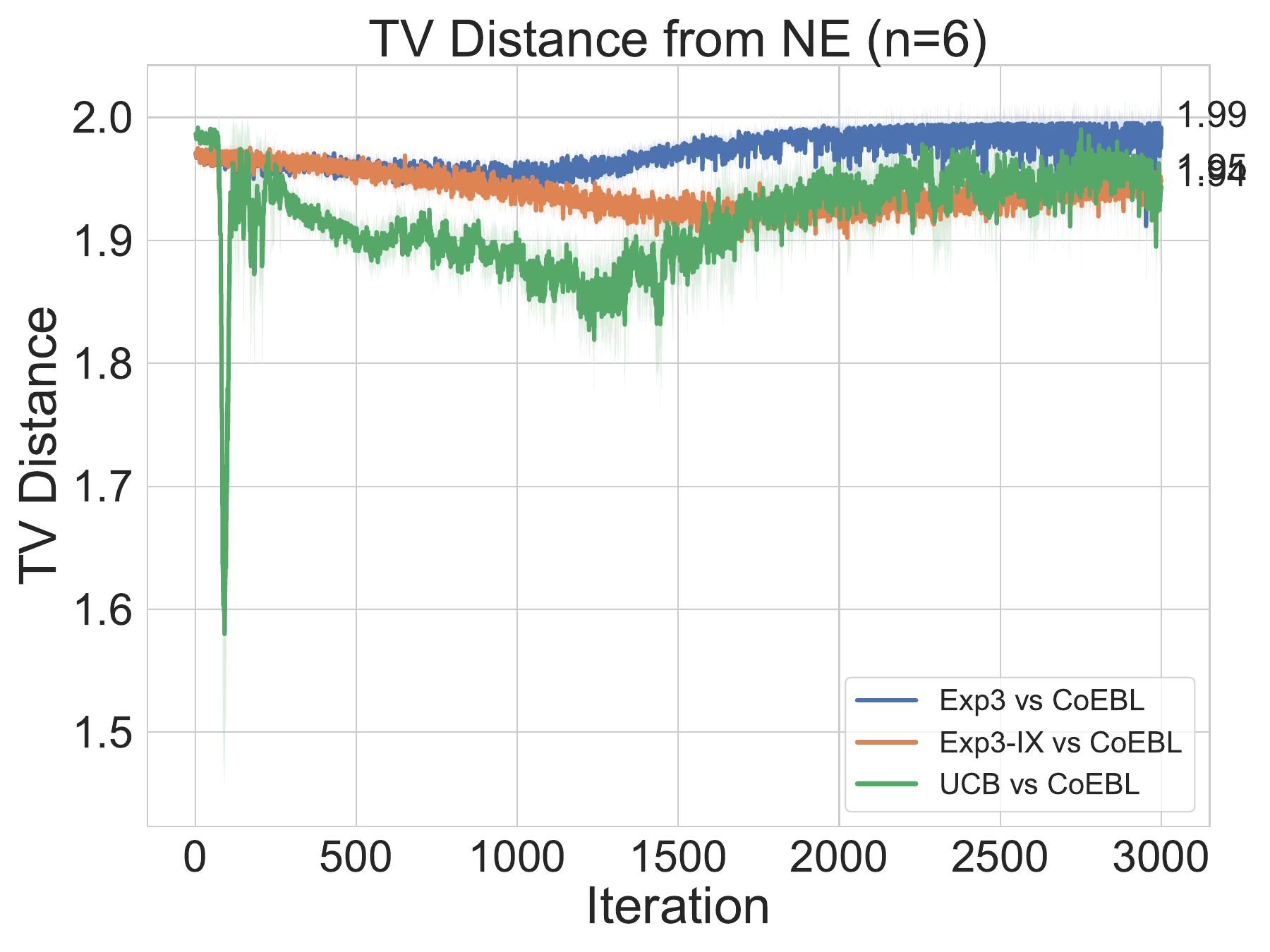}}
    \hfill
    \subfloat[
    % TV (n=7)
    ]{
    \includegraphics[width=0.33\linewidth]{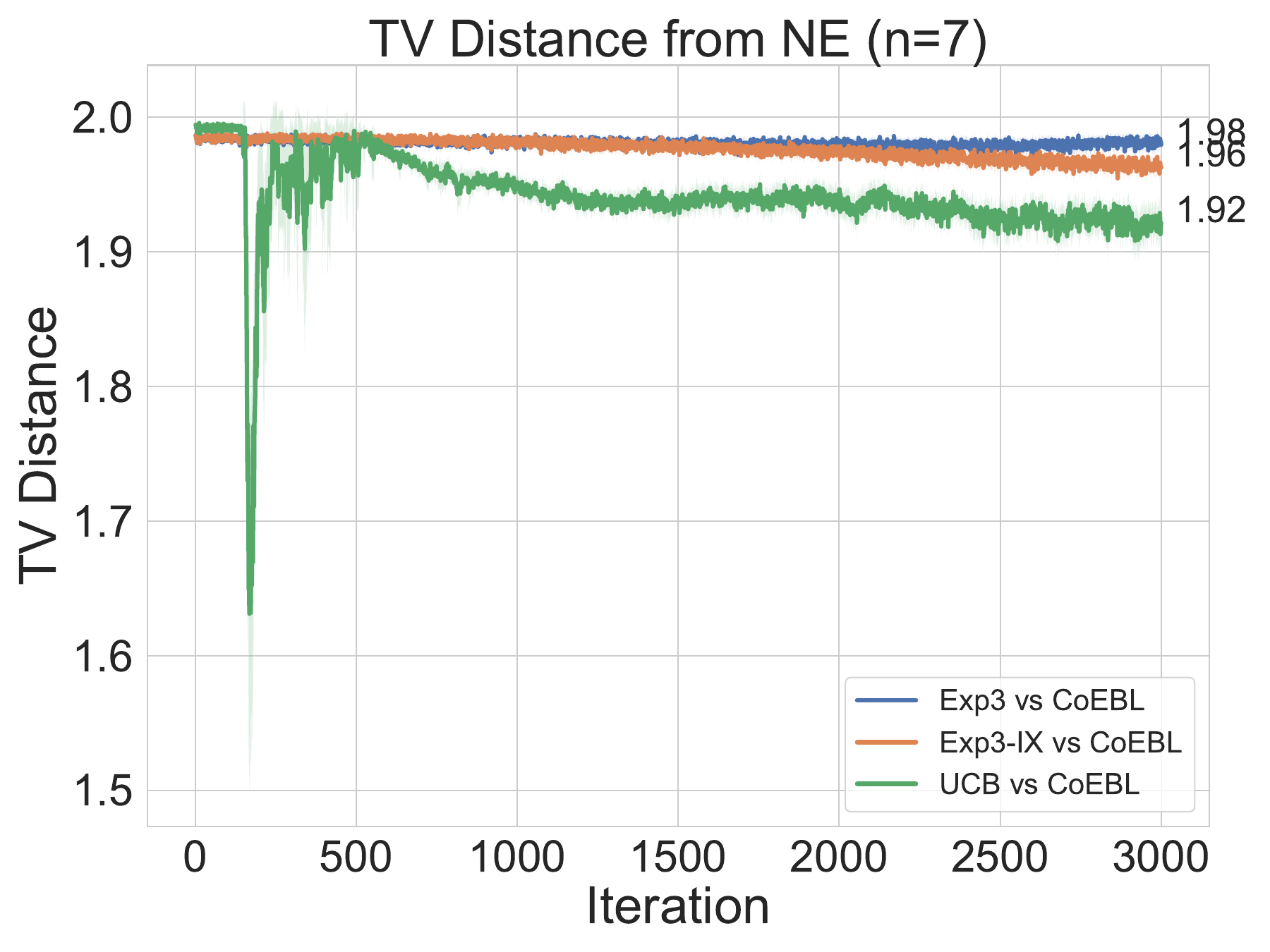}}
    \caption{Regret and TV-distance for $\alg~1$-vs-$\alg~2$ on \Diagonal for $n=2, \ldots, 7$.}
    \hfill
    \label{fig:Regret_Diagonal2_2} 
 \end{figure}
 
 \clearpage
 \subsubsection{\BigNum Game}
 
 \BigNum is another challenging two-player zero-sum game proposed and analysed by \citep{ijcai2024p336}.
 In this game, each player aims to select a number greater than their opponent's.
 The players' action space is defined as $\X = \{0,1\}^n$, where binary bitstrings of length $n$ correspond to natural numbers in the range $[0, 2^n-1]$.
 If the players select the same number, they receive a payoff of $0$. 
 If the difference between the players' numbers is exactly $1$, the player with the larger number receives a payoff of $2$, while the player with the smaller number receives $-2$.
 Otherwise, the player with the larger number receives a payoff of $1$, and the player with the smaller number receives a payoff of $-1$.
 To simplify the game and align it with ternary games, we modify the payoff function $\BigNum:\X \times \X \rightarrow \{ -1, 0, 1 \}$ defined by:
 \begin{align*}
    \BigNum(x, y)
    := \begin{cases} 
      0 &  x = y \\
      1 &  x > y \\
     -1 &  x < y 
   \end{cases}.
 \end{align*}
 The payoff matrix of the \BigNum game for $n=2$ is:
 
 \begin{table}[H]
 \centering
 \begin{tabular}{c|c c c c}
    & 00 & 01 & 10 & 11 \\ \hline
 00 & 0  & -1  & -1  & -1  \\
 01 & 1  & 0   & -1  & -1  \\
 10 & 1  & 1   & 0   & -1  \\
 11 & 1  & 1   & 1   & 0   \\
 \end{tabular}
 \caption{
 The payoff matrix of the \BigNum game for $n=2$.
 Binary bitstrings represent the pure strategies available to each player: $0 = (00)_2$, $1 = (01)_2$, $2 = (10)_2$, and $3 = (11)_2$. In this game, players compare their numbers from $\mathbb{N}$.
 }
 \end{table}
 
 As proved by \citet{ijcai2024p336}, this payoff matrix also exhibits a unique pure Nash equilibrium where both players choose $1^n \in \{0,1\}^n$ (i.e., the binary string of all ones, corresponding to $2^n-1 \in \mathbb{N}$). 
 This corresponds to the mixed Nash equilibrium $x^*=(0, \cdots ,1)$ and $y^*=(0,\cdots ,1)$.
 We conduct experiments using
 Algorithms~\ref{alg:Exp3} to~\ref{alg:MGBF} and compare them with our proposed Algorithm~\ref{alg:CoEBL} 
 (i.e. \coebl) on this matrix game benchmark, the \BigNum game.

 \begin{figure}[!ht]
    \centering
    \subfloat[$n=2$
    ]{%
       \includegraphics[width=0.33\linewidth]{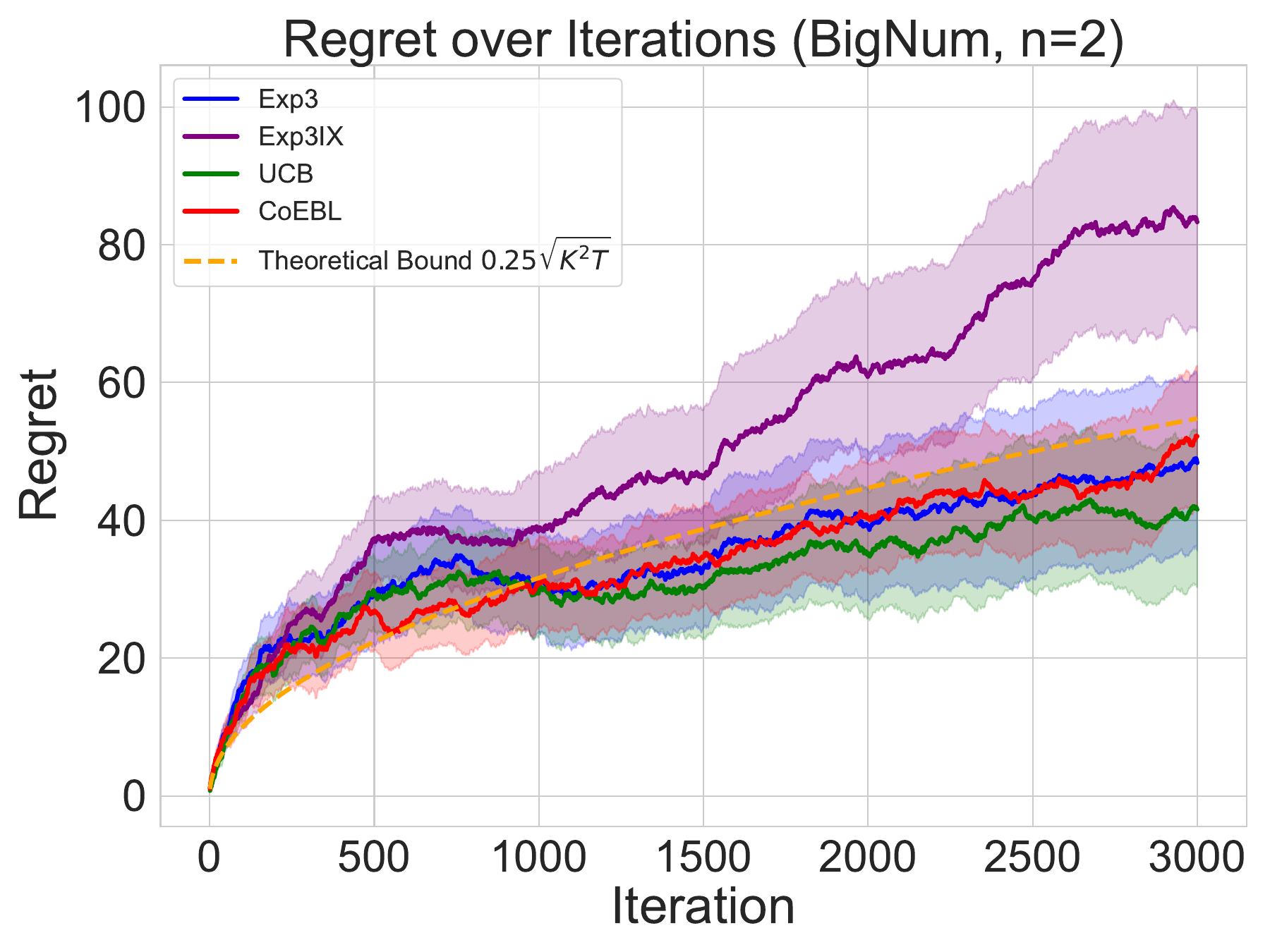}}
    \hfill
    \subfloat[$n=3$
    ]{%
       \includegraphics[width=0.33\linewidth]{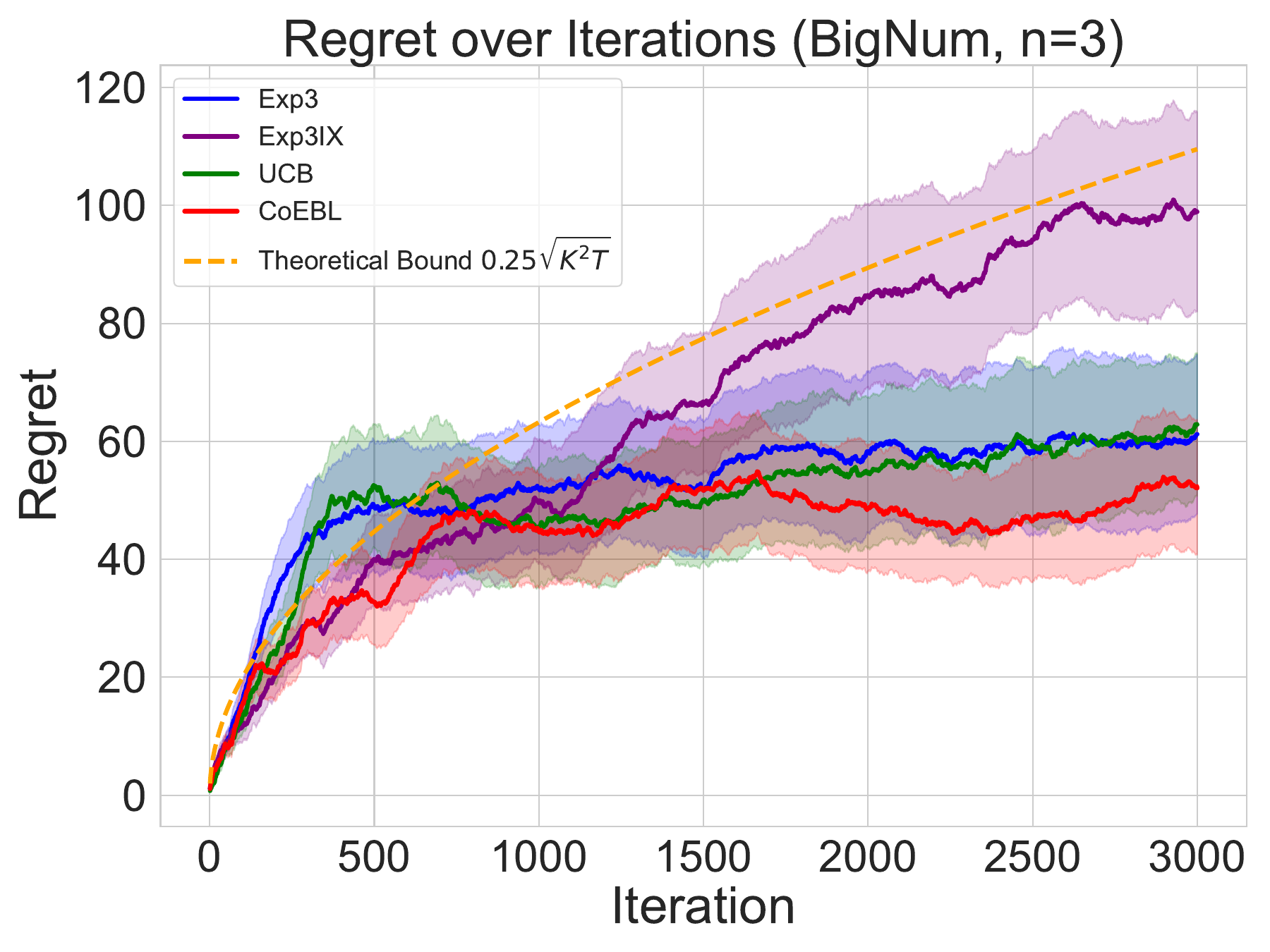}}
    \hfill
    \subfloat[$n=4$
    ]{%
       \includegraphics[width=0.33\linewidth]{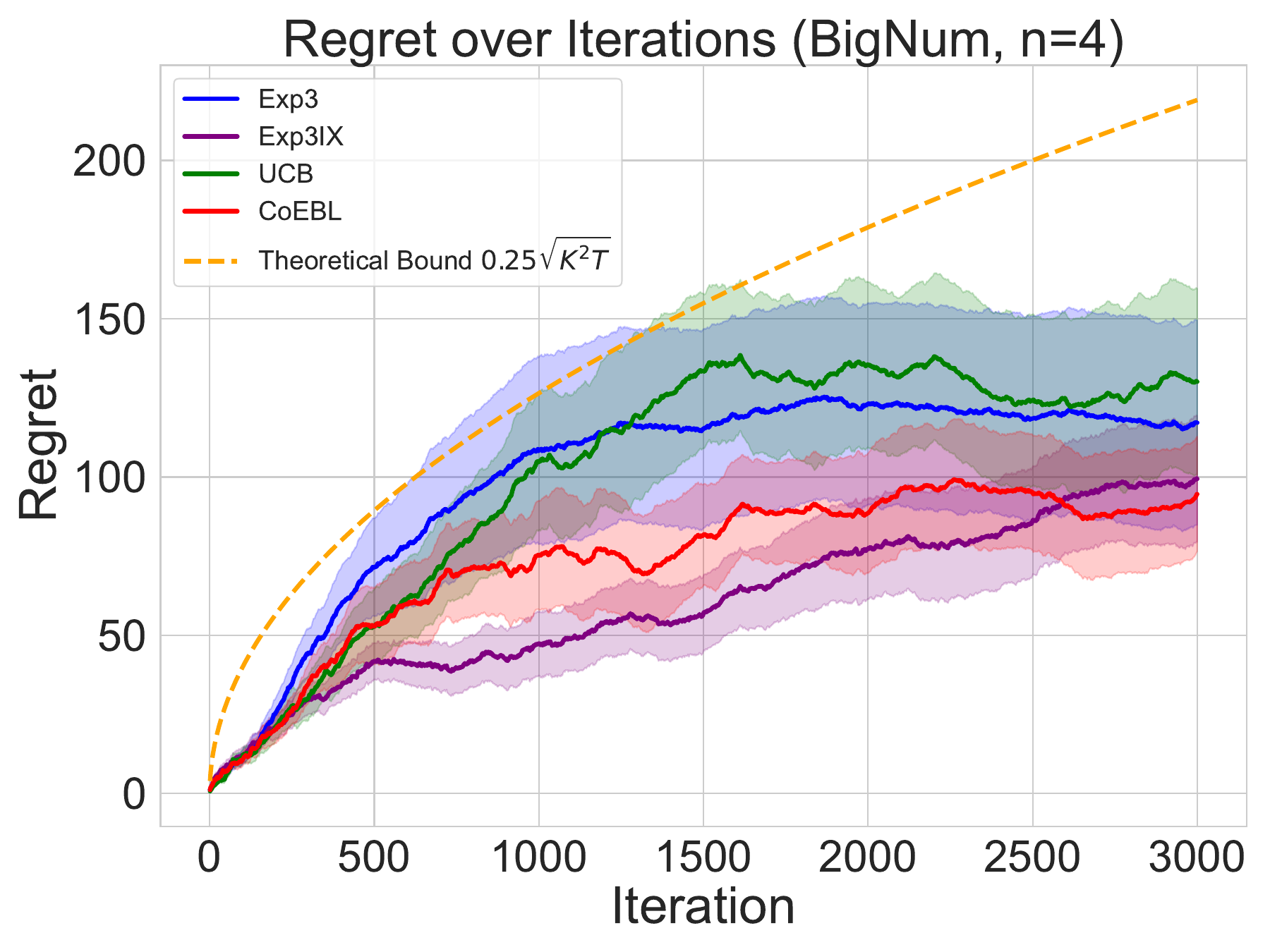}}
    \hfill
    \subfloat[$n=5$
    ]{%
       \includegraphics[width=0.33\linewidth]{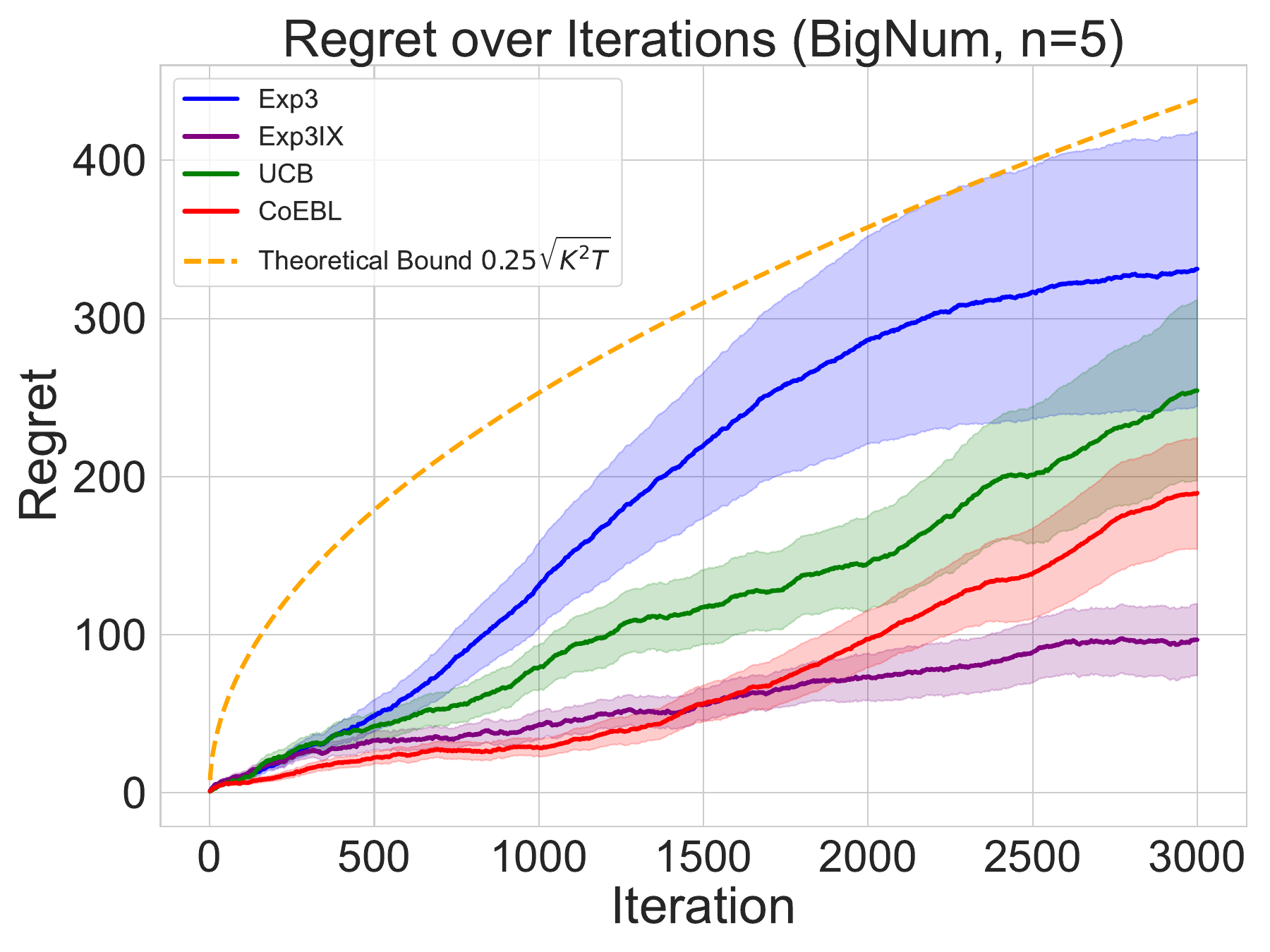}}
    \hfill
    \subfloat[$n=6$
    ]{%
       \includegraphics[width=0.33\linewidth]{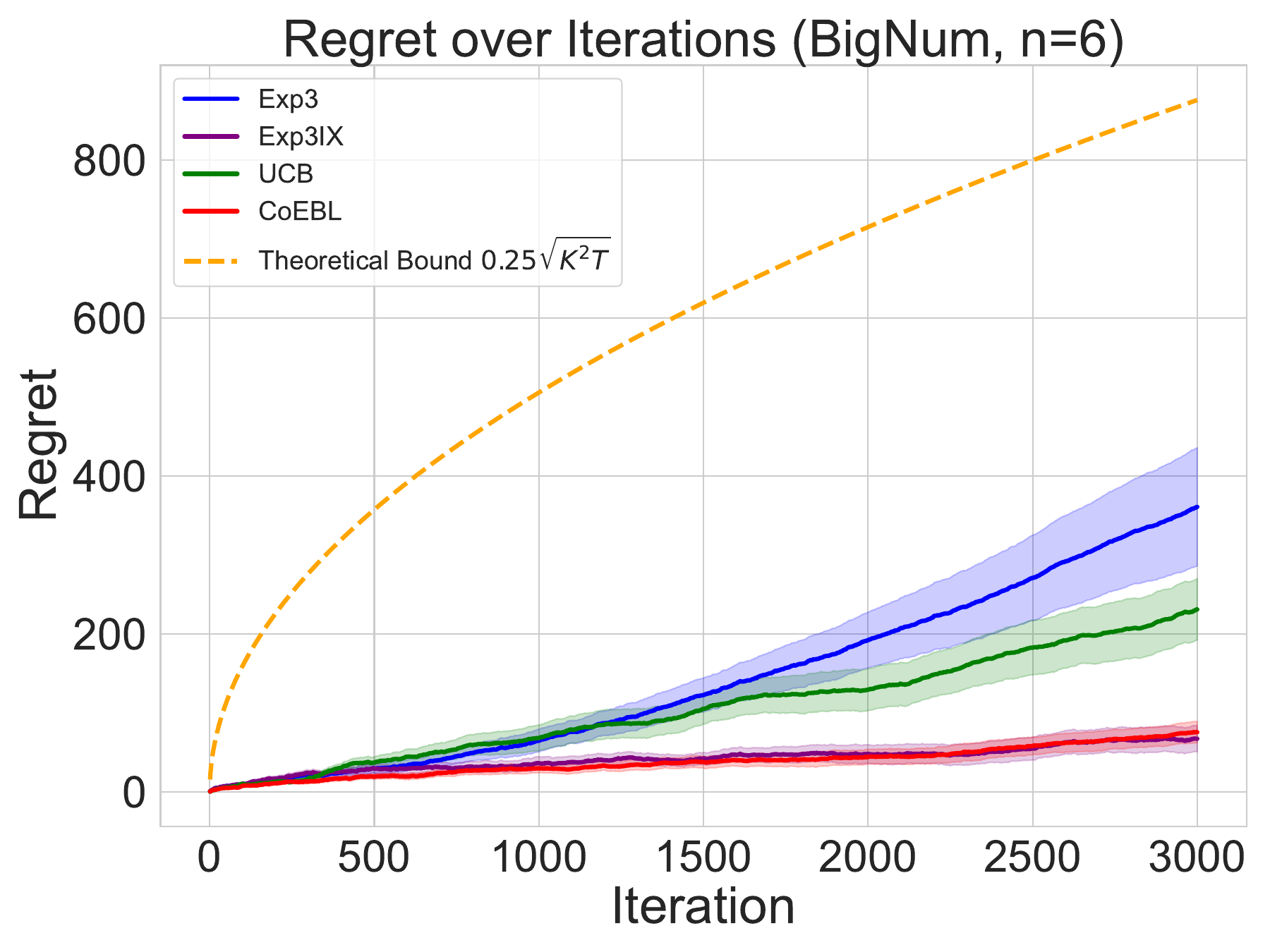}}
    \hfill
    \subfloat[$n=7$
    ]{%
       \includegraphics[width=0.33\linewidth]{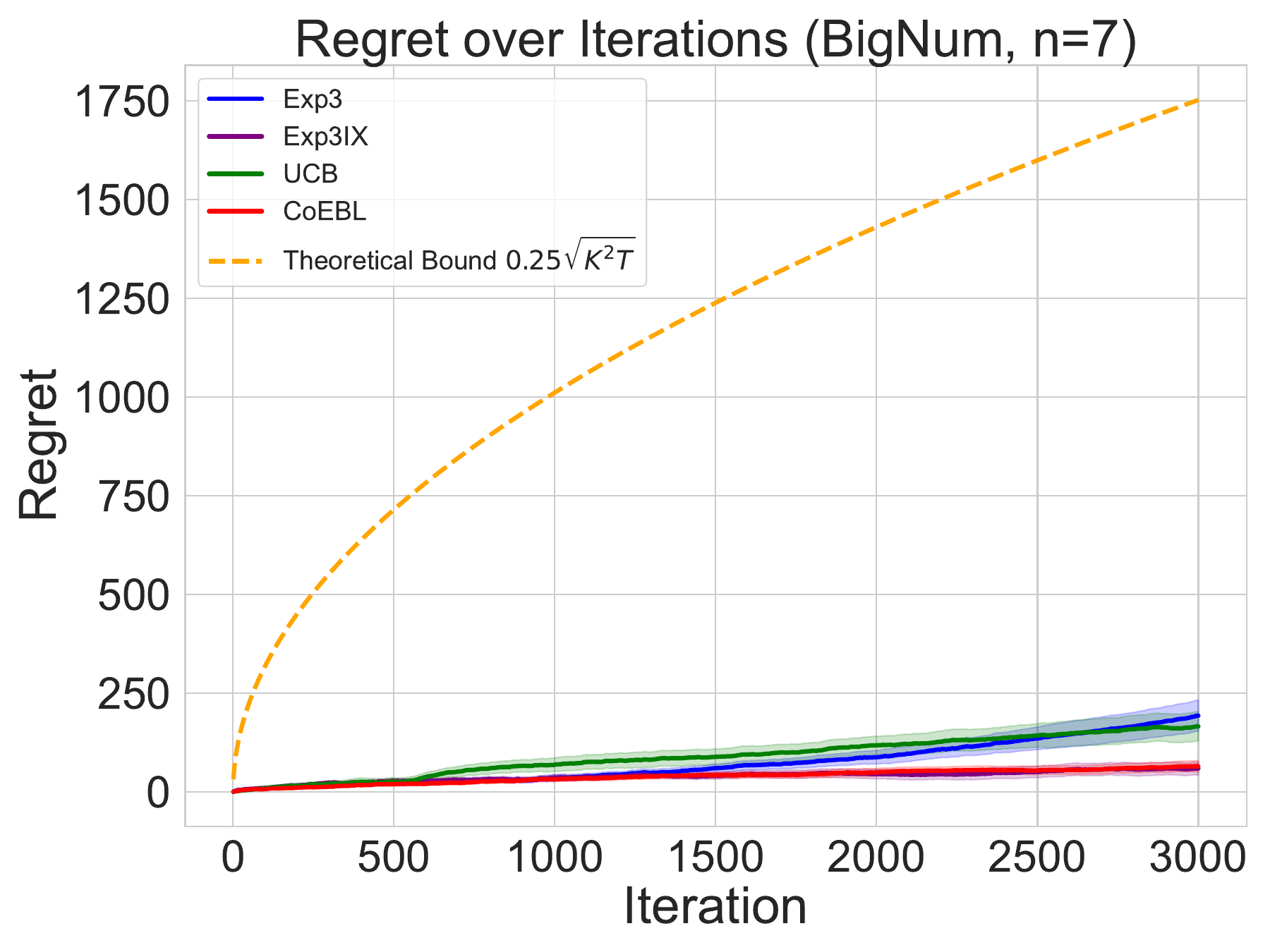}}
    \hfill
    \subfloat[$n=2$
    ]{%
       \includegraphics[width=0.33\linewidth]{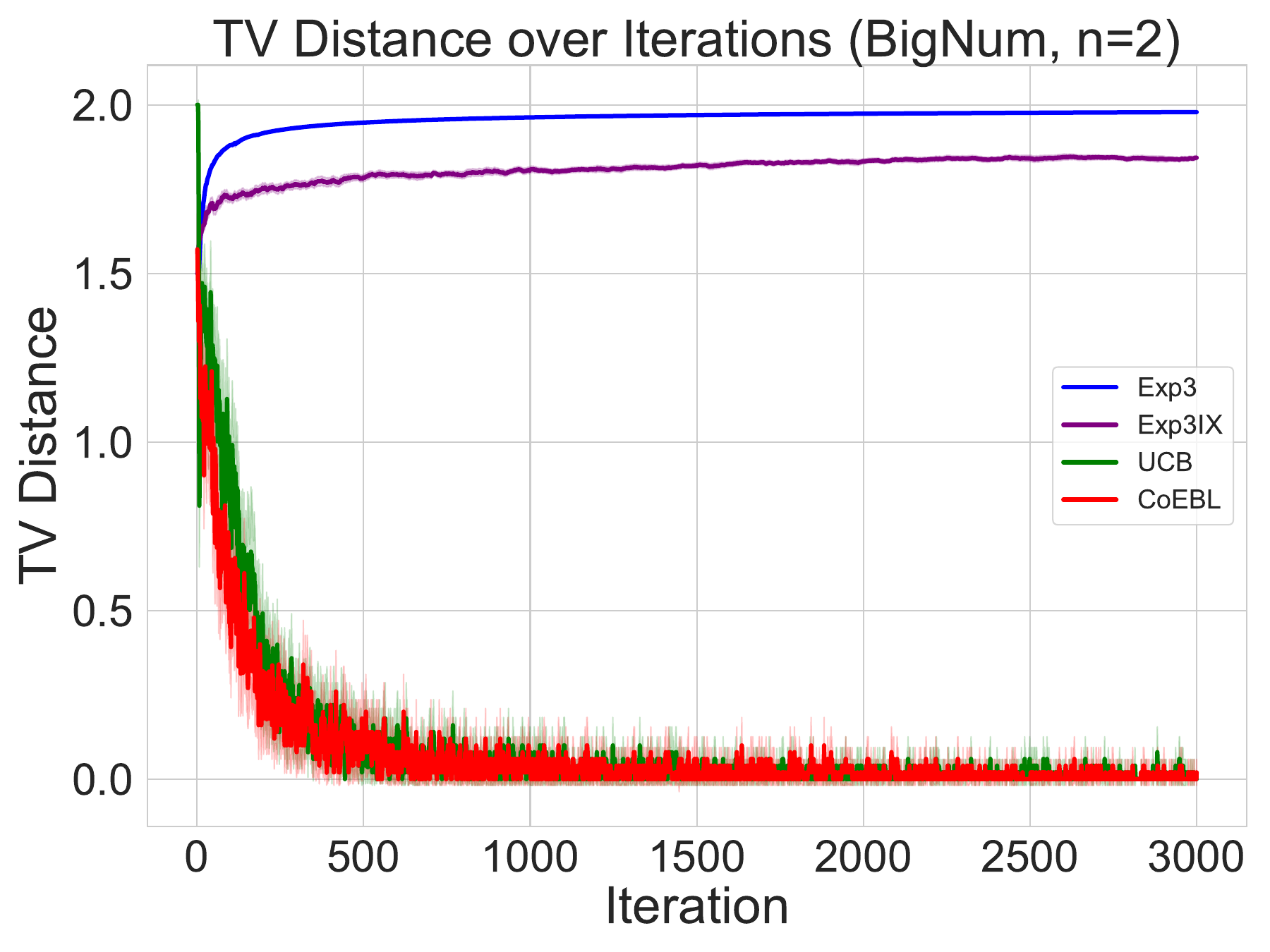}}
    \hfill
    \subfloat[$n=3$
    ]{%
       \includegraphics[width=0.33\linewidth]{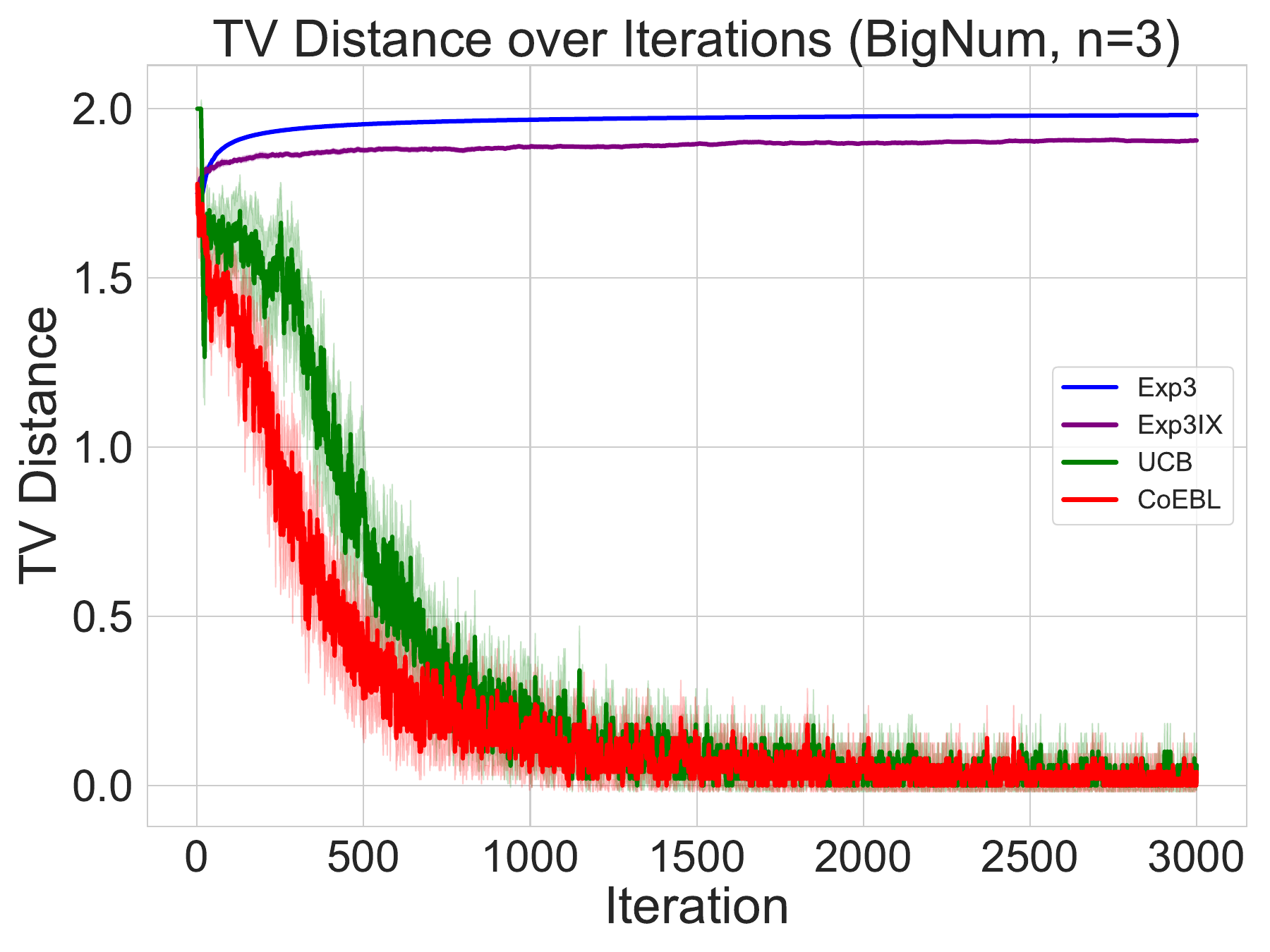}}
    \hfill
    \subfloat[$n=4$
    ]{%
       \includegraphics[width=0.33\linewidth]{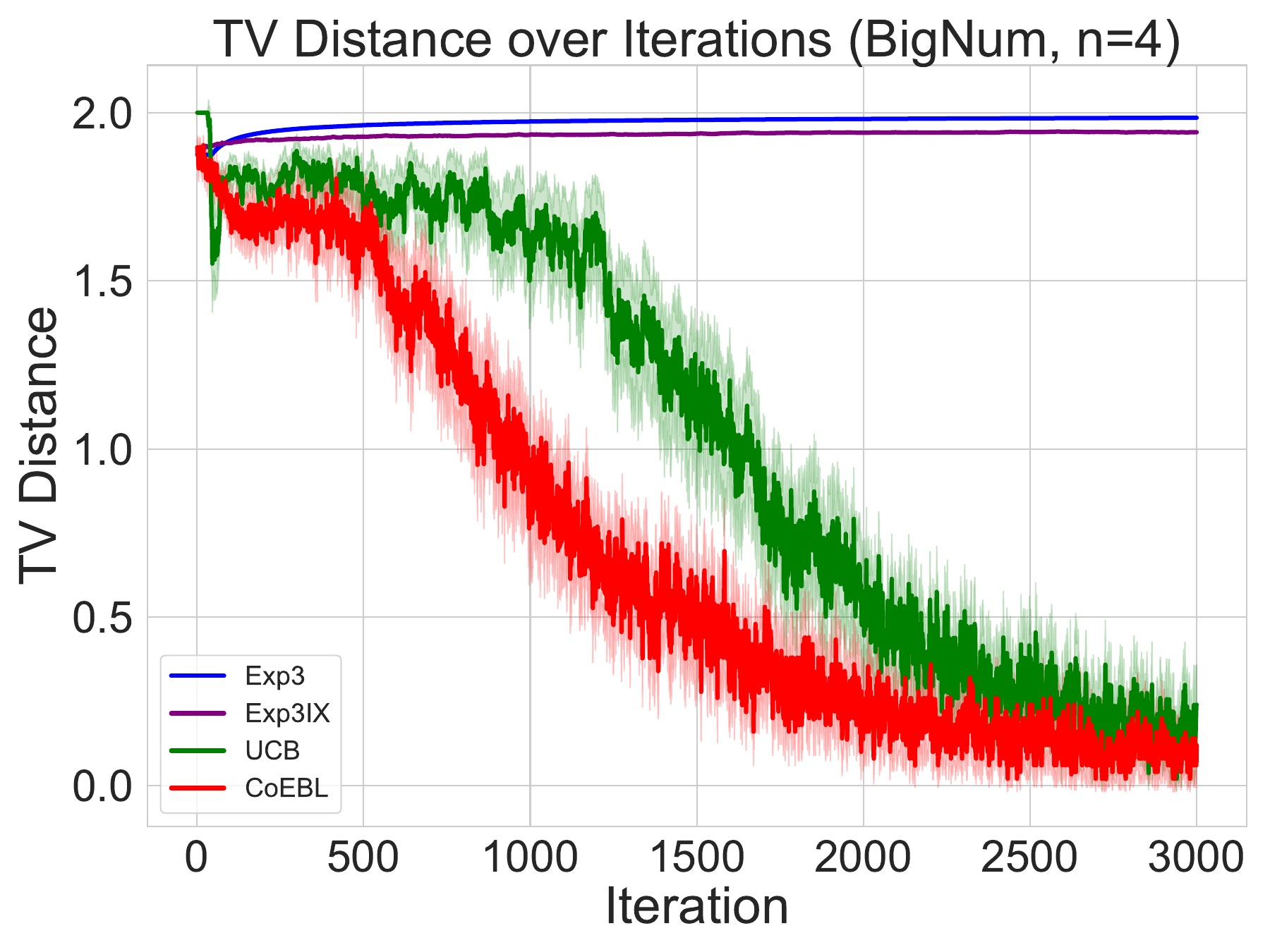}}
    \hfill
    \subfloat[$n=5$
    ]{%
       \includegraphics[width=0.33\linewidth]{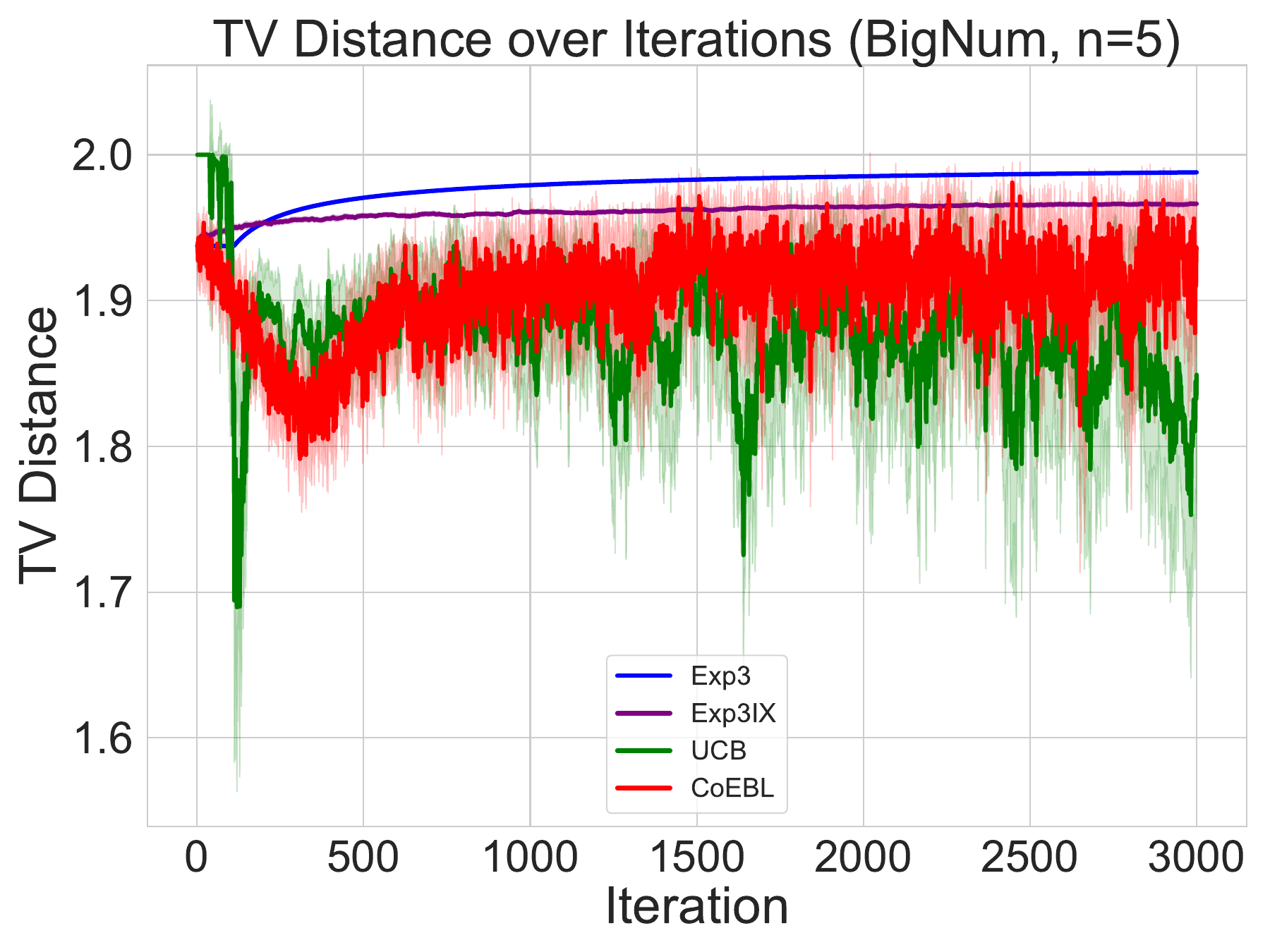}}
    \hfill
    \subfloat[$n=6$
    ]{%
       \includegraphics[width=0.33\linewidth]{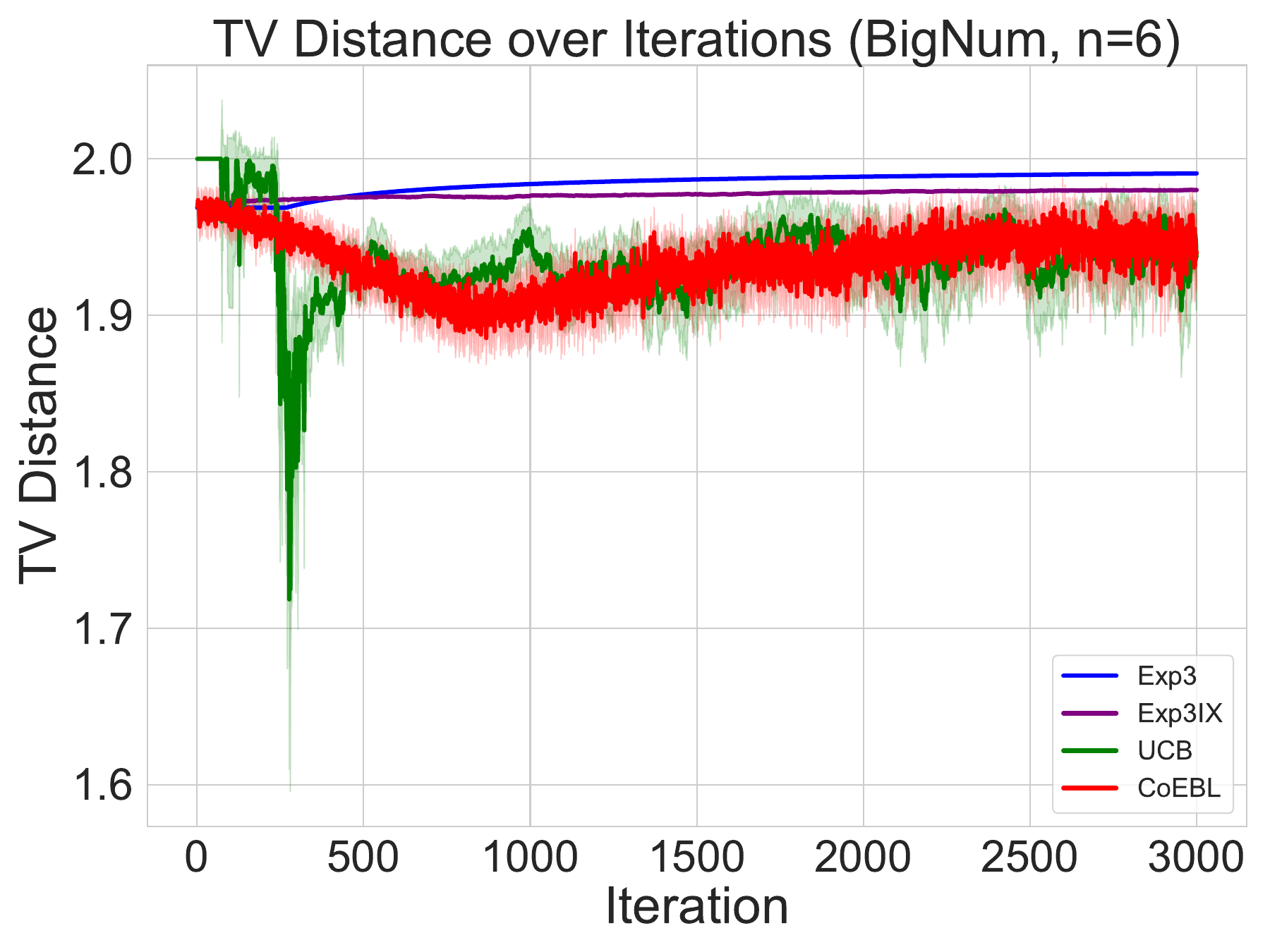}}
    \hfill
    \subfloat[$n=7$
    ]{%
       \includegraphics[width=0.33\linewidth]{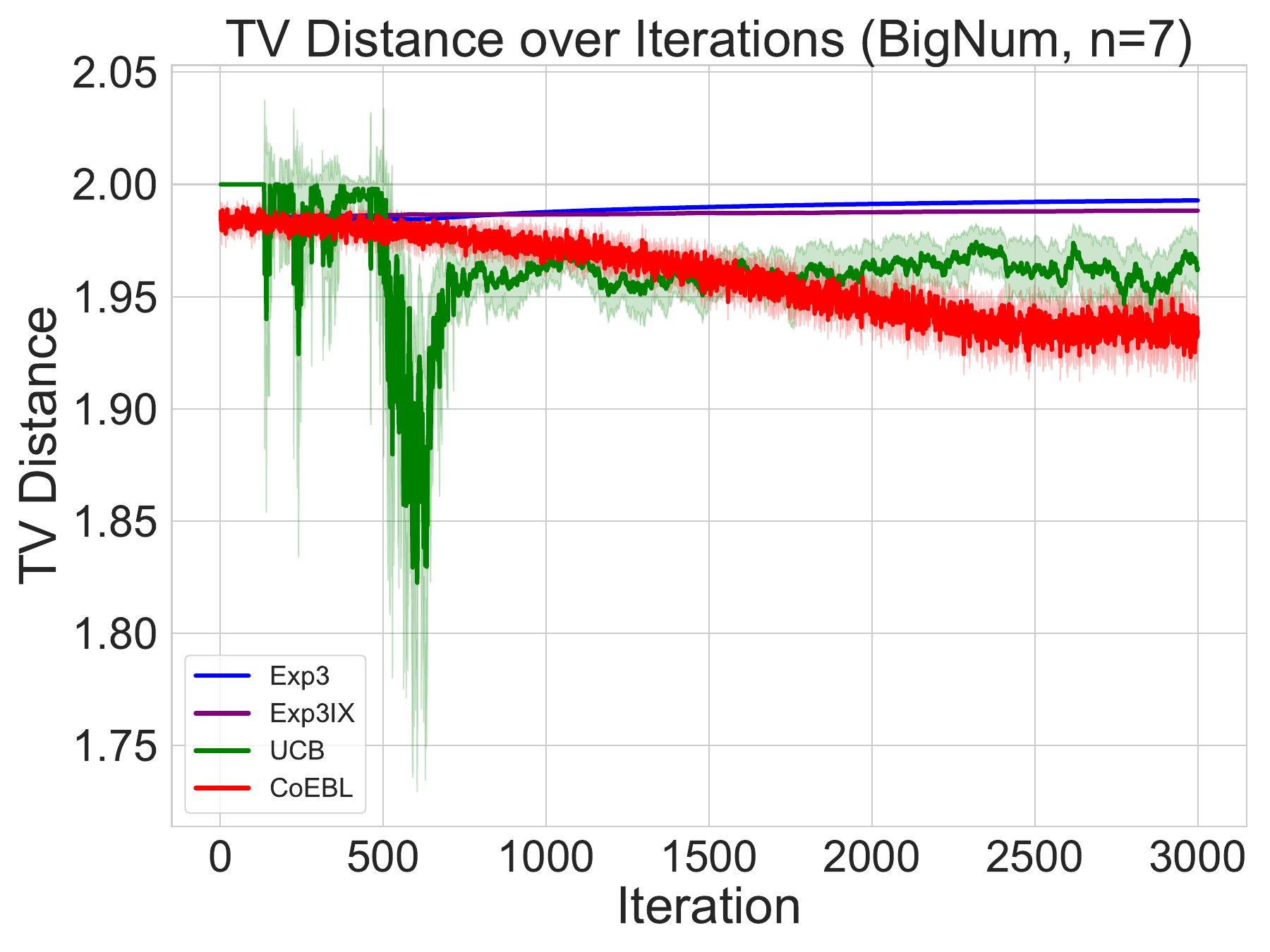}}
    \hfill
    \caption{Regret and TV Distance for Self-Plays on \BigNum}
    \label{fig:Comparison_Regret_TV_BigNum} 
 \end{figure}
 
 % \begin{figure}[!ht]
 %    \centering
 %    \subfloat[$n=2$
 %    ]{%
 %       \includegraphics[width=0.33\linewidth]{figures/BigNum/SelfPlay/regret_comparison_n2.pdf}}
 %    \hfill
 %    \subfloat[$n=3$
 %    ]{%
 %       \includegraphics[width=0.33\linewidth]{figures/BigNum/SelfPlay/regret_comparison_n3.pdf}}
 %    \hfill
 %    \subfloat[$n=4$
 %    ]{%
 %       \includegraphics[width=0.33\linewidth]{figures/BigNum/SelfPlay/regret_comparison_n4.pdf}}
 %    \hfill
 %    \subfloat[$n=2$
 %    ]{%
 %       \includegraphics[width=0.33\linewidth]{figures/BigNum/SelfPlay/tvd_comparison_n2.pdf}}
 %    \hfill
 %    \subfloat[$n=3$
 %    ]{%
 %       \includegraphics[width=0.33\linewidth]{figures/BigNum/SelfPlay/tvd_comparison_n3.pdf}}
 %    \hfill
 %    \subfloat[$n=4$
 %    ]{%
 %       \includegraphics[width=0.33\linewidth]{figures/BigNum/SelfPlay/tvd_comparison_n4.pdf}}
 %    \hfill
 %    \caption{Regret and TV Distance for Self-Plays on \BigNum}
 %    \label{fig:Comparison_Regret_TV_BigNum} 
 % \end{figure}

 In Figure~\ref{fig:Comparison_Regret_TV_BigNum}, we present the self-play results of each algorithm on the \BigNum game for various values of $n$.
 We observe that \coebl exhibits sublinear regret in the \BigNum game, similar to other bandit baselines, and aligns with our theoretical bound.
 In terms of convergence measured by TV-distance, \coebl converges to the Nash equilibrium for $n=2,3, 4$, 
 while the other baselines do not converge.
 However, after $n=5$ (as the number of pure strategies increases exponentially), \coebl also fails to converge to the Nash equilibrium.
 
 In Figure~\ref{fig:Regret_BigNum2}, we present the regret and TV-distance for $\alg~1$-vs-$\alg~2$ on \BigNum.
 Similar to the \Diagonal game, we observe that all regret values are positive with minimum $8.39$ and maximum $351.27$, indicating that the minimiser is winning on average. 
 Thus, \coebl outperforms the other bandit baselines in \BigNum for all $n=2, \ldots ,7$.
 
 \begin{figure}[!ht]
    \centering
    \subfloat[
    % Regret (n=2)
    ]{%
       \includegraphics[width=0.33\linewidth]{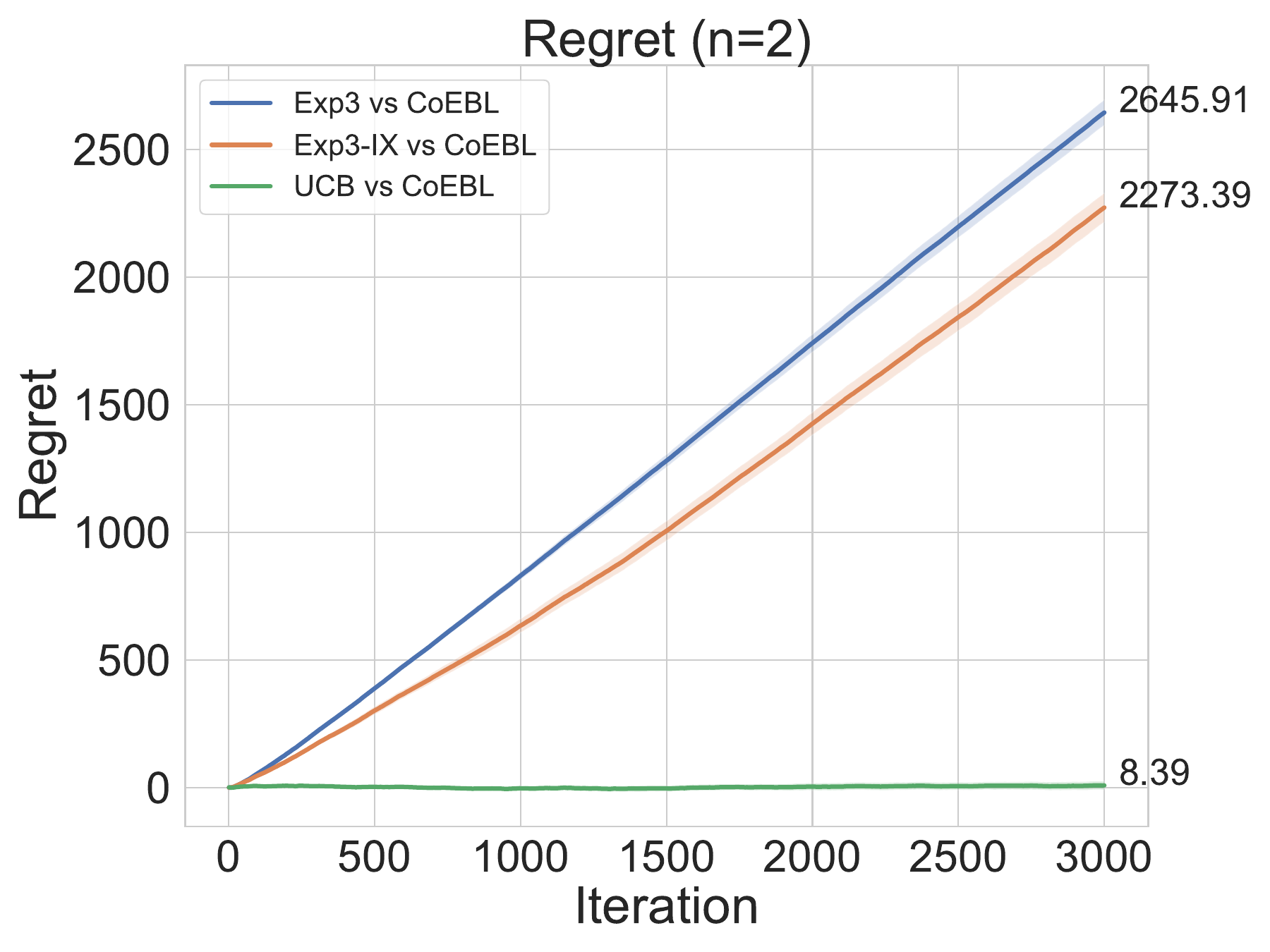}}
    \hfill
    \subfloat[
    % Regret (n=3)
    ]{%
       \includegraphics[width=0.33\linewidth]{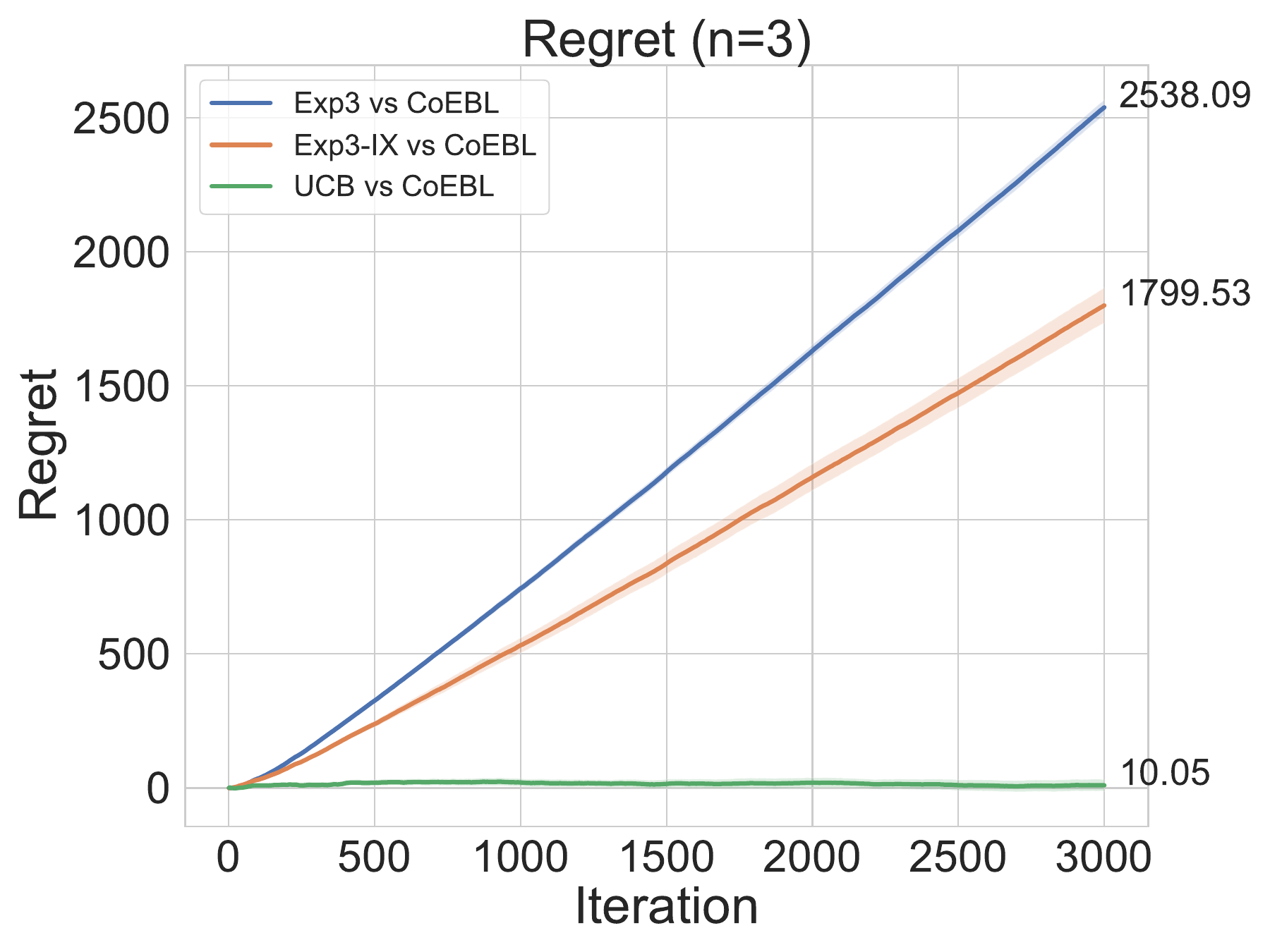}}
    \hfill
    \subfloat[
    % Regret (n=4)
    ]{%
       \includegraphics[width=0.33\linewidth]{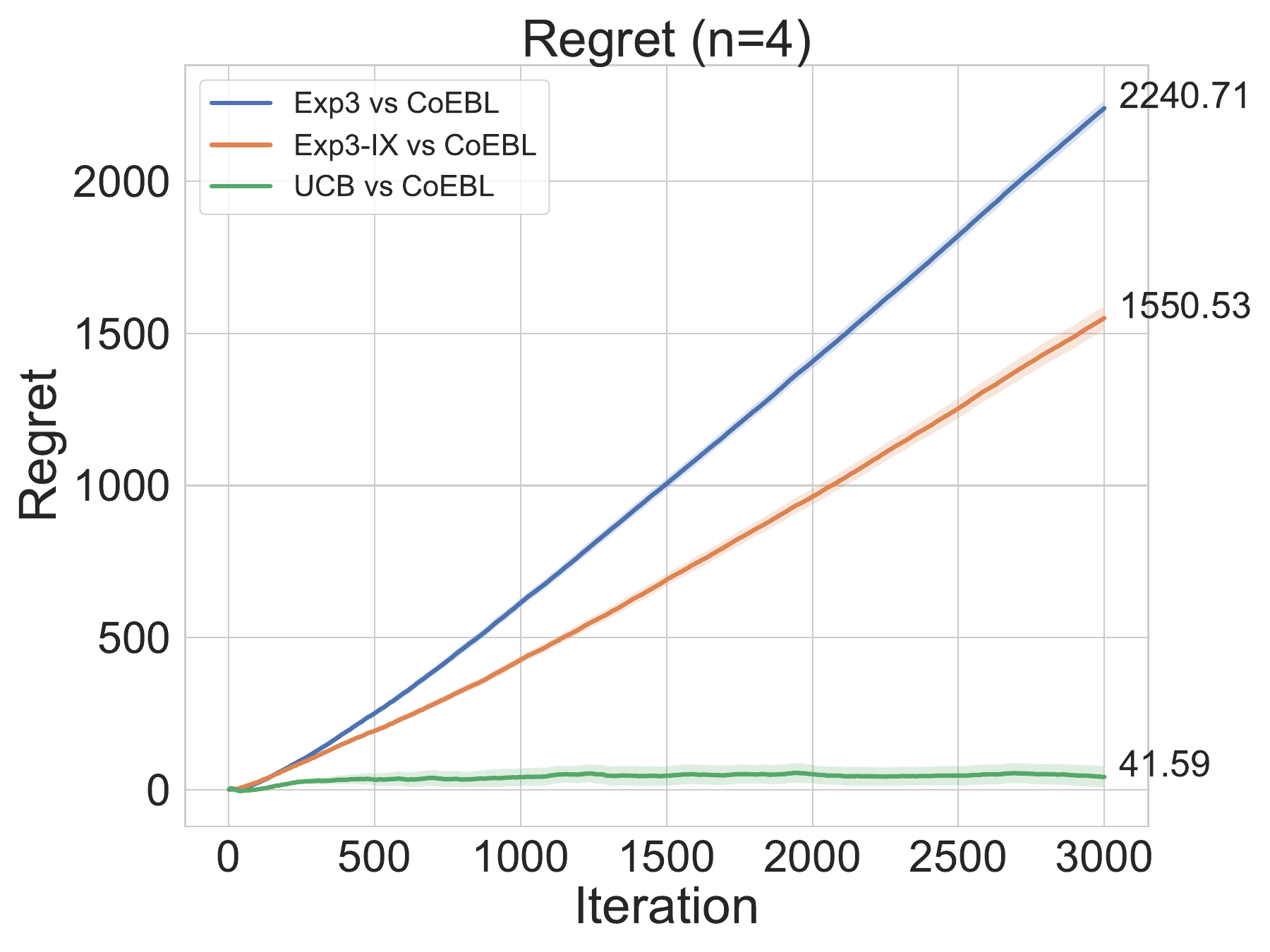}}
    \hfill
    \subfloat[
    % Regret (n=5)
    ]{%
       \includegraphics[width=0.33\linewidth]{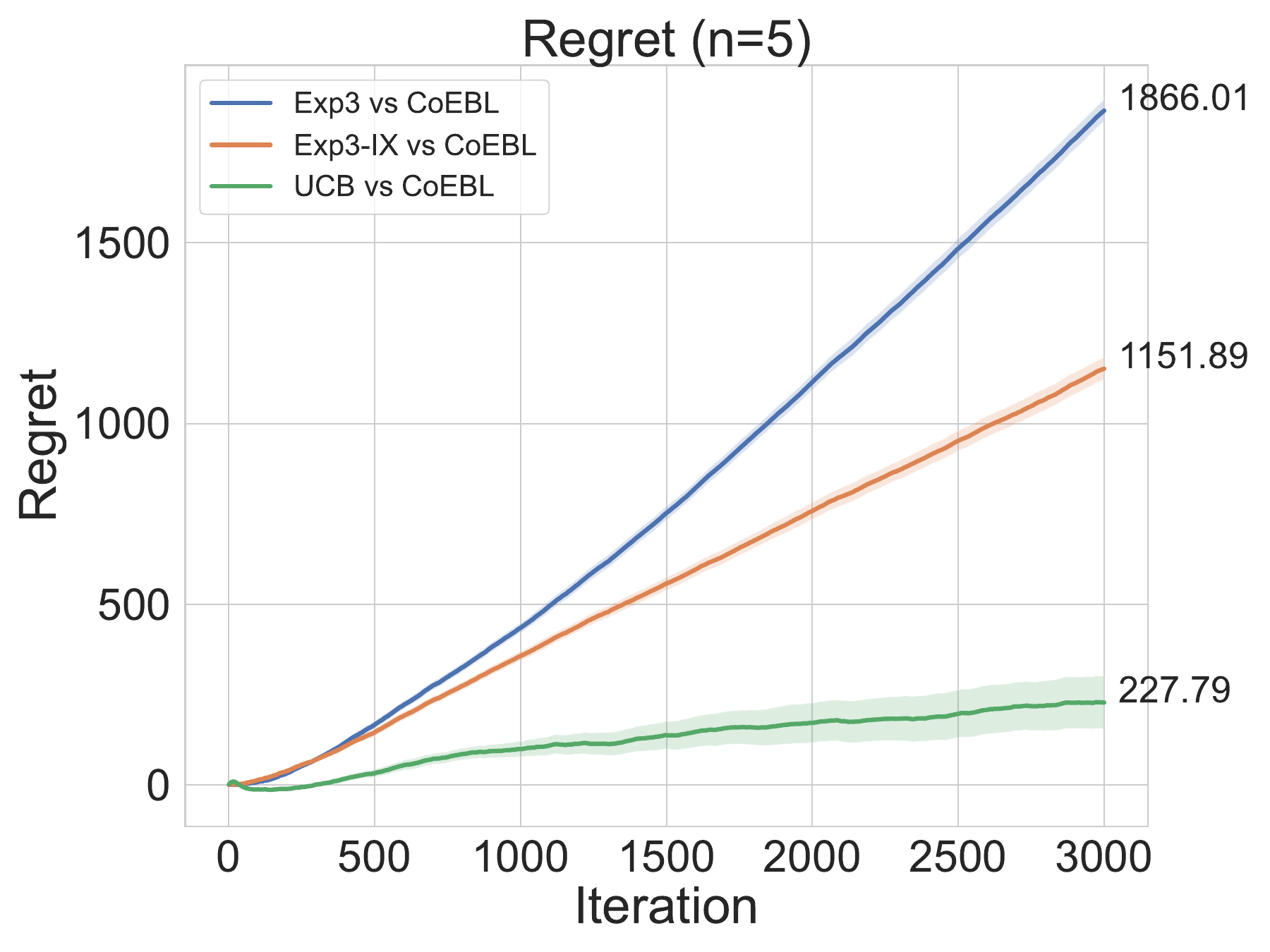}}
    \hfill
    \subfloat[
    % Regret (n=6)
    ]{%
       \includegraphics[width=0.33\linewidth]{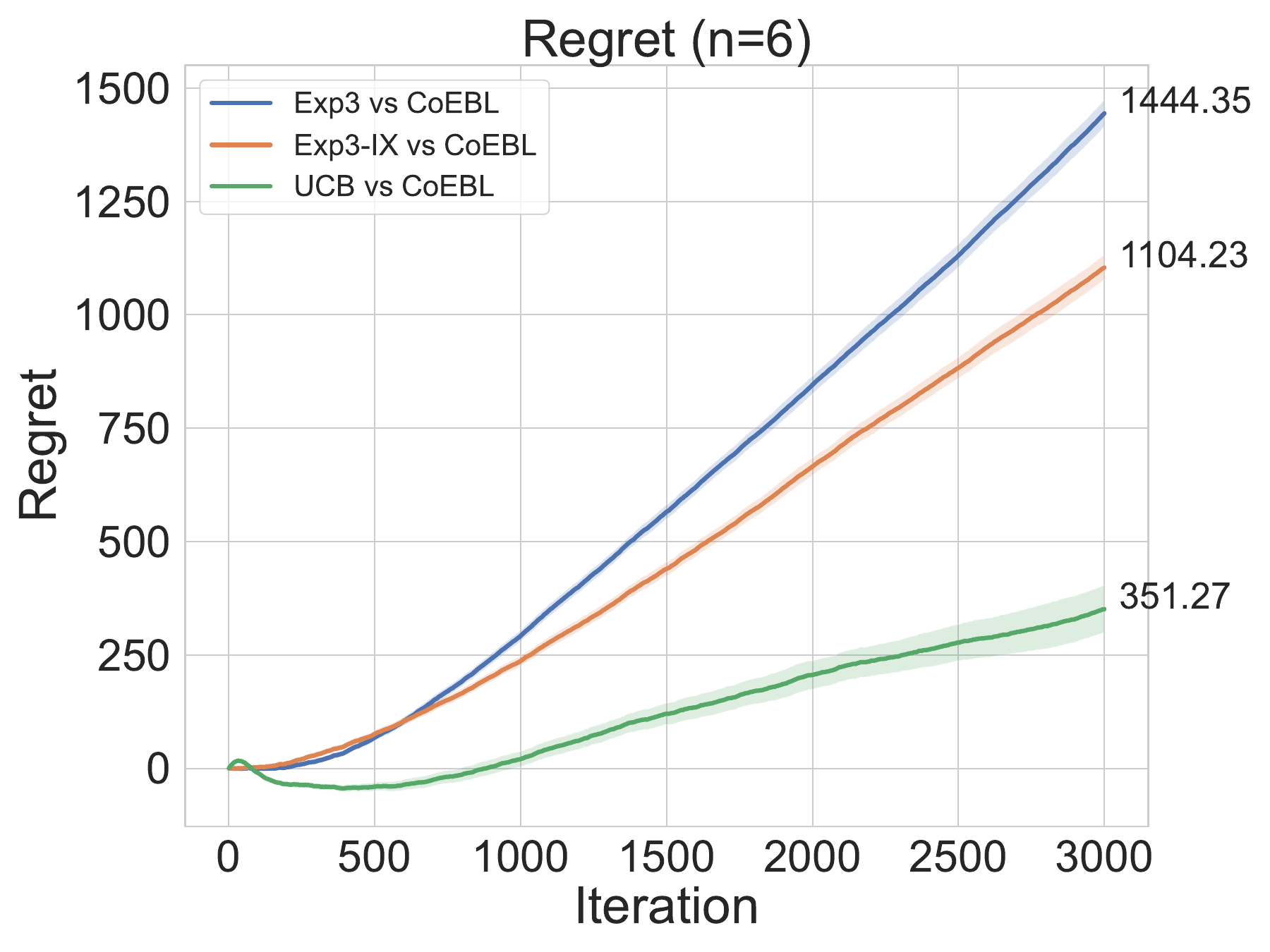}}
    \hfill
    \subfloat[
    % Regret (n=7)
    ]{%
       \includegraphics[width=0.33\linewidth]{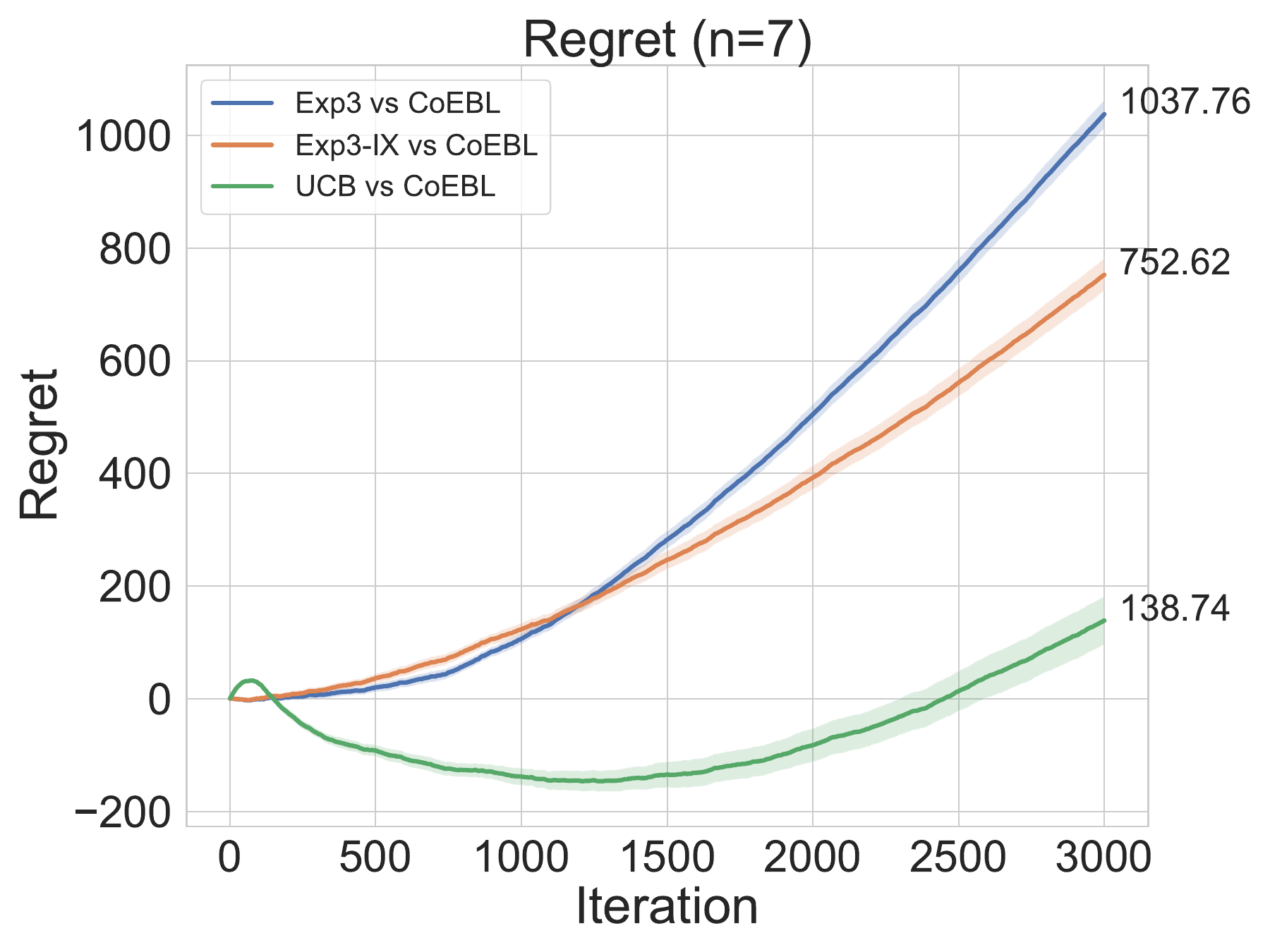}}
    \hfill
    \subfloat[
    % TV (n=2)
    ]{%
       \includegraphics[width=0.33\linewidth]{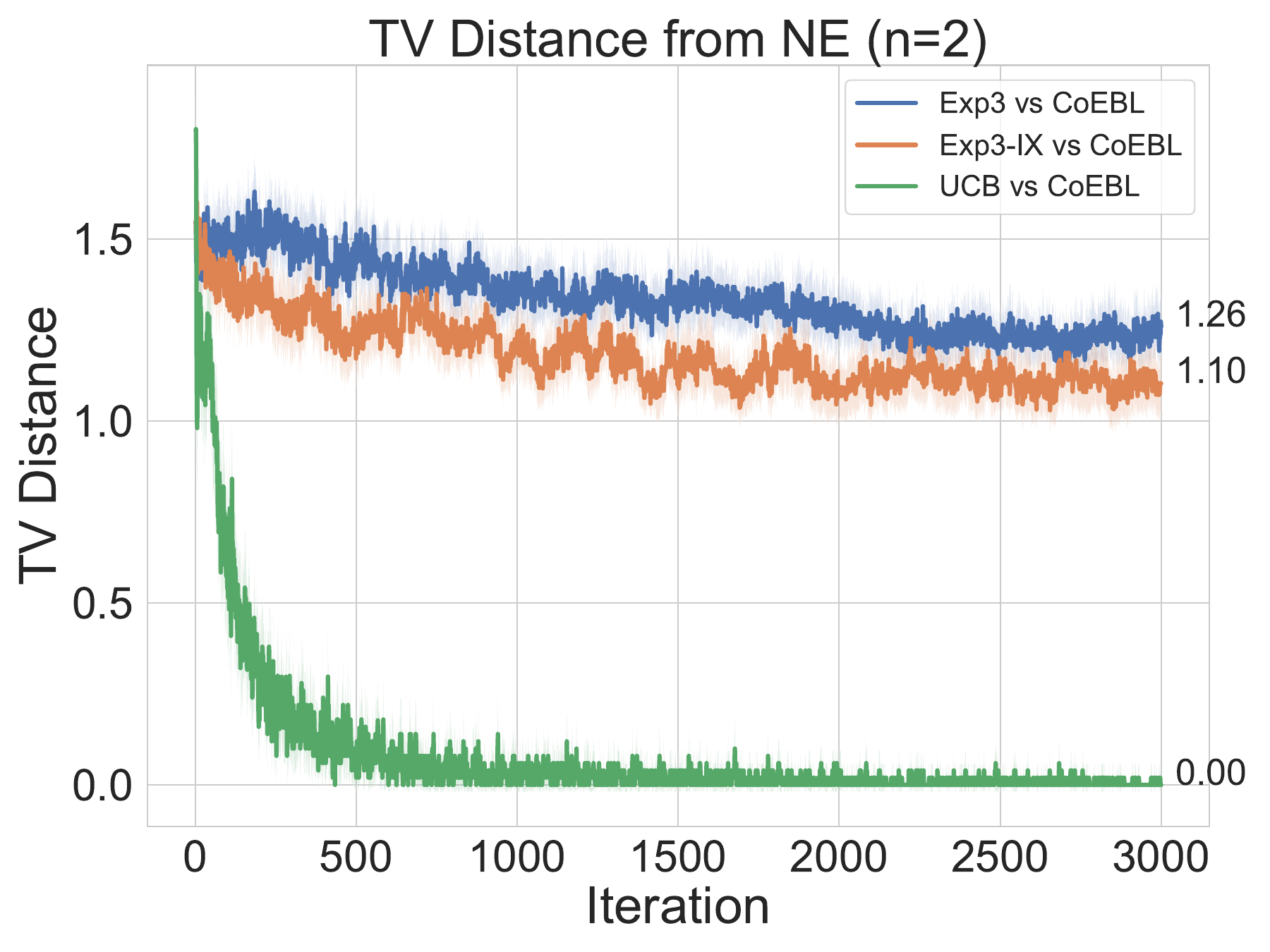}}
    \hfill
    \subfloat[
    % TV (n=3)
    ]{%
       \includegraphics[width=0.33\linewidth]{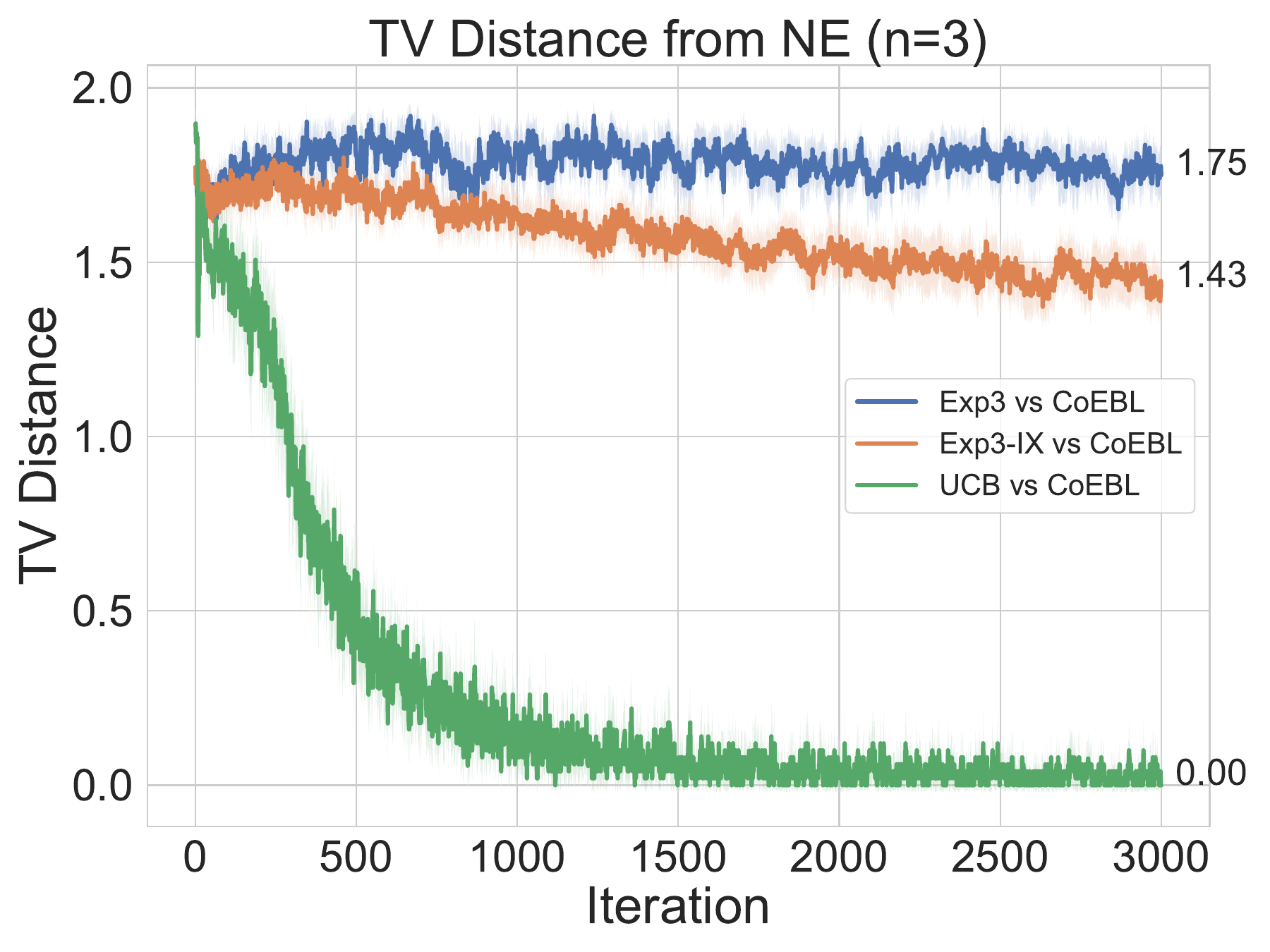}}
    \hfill
    \subfloat[
    % TV (n=4)
    ]{%
       \includegraphics[width=0.33\linewidth]{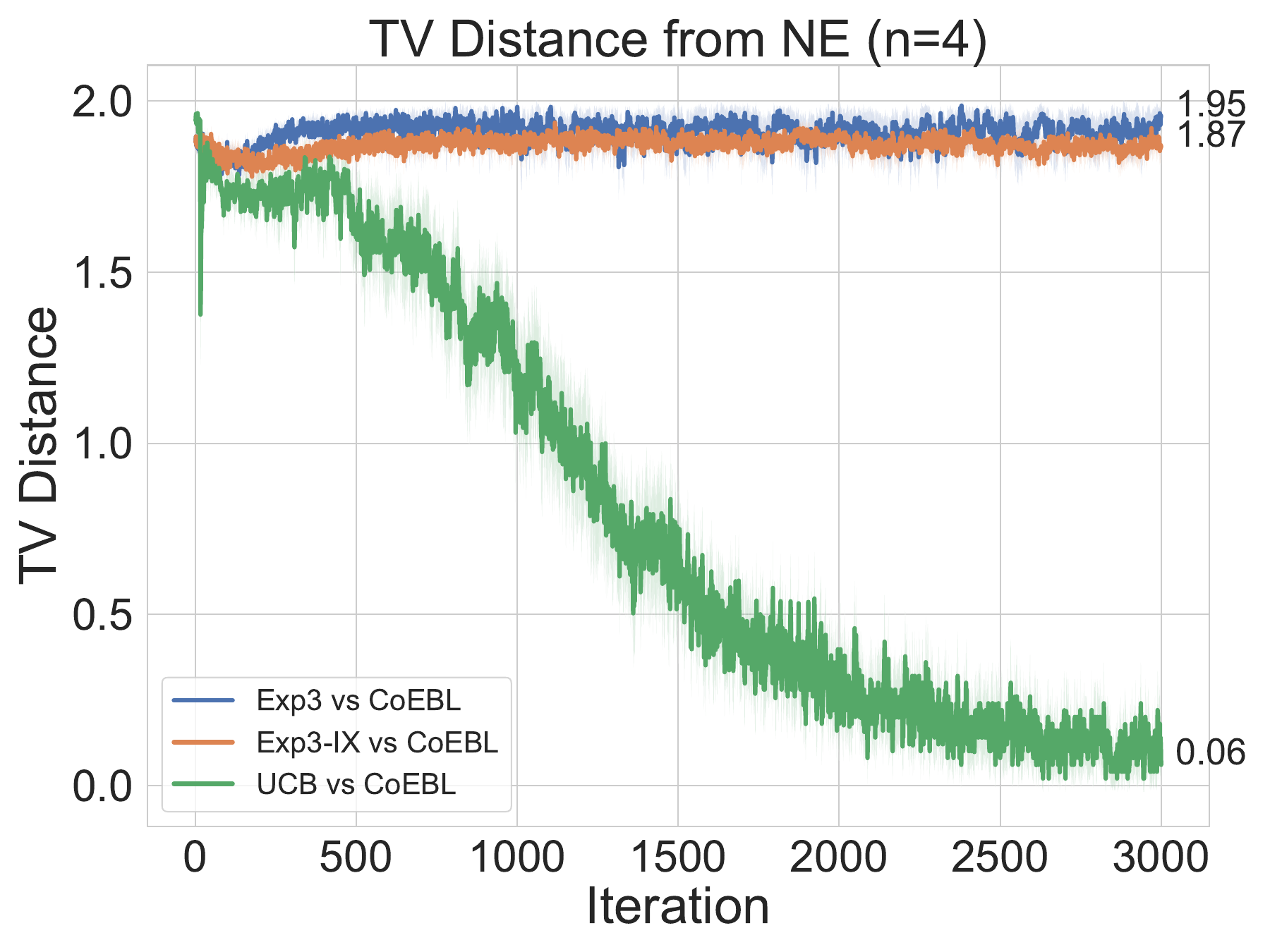}}
    \hfill
    \subfloat[
    % TV (n=5)
    ]{%
       \includegraphics[width=0.33\linewidth]{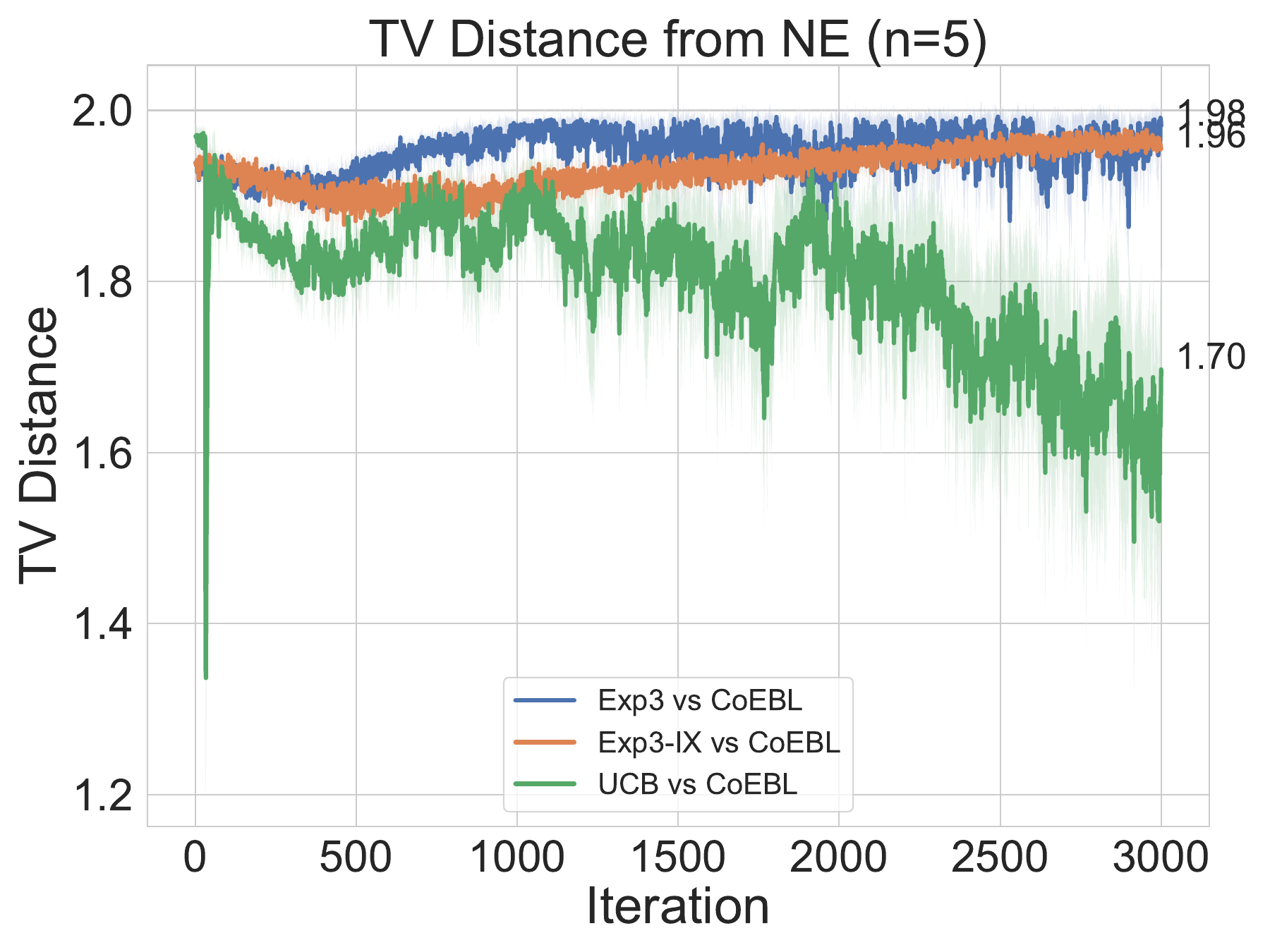}}
    \hfill
    \subfloat[
    % TV (n=6)
    ]{%
       \includegraphics[width=0.33\linewidth]{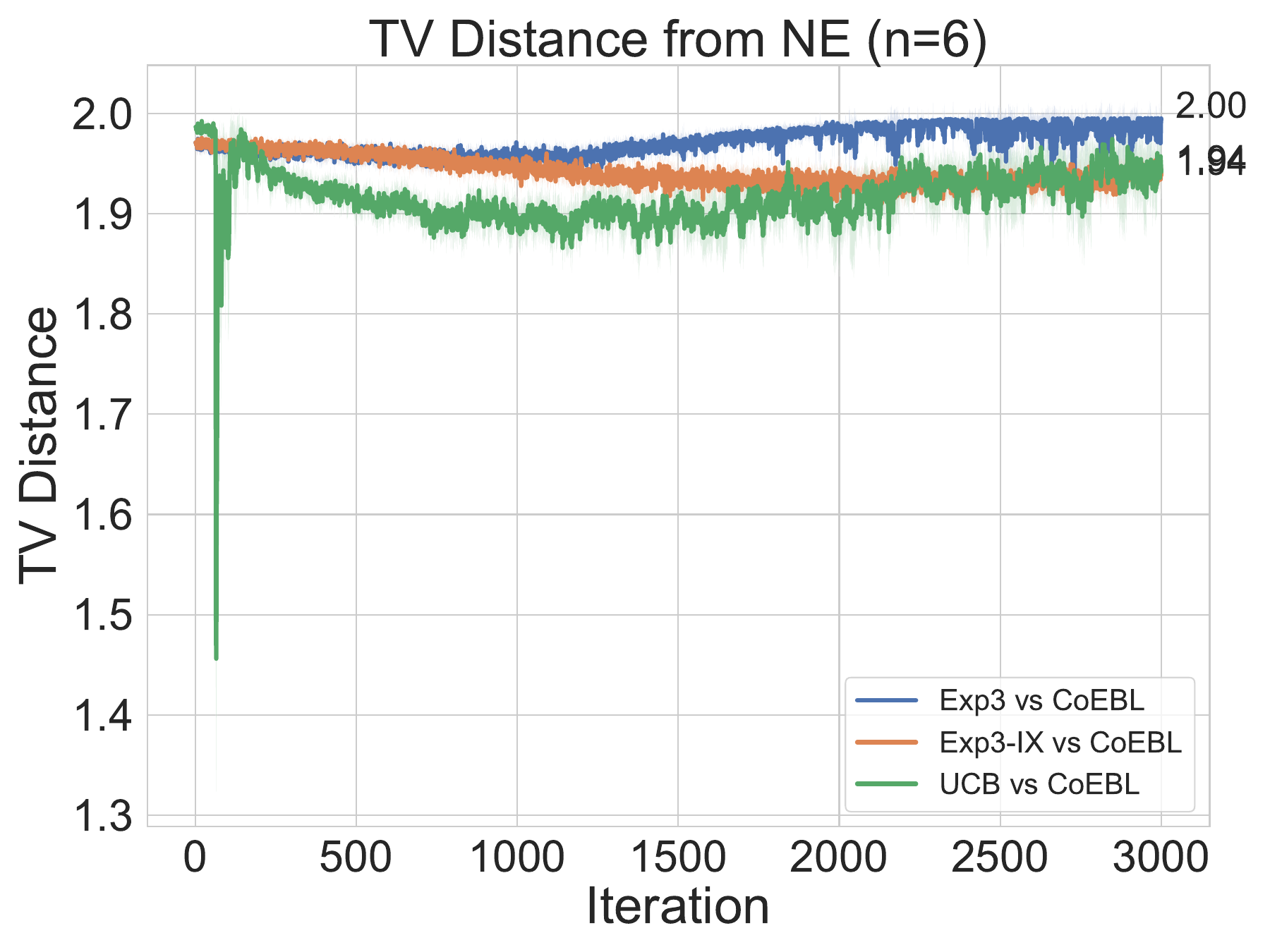}}
    \hfill
    \subfloat[
    % TV (n=7)
    ]{
    \includegraphics[width=0.33\linewidth]{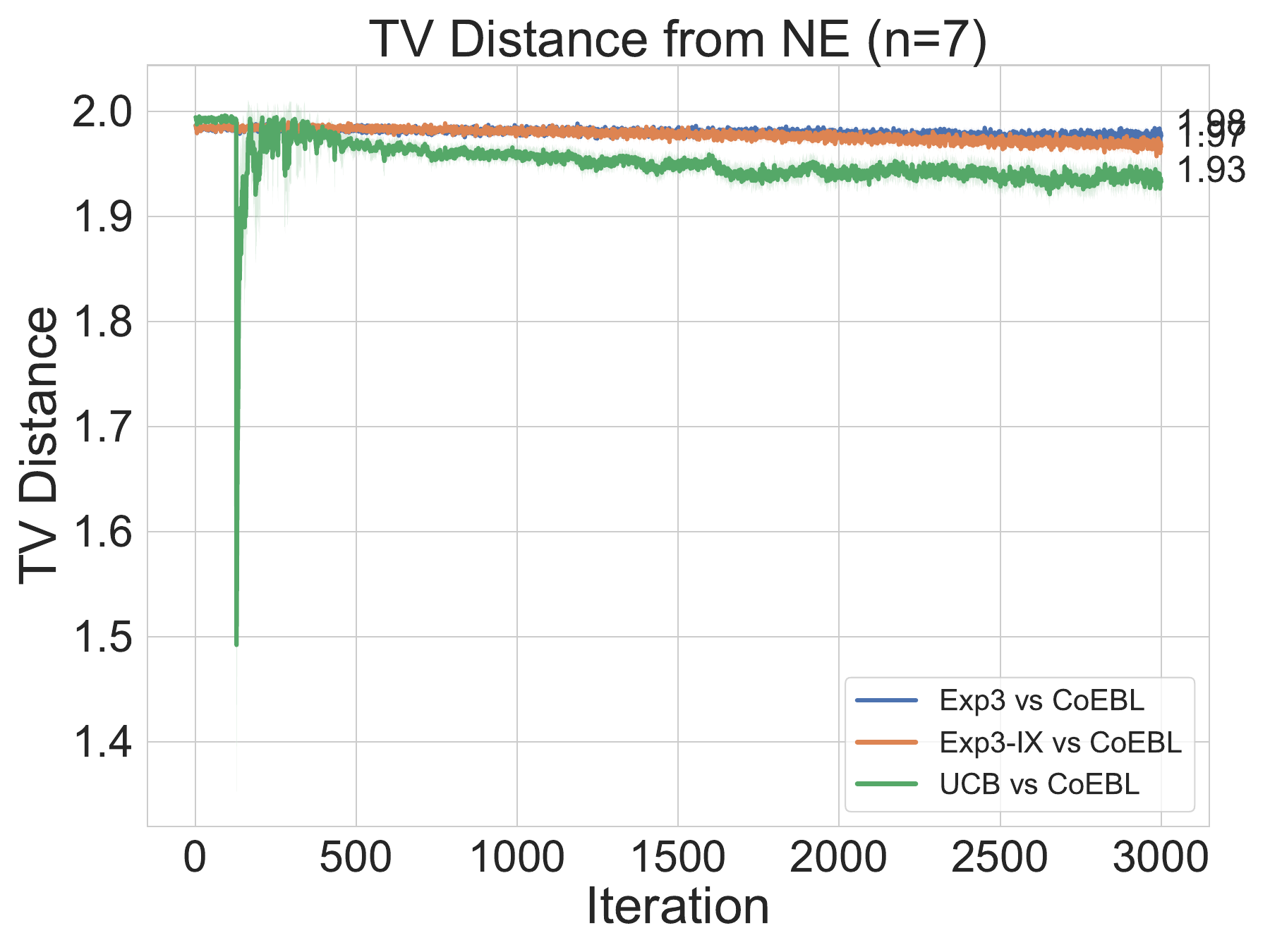}}
    \caption{Regret and TV-distance for $\alg~1$-vs-$\alg~2$ on \BigNum.}
    \hfill
    \label{fig:Regret_BigNum2} 
 \end{figure}

\end{document}